\documentclass{article}

\usepackage{iclr2027_conference,times}
\iclrfinalcopy

\usepackage[utf8]{inputenc} % allow utf-8 input
\usepackage[T1]{fontenc}    % use 8-bit T1 fonts
\usepackage{hyperref}       % hyperlinks
\usepackage{url}            % simple URL typesetting
\usepackage{booktabs}       % professional-quality tables
\usepackage{amsfonts}       % blackboard math symbols
\usepackage{nicefrac}       % compact symbols for 1/2, etc.
\usepackage{microtype}      % microtypography

\usepackage{float}
\usepackage[export]{adjustbox}

\usepackage{soul}

\usepackage{amsmath}
\usepackage{amssymb}
\usepackage{mathtools}
\usepackage{amsthm}
\usepackage{amsthm}
\usepackage{amsfonts}

\usepackage{thmtools, thm-restate}

\usepackage{mathtools}
\usepackage{mathrsfs}

\usepackage[capitalize]{cleveref}

\usepackage{xparse}

\PassOptionsToPackage{dvipsnames}{xcolor}
\usepackage{tikz,pgfplots}
\usepackage{tikzscale}
\pgfplotsset{compat=1.18}
\usetikzlibrary{bayesnet}
\usetikzlibrary{shapes, arrows, calc, positioning, matrix}
\usepackage{tikz-cd}
\tikzcdset{scale cd/.style={every label/.append style={scale=#1},
    cells={nodes={scale=#1}}}}
\tikzstyle{do} = [obs, rectangle]
\usetikzlibrary{calc,intersections}
\usetikzlibrary{angles}
\usepackage{soul}
\tikzset{
    right angle quadrant/.code={
        \pgfmathsetmacro\quadranta{{1,1,-1,-1}[#1-1]}     % Arrays for selecting quadrant
        \pgfmathsetmacro\quadrantb{{1,-1,-1,1}[#1-1]}},
    right angle quadrant=1, % Make sure it is set, even if not called explicitly
    right angle length/.code={\def\rightanglelength{#1}},   % Length of symbol
    right angle length=2ex, % Make sure it is set...
    right angle symbol/.style n args={3}{
        insert path={
            let \p0 = ($(#1)!(#3)!(#2)$),     % Intersection
                \p1 = ($(\p0)!\quadranta*\rightanglelength!(#3)$), % Point on base line
                \p2 = ($(\p0)!\quadrantb*\rightanglelength!(#2)$), % Point on perpendicular line
                \p3 = ($(\p1)+(\p2)-(\p0)$) in  % Corner point of symbol
            (\p1) -- (\p3) -- (\p2)
        }
    }
}

\pgfplotsset{compat=1.18}
\usetikzlibrary{hobby}                
\usetikzlibrary{decorations.markings}
\usetikzlibrary{decorations.text}
\usetikzlibrary{backgrounds}
\usetikzlibrary{arrows.meta} %% JUST FOR pi-x.tikz

\usepackage{graphicx}
\usepackage{subcaption}

\usepackage{enumitem}

\usepackage{wrapfig}

\usepackage{booktabs}
\usepackage{array}
\usepackage{longtable}

\usetikzlibrary{patterns.meta}

\usepackage{pdfrender}

\usepackage{makecell}
\usepackage{multirow}
\usepackage{booktabs}

\usepackage{stfloats}

\usepackage{etoolbox}

\usepackage{titletoc}

\usepackage{scalerel} % For auto-scaling the symbol in subscripts/superscripts; used for talos-pi

\usepackage{nicefrac}       % compact symbols for 1/2, etc.
\usepackage{xfrac}

\usepackage[embolden=0pt]{talos}

\usepackage{upgreek}

\usepackage{algorithm}
\usepackage{algorithmic}

\usepackage{pifont}

\declaretheorem{theorem}
\declaretheorem{corollary}
\declaretheorem{lemma}
\declaretheorem{proposition}

\declaretheorem[style=definition]{assumption}

\declaretheorem[style=definition]{remark}

\crefname{theorem}{Thm.}{Thms.}
\crefname{subtheorem}{Thm.}{Thms.}
\crefname{corollary}{Cor.}{Cors.}
\crefname{lemma}{Lem.}{Lems.}
\crefname{proposition}{Prop.}{Props.}
\crefname{observation}{Obs.}{Obs.}
\crefname{assumption}{Asm.}{Asms.}
\crefname{example}{Ex.}{Exs.}
\crefname{remark}{Rem.}{Rems.}
\crefname{equation}{Eq.}{Eqs.}
\crefname{condition}{Cond.}{Conds.}
\crefname{algorithm}{Algo.}{Algos.}

\Crefname{theorem}{Theorem}{Theorems}
\Crefname{subtheorem}{Theorem}{Theorems}
\Crefname{corollary}{Corollary}{Corollaries}
\Crefname{lemma}{Lemma}{Lemmas}
\Crefname{proposition}{Proposition}{Propositions}
\Crefname{observation}{Observation}{Observations}
\Crefname{assumption}{Assumption}{Assumptions}
\Crefname{example}{Example}{Examples}
\Crefname{remark}{Remark}{Remarks}
\Crefname{equation}{Equation}{Equations}
\Crefname{condition}{Condition}{Conditios}

\crefformat{chapter}{\S#2#1#3}
\crefmultiformat{chapter}{\S\S#2#1#3}{ and~#2#1#3}{, #2#1#3}{, and~#2#1#3}

\crefformat{section}{\S#2#1#3}
\crefmultiformat{section}{\S\S#2#1#3}{ and~#2#1#3}{, #2#1#3}{, and~#2#1#3}

\crefformat{appendix}{\S#2#1#3}
\crefmultiformat{appendix}{\S\S#2#1#3}{ and~#2#1#3}{, #2#1#3}{, and~#2#1#3}
\Crefformat{appendix}{\S#2#1#3}
\Crefmultiformat{appendix}{\S\S#2#1#3}{ and~#2#1#3}{, #2#1#3}{, and~#2#1#3}

\usepackage{stackengine}
\usepackage[dvipsnames]{xcolor}
\input pdf-trans
\newbox\qbox
\def\usecolor#1{\csname\string\color@#1\endcsname\space}
\newcommand\bordercolor[1]{\colsplit{1}{#1}}
\newcommand\fillcolor[1]{\colsplit{0}{#1}}
\newcommand\colsplit[2]{\colorlet{tmpcolor}{#2}\edef\tmp{\usecolor{tmpcolor}}%
  \def\tmpB{}\expandafter\colsplithelp\tmp\relax%
  \ifnum0=#1\relax\edef\fillcol{\tmpB}\else\edef\bordercol{\tmpC}\fi}
\def\colsplithelp#1#2 #3\relax{%
  \edef\tmpB{\tmpB#1#2 }%
  \ifnum `#1>`9\relax\def\tmpC{#3}\else\colsplithelp#3\relax\fi
}
\newcommand\outline[1]{\leavevmode%
  \def\maltext{#1}%
  \setbox\qbox=\hbox{\maltext}%
  \boxgs{Q q 2 Tr \thickness\space w \fillcol\space \bordercol\space}{}%
  \copy\qbox%
}
\newcommand\myMathbb[2][1]{%
  \stackengine{0pt}{\def\thickness{.15}\outline{${#2}$}}{\kern.1pt\outline{${#2}$}}{O}{l}{F}{F}{L}}
\bordercolor{black}
\fillcolor{white}
\def\thickness{.4}% TO CHANGE THICKNESS OF SHADOW

\newcommand{\rnum}[1]{\expandafter{\romannumeral #1\relax}}
\newcommand{\RNum}[1]{\uppercase\expandafter{\romannumeral #1\relax}}

\newif\ifrestatingtheorem
\restatingtheoremfalse

\definecolor{deep_blue}{HTML}{4c72b0}
\definecolor{deep_orange}{HTML}{dd8452}
\definecolor{deep_green}{HTML}{55a868}
\definecolor{deep_red}{HTML}{c44e52}
\definecolor{deep_purple}{HTML}{8172b3}
\definecolor{deep_brown}{HTML}{937860}
\definecolor{deep_pink}{HTML}{da8bc3}
\definecolor{deep_grey}{HTML}{8c8c8c}

\definecolor{DarkBlue}{RGB}{0,0,139}
\newif\ifshowchanges
\showchangestrue   % or \showchangesfalse
\newcommand{\GenerateBlobCoords}[5]{
    \pgfmathsetseed{#4}
    \def\N{7} % 7 points = Sexy Convexity
    \global\def\TempCoords{}
    \foreach \k in {1,...,\N}{
        \pgfmathsetmacro{\ang}{(\k * (360/\N)) + (rand * 15)}
        \pgfmathsetmacro{\rad}{#3 + (rand * 0.15 * #3)}
        \pgfmathsetmacro{\bx}{#1 + \rad*cos(\ang)}
        \pgfmathsetmacro{\by}{#2 + \rad*sin(\ang)}
        \xdef\TempCoords{\TempCoords (\bx,\by)}
    }
    \global\let#5\TempCoords
}

\newcommand{\DrawBlob}[2]{
    \draw[
        color=#2, 
        ultra thick, 
        fill=#2, 
        fill opacity=0.1,
    ] plot [closed hobby, tension=1] coordinates { #1 };
}
\newcommand{\EntropyBlob}[5]{
    \pgfmathsetseed{#4}
    
    \def\N{7} % 7 points = Sexy Convexity
    \global\def\BlobCoords{}
    \foreach \k in {1,...,\N}{
        \pgfmathsetmacro{\ang}{(\k * (360/\N)) + (rand * 15)}
        \pgfmathsetmacro{\rad}{#3 + (rand * 0.15 * #3)}
        
        \pgfmathsetmacro{\bx}{#1 + \rad*cos(\ang)}
        \pgfmathsetmacro{\by}{#2 + \rad*sin(\ang)}
        
        \xdef\BlobCoords{\BlobCoords (\bx,\by)}
    }
    
    \edef\DrawCmd{
        \noexpand\draw[
            color=#5, 
            ultra thick, 
            fill=#5, 
            fill opacity=0.1,
        ] plot [closed hobby, tension=1] coordinates { \BlobCoords };
    }
    \DrawCmd
}

\newcommand{\CigarBlob}[6]{
    \pgfmathsetseed{#5}

    \def\N{7}

    \global\def\BlobCoords{}

    \foreach \k in {1,...,\N}{
        \pgfmathsetmacro{\t}{
            (\k-1)*(360/\N) + rand*8
        }

        \pgfmathsetmacro{\stretch}{
            1 + 0.25*cos(2*\t)
        }

        \pgfmathsetmacro{\noise}{
            1 + 0.10*rand
        }

        \pgfmathsetmacro{\bx}{
            #1 + #3*\stretch*\noise*cos(\t)
        }

        \pgfmathsetmacro{\by}{
            #2 + #4*\noise*sin(\t)
        }

        \xdef\BlobCoords{\BlobCoords (\bx,\by)}
    }

    \edef\DrawCmd{
        \noexpand\draw[
            color=#6,
            ultra thick,
            fill=#6,
            fill opacity=0.1
        ]
        plot [closed hobby, tension=1]
        coordinates { \BlobCoords };
    }

    \DrawCmd
}

\hypersetup{
    colorlinks=true,
    linkcolor=deep_red,      % Internal links (cref, ref) - matches your primary color
    citecolor=deep_green,     % Citations - distinct but harmonious
    urlcolor=deep_blue,     % URLs - different enough to distinguish
    filecolor=deep_purple     % File links
}

\dottedcontents{section}[2.3em]{}{2.3em}{5pt}

\makeatletter

\newcommand{\stored@titlecmd}{\@latex@warning@no@line{Title command not stored}}
\newcommand{\stored@authorcmd}{\@latex@warning@no@line{Author command not stored}}

\AfterEndPreamble{%
  \global\let\stored@titlecmd\@title
  \global\let\stored@authorcmd\@author
}

\newcommand{\appendixmaketitle}[1]{%
  {%
  \par
  \begingroup
    \let\thanks\@gobble
    \thispagestyle{plain}
    \vbox{%
      \hsize\textwidth
      \linewidth\hsize
      \vskip 0.1in
      \@toptitlebar
      \centering
      {\Large\bf #1\par}
      \@bottomtitlebar

      \if@anonymous % <--- CHANGED for neurips_2026.sty
        \begin{tabular}[t]{c}\bf\rule{\z@}{24\p@}
          Anonymous Author(s) \\
          Affiliation \\
          Address \\
          \texttt{email} \\
        \end{tabular}%
      \else
        \def\And{%
          \end{tabular}\hfil\linebreak[0]\hfil%
          \begin{tabular}[t]{c}\bf\rule{\z@}{24\p@}\ignorespaces%
        }
        \def\AND{%
          \end{tabular}\hfil\linebreak[4]\hfil%
          \begin{tabular}[t]{c}\bf\rule{\z@}{24\p@}\ignorespaces%
        }
        \begin{tabular}[t]{c}\bf\rule{\z@}{24\p@}\stored@authorcmd\end{tabular}%
      \fi
      \vskip 0.3in \@minus 0.1in
    }%
  \endgroup
  \par
  \vskip \bigskipamount
  }%
}

\makeatother

\newif\ifshowfigures
\showfiguresfalse

\newcommand{\asmOneA}{Asm.~\hyperref[assumption:hypothesis-class:no-bias]{1a}}
\newcommand{\asmOneB}{Asm.~\hyperref[assumption:hypothesis-class:abs-cont]{1b}}
\newcommand{\asmOneC}{Asm.~\hyperref[assumption:hypothesis-class:finite-dim]{1c}}

\newcommand{\asmOneShortA}{\hyperref[assumption:hypothesis-class:no-bias]{1a}}
\newcommand{\asmOneShortB}{\hyperref[assumption:hypothesis-class:abs-cont]{1b}}
\newcommand{\asmOneShortC}{\hyperref[assumption:hypothesis-class:finite-dim]{1c}}

\newcommand{\asmOneShortAB}{Asms.~\hyperref[assumption:hypothesis-class:no-bias]{1a}--\hyperref[assumption:hypothesis-class:abs-cont]{1b}}
\newcommand{\asmOneShortBC}{Asms.~\hyperref[assumption:hypothesis-class:abs-cont]{1b}--\hyperref[assumption:hypothesis-class:finite-dim]{1c}}

\usepackage{amsfonts,bm}

\newcommand{\figleft}{{\em (Left)}}

\newcommand{\figright}{{\em (Right)}}
\newcommand{\figtop}{{\em (Top)}}
\newcommand{\figbottom}{{\em (Bottom)}}

\newcommand{\captiona}{{\em (a)}}
\newcommand{\captionb}{{\em (b)}}
\newcommand{\captionc}{{\em (c)}}
\newcommand{\captiond}{{\em (d)}}

\def\1{\bm{1}}

\def\vgamma{{\bm{\gamma}}}

\def\tgamma{\tilde{{\gamma}}}
\def\stgamma{\smash[t]{\tilde{{\gamma}}}}

\def\vtheta{{\bm{\theta}}}

\def\vw{{\bm{w}}}
\def\mA{{\bm{A}}}

\def\mI{{\bm{I}}}

\def\mSigma{{\bm{\Sigma}}}

\def\tb{{\tilde{b}}}
\def\stb{{\smash[t]{\tilde{b}}}}
\def\tX{{\widetilde{X}}}
\def\stX{{\smash[t]{\widetilde{X}}}}
\def\tsigma{\smash[t]{\tilde{\sigma}}}

\def\tV{{\widetilde{V}}}
\def\stV{{\smash[t]{\widetilde{V}}}}

\def\sA{{\mathcal{A}}}
\def\sB{{\mathcal{B}}}
\def\sC{{\mathcal{C}}}
\def\sE{{\mathcal{E}}}
\def\sF{{\mathcal{F}}}
\def\sH{{\mathcal{H}}}
\def\sI{{\mathcal{I}}}
\def\sK{{\mathcal{K}}}

\def\sN{{\mathcal{N}}}

\def\sP{{\mathcal{P}}}
\def\sQ{{\mathcal{Q}}}
\def\sS{{\mathcal{S}}}
\def\sT{{\mathcal{T}}}

\def\sV{{\mathcal{V}}}

\newcommand{\R}{\mathbb{R}}
\newcommand{\Var}{\mathrm{Var}}

\newcommand{\Cov}{\mathrm{Cov}}
\newcommand{\var}{\mathrm{Var}}

\DeclareMathOperator*{\argmin}{argmin}

\newcommand{\mgeq}{\succcurlyeq}
\newcommand{\mgt}{\succ}

\newcommand{\indep}{\mathrel{\perp\mkern-10.5mu\perp}}

\newcommand\given{\;\!\middle\vert\;\!}

\NewDocumentCommand\intervention{mg}{
    \ensuremath{
        \operatorname{do}\pr*{ 
            #1 \IfNoValueTF{#2}{}{ = #2 }
        }
    }
}

\newcommand{\target}{\star}

\newcommand{\inv}{\textnormal{\i}\mathrm{nv}}

\newcommand{\erm}{\mathrm{erm}}
\newcommand{\term}{%
  \smash[t]{%
    \mkern7mu
    \widetilde{\mkern-7mu\mathrm{erm}\mkern-7mu}%
    \mkern7mu
  }%
}

\newcommand{\pR}{r}
\newcommand{\iv}{\textnormal{\i}\mathrm{v}}
\newcommand{\tiv}{\smash[t]{\widetilde{\textnormal{\i}\mathrm{v}}}}
\newcommand{\pid}{\mathrm{p}\textnormal{\i}}
\newcommand{\tpid}{\smash[t]{\widetilde{\pid}}}
\newcommand{\IPM}{D}
\newcommand{\DPI}{\rho}

\newcommand{\vtau}{\bm{\tau}}

\def\vx{{\bm{x}}}

\def\vz{{\bm{z}}}

\newcommand{\tg}{\tilde{g}}

\newcommand{\tZ}{\widetilde{Z}}
\newcommand{\stZ}{\smash[t]{\widetilde{Z}}}

\NewDocumentCommand\cov{mg}{
    \ensuremath{
        {\bm{\Sigma}}_{#1}^{\IfNoValueTF{#2}{}{#2}}
    }
}
\NewDocumentCommand\tCov{mg}{
    \ensuremath{
        \smash{\tilde{\bm{\Sigma}}}_{#1}^{\IfNoValueTF{#2}{}{#2}}
    }
}

\newcommand{\E}{\mathbb{E}}

\NewDocumentCommand\0{mg}{
\ensuremath{
    \mathbf{0}_{ #1 \IfNoValueTF{#2}{}{\times #2}}
}
}

\def\sX{{\mathcal{X}}}
\def\sY{{\mathcal{Y}}}
\def\sZ{{\mathcal{Z}}}
\def\sG{{\mathcal{G}}}
\def\sD{{\mathcal{D}}}
\def\sR{{\mathcal{R}}}

\def\sPi{{\mathcal{\Pi}}}

\def\Id{{\rm Id}}

\def\shift{\sS}
\def\varShift{\sK}
\def\meanShift{\delta}
\def\meanShiftAdj{\meanShift^\ast}
\def\mNoise{{\eta}}
\def\eps{{\varepsilon}}

\DeclarePairedDelimiter\abs{\lvert}{\rvert}
\DeclarePairedDelimiter{\pr}{(}{)}
\DeclarePairedDelimiter{\br}{\{}{\}}
\DeclarePairedDelimiter{\sqPr}{[}{]}
\DeclarePairedDelimiter\norm{\lVert}{\rVert}
\newcommand{\sqNorm}[1]{\norm*{#1}^2}

\newcommand{\EE}[3]{\mathbb{E}_{#1}^{#2}\sqPr*{\;\! #3 \;\!}}

\newcommand{\eE}[3]{\mathbb{E}_{#1}^{#2}\sqPr{ #3 }}
\def\ee{\mathbb{E}}

\newcommand{\Risk}[3]{R_{#1}^{#2}\pr*{#3}}

\newcommand{\loss}[1]{\ell\pr*{#1}}
\newcommand{\tr}[1]{ \operatorname{tr}\pr*{#1} }
\newcommand{\Span}[1]{ \operatorname{span}\pr*{#1} }

\newcommand{\Null}[1]{ \operatorname{null}\pr*{#1} }
\newcommand{\Rank}[1]{ \operatorname{rank}\pr*{#1} }

\newcommand{\error}[1]{ E^{#1} }

\newcommand{\residual}{\xi}

\newcommand{\func}[2]{#1\pr*{#2}}

\NewDocumentCommand\distribution{mg}{
    \ensuremath{
        p\IfNoValueTF{#2}{}{ ^{\scriptstyle{#2}} }_{ \scriptstyle{#1}}
    }
}

\NewDocumentCommand\candidateDistribution{mg}{
    \ensuremath{
        q_{ \scriptstyle{#1}} \IfNoValueTF{#2}{}{ ^{#2} }
    }
}

\def\diff{\mathrm{d}}
\def\graph{{\sG}}
\def\tgraph{{\widetilde\sG}}
\newcommand{\Pa}[1]{ {\rm Pa}_\graph\pr*{#1} }
\newcommand{\pa}[1]{ {\rm pa}_\graph\pr*{#1} }

\newcommand{\NonDesc}[1]{ {\rm NonDesc}_\graph\pr*{#1} }
\newcommand{\paperTitle}{Symmetry-Informed\\ Causal Partial Identification}
\title{\paperTitle}

\author{%
  \href{mailto:uzair.akbar@gatech.edu?Subject=Your paper on symmetry-informed partial identification}{Uzair Akbar}\\
  Georgia Tech
  \And
  Zulfiqar Zaidi\\
  Georgia Tech
  \And
  Niki Kilbertus\\
  TU Munich\\
  Helmholtz Munich
  \And
  Krikamol Muandet\\
  Rational Intelligence \\ CISPA
  \And
  ~Bo~Dai\\
  Georgia Tech
}

\begin{document}
\maketitle
% The ICLR style sets the running header inside \maketitle
% ("Published as a conference paper at ICLR 2027"); override it here.
\lhead{Preprint.}

\begin{abstract}
Partial identification (PI) entails estimating bounds on causal effects by encoding different assumptions on data generation as a constrained optimization problem. Such bounds can suffice to inform policy decisions even if the causal effect itself is not identifiable. Often vacuous in practice, practitioners seek to exhaustively encode domain knowledge as additional constraints to make the PI bounds more informative. We introduce known \emph{data symmetries}---invariance of the causal effect under certain data transformations---as a new source of constraints to inform PI. We operationalize this as a \emph{shape constraint} on the causal function, and via a \emph{change of measure} against which PI is posed using simple data pre-processing. 
Both approaches are shown to sharpen bounds under two canonical PI models. This is shown both theoretically for the population case, and via experiments in the finite-sample case. More broadly, our framework establishes data symmetries as a natural, underutilized source of background knowledge for robust causal inference.
\end{abstract}

\section{Introduction}
Interventions are the gold standard for answering causal queries~\citep{Pearl_2009,eoci}. Absent interventions, observed associations may be due to hidden confounders rather than causal effects, so researchers~resort~to~adjusting via auxiliary variables, which rest on strong, untestable assumptions that~may~simply~fail~\citep{iv_rui}.
\nocite{making-sensemakr}
\nocite{rosenbaum2002observational}
Partial identification (PI) offers a principled relaxation to this problem. Rather than insisting on assumptions sufficient for point identification, PI returns \emph{bounds} on the causal effect under a milder set of assumptions~\citep{causal-programming}. These bounds are constructed by considering all possible data-generating processes (DGPs) compatible with both the observed data and the chosen assumptions. Consequently, they are only as informative as the strength of said assumptions, and can often be near-vacuous. To make PI actionable for decision-making, practitioners must exhaustively incorporate credible domain knowledge into the PI solver to restrict the set of compatible DGPs, thereby sharpening the bounds.
\nocite{leaky-iv}
\nocite{pmlr-v258-penn25a}
\nocite{plausibly-exogenous}
\nocite{misspecified-npiv1}
\nocite{misspecified-npiv2}
\nocite{r2-leaky-iv}

While existing PI families encode domain knowledge by bounding the strength of hidden confounding~\citep{MSM,rosenbaum2002observational}, positing a causal graph~\citep{implicit-g-constrained-pid}, or imposing shape restrictions on the class of causal functions~\citep{shape-constraints}, a natural and accessible source of prior knowledge remains underutilized: \emph{known~data~symmetries}. Symmetries describe properties of a system that remain invariant under certain data transformations; for example, image content may not change under rotation/reflection/translation~\citep{bronstein2021geometric}.
Readily used in statistical regularization~\citep{group-theory-da-jmlr,invariance_benefit}, we propose re-purposing known symmetries as a natural shape restriction on the set of candidate causal functions to sharpen~PI~bounds.

\paragraph{Contributions.} We give two operationalizations: an invariance shape constraint on the causal functions added to white-box PI solvers~(\cref{sec:explicit-constraint}), and a change-of-measure approach where off-the-shelf or black-box solvers are simply fed pre-processed data via symmetry transformations~(\cref{sec:implicit-constraint}). 
We license the latter by showing it realizes a soft intervention: preserving the causal mechanism and graph up to relabeling, while diluting confounding. Thus, validity is preserved on aforementioned PI families. We then develop observable conditions for bound sharpening, and a simple pad-and-intersect protocol to guard against misspecified symmetry under two canonical PI models: bounded confounding~\citep{making-sensemakr,long-story-short} and~leaky~instruments~\citep{misspecified-npiv1,r2-leaky-iv}. Finally, we validate both approaches on simulated and real data.

\section{Preliminaries}
\label{sec:preliminaries}

%%%%%%%%%%%%%%%%%%%%%%%%%%%%%%%%%%%%%%%%%%
%%%%%%%%%%%%%%%%%%%%%%%%%%%%%%%%%%%%%%%%%%

\subsection{Statistical vs. causal inference}
\label{sec:statistical-vs-causal-inference}
Given $n$ samples $\sD\coloneqq \br*{ (\vx_i, y_i, \vz_i) }_{i\in [n]}$ from the observational distribution $\distribution{XYZ}$ of treatment $X\in\sX\subseteq \R^{d_{\vx}}$, outcome $Y\in\sY\subseteq \R$ and pre-treatment co-variates $Z\in\sZ\subseteq \R^{d_{\vz}}$, we aim to find the function $h_{\target}\in\sH\coloneqq \{ h : \sX\to \R\}$ defining the causal effect of $X$ on $Y$ under the following~DGP:
\begin{align}
\label{eq:canonical-sem}
    Y = \func{h_{\target}}{X} + U + \residual
    ,
    &&
    \EE{}{}{U} = \EE{}{}{\residual} = 0
    ,
    &&
    \residual \indep (U, X, Z)
    .
\end{align}
The unobserved noise $(\residual, U)$ is assumed mean-centered w.l.o.g., with the intercept absorbed into $h_{\target}$. 
Independence of noise $U \indep X$ implies $\EE{}{}{Y\given \vx} = \func{h_{\target}}{\vx}$. \emph{Statistical inference} entails identifying precisely the Bayes optimal predictor $\EE{}{}{Y\given \vx}$ from $\sD$ via \emph{empirical risk minimization} (ERM): for a proper, convex loss $\ell:\R\times\R\rightarrow\R_{+}$, minimize over $h\in\sH$ an empirical version of the \emph{statistical~risk},
\begin{equation}
    \label{eq:statistical-risk}
    \Risk{\erm}{}{ h } \coloneqq \EE{}{}{ \loss{ Y, h\pr*{X} } } .
\end{equation}
Then, for a sufficiently rich hypothesis class, the minimizer $h_{\erm}$ provides an unbiased~estimate~of~$h_{\target}$.~

However, independence of $X$ and the unobserved noise $U$ may not hold in general, in which case $\EE{}{}{U \given \vx} \neq 0$ and the conditional $\EE{}{}{Y\given \vx}$ now gives a biased estimate of $\func{h_{\target}}{\vx}$~\citep{Pearl_2009,eoci}. This dependence arises due to unobserved common causes $U$ of $X$ and $Y$, known as \emph{confounders}. We say that $X$ and $Y$ are confounded and refer to the resulting bias as \emph{confounding bias}~\citep{Pearl_2009}. The observed pre-treatment co-variates $Z$ may contain auxiliary information, e.g. proxies for such confounders, that can help to mitigate this bias. This encompasses \emph{causal inference} which seeks to identify $h_{\target}$ under hidden confounding, or at the very least account for it by finding bounds on $h_{\target}$ should point-identification not be possible. Both approaches are~outlined~below.

%%%%%%%%%%%%%%%%%%%%%%%%%%%%%%%%%%%%%%%%%%
%%%%%%%%%%%%%%%%%%%%%%%%%%%%%%%%%%%%%%%%%%

\subsection{Point identification via intervention and adjustment}
\label{sec:point-identification-via-intervention-and-adjustment}
We can make treatment $X$ and the confounder $U$ independent via an \emph{intervention} $\intervention{ X }{ \vx }$ that sets $X$ to a fixed value $\vx$ in \cref{eq:canonical-sem} during data generation. Under such a treatment intervention, henceforth denoted as $\intervention{ \vx }$ for brevity, outcome $Y$ follows the \emph{interventional distribution} $\distribution{Y \mid \intervention{\vx}}$. The loss from \cref{eq:statistical-risk} evaluated under $\distribution{Y \mid \intervention{\vx}}$ now defines the \emph{causal~risk} $R$~\citep{causal-regularization2},
\begin{align}
\label{eq:causal-risk}
    \Risk{ \vx }{}{y} \coloneqq \EE{}{}{ \loss{ Y, y } \given \intervention{\vx} },
    &&
    \Risk{}{}{h} \coloneqq \EE{ X\sim p_X }{}{ \Risk{X}{}{h(X)} }
    .
\end{align}
The target estimand of \cref{eq:causal-risk} is the \emph{average treatment effect (ATE)} $\EE{}{}{Y\given \intervention{\vx}}$ which equals $\func{h_{\target}}{\vx}$ for the data generation in \cref{eq:canonical-sem}~\citep{backdoor}. Minimizers of \cref{eq:causal-risk} hence yield unbiased estimates of $h_{\target}(\vx)$. To capture the \emph{estimation error} for a candidate hypothesis $h\in\sH$, we use \emph{causal~excess~risk}~\citep{causal-data-augmentation,causal-interpolation-regularization}
\begin{align*}
    E(h) \coloneqq \Risk{}{}{ h } - \Risk{}{}{ h_{\target} } ,
\end{align*}
which removes the irreducible noise from \cref{eq:causal-risk} and measures how far $h$ is from $h_{\target}$, with $E(h_{\target})=0$.
% \ifshowfigures
\begin{figure*}[tbp]
  \captionsetup{type=figure}\addtocounter{figure}{-1}
  \centering
  \begin{minipage}[t]{0.49\textwidth}
    \centering
    \begin{subfigure}[c]{0.55075\linewidth}
      \centering
          \includegraphics[width=0.7667\linewidth, trim=10pt 0pt 0pt 0pt, clip]{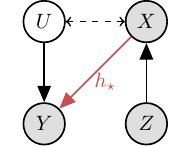}
      \caption{Hidden~confounding~model}
      \label{fig:sem}
    \end{subfigure}
    \hfill%
    \begin{subfigure}[c]{0.4225\linewidth}%
      \centering
          \includegraphics[width=1\linewidth, trim=10pt 0pt 0pt 0pt, clip]{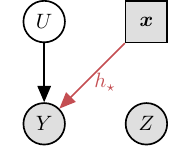}
      \caption{Intervention $\intervention{\vx}$}
      \label{fig:intervention}
    \end{subfigure}
    \captionof{figure}{
      {\captiona} Graph of~\cref{eq:canonical-sem} with IV~$Z$ and hidden noise~$U$ that confounds $(X,Y)$.
      {\captionb} An intervention $\intervention{X\!=\vx}$ removes this confounding.
    }
    \label{fig:iv-intervention}
  \end{minipage}
  \hfill
  \begin{minipage}[t]{0.49\textwidth}
    \centering
    \setlength{\fboxsep}{0pt}
    \begin{subfigure}[c]{0.5\linewidth}
      \centering
      \adjustbox{width=\linewidth, trim=0.6cm 0.1cm 0.1cm 0cm, clip}{
        \includegraphics{tikz/pi-x.tikz}
      }
    \end{subfigure}
    \hfill
    \begin{subfigure}[c]{0.46\linewidth}
      \centering
      \adjustbox{width=\linewidth, trim=0.275cm 0pt 0.1cm 0pt, clip}{
        \includegraphics{tikz/pi-metrics.tikz}
      }
    \end{subfigure}
    \captionof{figure}{
      PI sets $\sH_{\pid}(\vx)$ {\figleft} and $\sH_{\pid}$ {\figright} for various sensitivity parameters $\vgamma$ and different metrics of interest; estimation errors $E_{\vx}, E$, approx. errors $E^-_{\vx}, E^-$, worst-case errors $E^+_{\vx}, E^+$.
    }
    \label{fig:pi-cartoon-depictions}
    \vspace{-1\baselineskip}
  \end{minipage}
\end{figure*} %HERE
% \fi

\vspace{-\baselineskip}
As interventions are often inaccessible for computing the risk in~\cref{eq:causal-risk}, estimating $h_{\target}$ usually relies on access to the full joint distribution $p$ of $(U, X, Y, Z)$ via \emph{back-door~adjustment}~$h^p_{\target}$~\citep{Pearl_2009},
\begin{align*}
    \func{h^{p}_{\target}}{\vx} \coloneqq 
    \EE{}{}{ \EE{}{}{Y\given \vx, Z, U} } ,
    &&
    (U, X, Y, Z)\sim\distribution{} .
\end{align*}
Alternatively, $h_{\target}$ is recovered via \emph{instrumental variables (IVs)} $Z$ satisfying~(\rnum{1})~\textbf{un-confoundedness:} $Z\indep U$, (\rnum{2}) \textbf{exclusion:} $Y\indep Z\mid X,U$~\citep{npiv}. Conditioning~\cref{eq:canonical-sem}~on~$Z$:
\begin{equation}
    \label{eq:iv-problem}
    \EE{}{}{Y - h_{\target}(X)\given \vz} = \EE{}{}{U \given \vz} = 0 .
\end{equation}
An IV estimate $h_{\iv}$ solves~\cref{eq:iv-problem} for $h_{\target}$, which now requires only observed $\distribution{XYZ}$~\citep{pmlr-v258-sun25d}.

%%%%%%%%%%%%%%%%%%%%%%%%%%%%%%%%%%%%%%%%%%
%%%%%%%%%%%%%%%%%%%%%%%%%%%%%%%%%%%%%%%%%%

\subsection{Partial identification and sensitivity analysis}
\label{sec:partial-identification-and-sensitivity-analysis}
Uniqueness of the IV solution in~\cref{eq:iv-problem} needs a \emph{complete} $Z$, i.e. $\EE{}{}{h(X)\given Z}=0$ iff $h(X)=0$ \citep{misspecified-npiv1}. Without completeness, $h_{\target}$ is unidentifiable~\citep{niki-targeted-sequential-iv}. Furthermore, the IV properties themselves may be violated~\citep{leaky-iv,pmlr-v258-penn25a,plausibly-exogenous,r2-leaky-iv}. And since $U$ is unobserved, $h_{\target}$ is generally not recoverable from $\distribution{XYZ}$ alone. Nevertheless, milder assumptions on the model in \cref{eq:canonical-sem} still allows \emph{partial identification} (PI) of $h_{\target}$ by considering \emph{all} the joint distributions $\candidateDistribution{}$ consistent with said assumptions~\citep{causal-programming},
\begin{equation*}
    \func{\sQ_{\pid}}{\distribution{XYZ}} \coloneqq 
    \br[\Big]{
    \candidateDistribution{}\in \sP_{\!\pid}
    \, : \,
    \IPM_{\sF}(
        \distribution{V}
        ,
        \candidateDistribution{V}
    )
    = 0
    } ,
    \qquad
    \IPM_{\sF}(p_V, q_V)\! \coloneqq \sup_{ f\in\sF } \abs[\Big]{\, \mathop{\mathbb{E}}\limits_{p}[f(V)] - \mathop{\mathbb{E}}\limits_{q}[f(V)] \,}
    ,
\end{equation*}
where some discrepancy metric $\IPM_{\sF}$, chosen here to be an \emph{integral probability metric} (IPM)~\citep{IMP}, matches $q, p$ on observed variables $V\subseteq\{X, Y, Z\}$ over a set of witness functions $\sF=\{f:\sV\to\R\}$ and the constraint set $\sP_{\!\pid}$ encodes our assumptions on the DGP~\cref{eq:canonical-sem}. If correctly specified, $\distribution{}\in\sP_{\!\pid}$, and the following set $\sH_{\pid}$ of candidate hypotheses contains the true $h_{\target}$,% and interval $\sH_{\pid}(\vx)\ni h_{\target}(\vx)$,
\begin{align*}
    \sH_{\pid} \coloneqq 
    \big\{
    h^{\candidateDistribution{}}_{\target}
    \, : \,
    \candidateDistribution{} \in \sQ_{\pid}
    \big\} ,
    &&
    \func{\sH_{\pid}}{\vx} \coloneqq 
    \big\{
    \func{h_{\target}^{q}}{\vx}
    \, : \,
    q \in \sQ_{\pid}
    \big\}
    ,
\end{align*}
where $\sQ_{\pid}$ abbreviates $\func{\sQ_{\pid}}{\distribution{XYZ}}$. Computing interval $\func{\sH_{\pid}}{\vx}$ at $\vx$ is often more practical than characterizing set $\sH_{\pid}$, as for convex $\sQ_{\pid}$ it amounts to solving two~constrained~optimization~problems
\begin{equation*}
\label{eq:min-bound-max-bound-optimization}
    \func{\sH_{\pid}}{\vx} = 
    \sqPr*{\;
        \min\nolimits_{\candidateDistribution{}\in\sQ_{\pid}}\func{h^{\candidateDistribution{}}_{\target}}{\vx},\quad
        \max\nolimits_{\candidateDistribution{}\in\sQ_{\pid}}\func{h^{\candidateDistribution{}}_{\target}}{\vx}
    \;} .
\end{equation*}
In either case, we want the identified sets to (\rnum{1}) contain the true solution, and (\rnum{2}) be sharp, i.e. as small as possible. Of course these respectively depend on the validity and strength of assumptions that one is willing to make within $\sP_{\!\pid}$, with point-estimation now being a special case where $\sQ_{\pid}$~is~singleton.

Hence, one may parameterize the constraints as $\sP^\vgamma_{\!\pid}$ to conduct \emph{sensitivity analysis}~by~varying~parameters $\vgamma$ to see how $\func{\sH^\vgamma_{\pid}}{\vx}, \sH^\vgamma_{\pid}$ evolve as assumptions are~relaxed/tightened~\citep{neural-sensitivity-pid}.

Lastly, since we are now discussing hypothesis sets, we define the three appropriate evaluation metrics
\begin{align*}
    \func{ E^- }{\sH_{\pid}} \coloneqq \smash[t]{\min_{ h\in\sH_{\pid}}} E(h)
    ,
    \qquad
    \func{ \overline{E} }{\sH_{\pid}} \coloneqq 
    \mathbb{E}_{h\sim{\rm Unif}(\sH_{\pid})}[
        E( h )
    ]
    ,
    \qquad
    \func{ E^+ }{\sH_{\pid}} \coloneqq \smash[t]{\max_{ h\in\sH_{\pid}}} E(h)
    .
\end{align*}
The \emph{approximation~error} $E^-$
measures how far the target $h_{\target}$ is from~$\sH_{\pid}$~\citep{estimation-approximation}, enabling a notion of approximate validity: $\sH_{\pid}$ is \emph{$\epsilon$-approximately valid}, or just \emph{$\epsilon$-valid}, whenever $E^-(\sH_{\pid})\leq \epsilon$~\citep{e-valid}.
The \emph{worst-case error} $E^+$ gives our primary notion of bound \emph{informativeness} for decision making, upper bounding the excess risk of candidates within $\sH_{\pid}$ relative to $h_{\target}$. We additionally introduce the \emph{average error} $\overline{E}$~\citep{avg-error} as an analytic surrogate for informativeness, approximating worst-case error up to a universal constant factor~(see~\cref{lemma:informativeness-surrogate}).

Similarly, the intervals $\func{\sH_{\pid}}{\vx}$ are evaluated using $E^-_{\vx}(\func{\sH_{\pid}}{\vx})$, $\overline{E}_{\vx}(\func{\sH_{\pid}}{\vx})$ and $E^+_{\vx}(\func{\sH_{\pid}}{\vx})$. We define these analogously to $E^- \!\! ,\overline{E}$ and $E^+$ respectively, but with $R_{\vx}$ as the primitive instead of~$R$.

\paragraph{Choice of constraints.} The construction of $\sP_{\!\pid}$ depends on domain knowledge. At the agnostic extreme, \citeauthor{manski}'s~\citeyearpar{manski} ``no-assumption'' bounds use only a bounded support of $Y$ to constrain $\sH_{\pid}(\vx)$ for binary $\vx$. Other approaches bound the strength of unmeasured confounding through odds ratios~\citep{rosenbaum2002observational}, propensity scores in the Marginal Sensitivity Model (MSM)~\citep{MSM}, or the share of variance explained by hidden confounders in the partial~$r^2$ model~\citep{making-sensemakr,long-story-short}. Robust optimization (RO) formulations instead take $\sQ_{\pid}$ to be a total-variation~\citep{hallucinatedDRO,dro-pid} or Wasserstein~\citep{implicit-g-constrained-pid} ball around $\distribution{XYZ}$. Alternatively, one may directly model the \emph{response function} $\func{H_{\target}}{\cdot}\coloneqq \func{h_{\target}}{\cdot} + U + \residual$~\citep{causal-programming}, e.g., with GANs matched to $\distribution{XYZ}$~\citep{gan-pid}. Many methods further sharpen bounds via conditional independences on auxiliary variables~\citep{niki-targeted-sequential-iv,CopSens}. A notable example is IV-based PI~\citep{balke1997bounds}, where $U$ may affect $Y$ non-additively, unlike \cref{eq:canonical-sem}. Neural variants extend this to continuous, high-dimensional treatments and IVs~\citep{iv_niki,gan-pid,causal-programming,implicit-g-constrained-pid,vacuous3}.

Perhaps the closest to our work are \emph{shape constraints}~\citep{shape-constraints}, which sharpen PI by restricting the class of structural or response functions $h_{\target}$ or $H_{\target}$. Examples include positivity~\citep{positivity}, monotonicity~\citep{monotone-treatment-response,monotone-iv}, convexity~\citep{convexity}, curvature~\citep{curvature}, smoothness~\citep{smoothness}, sparsity~\citep{sparse-iv}, and other regularities~\citep{implicit-g-constrained-pid}, each narrowing the set of compatible DGPs. In the same vein, we propose known invariances of $h_{\target}$ as a natural shape constraint.

%%%%%%%%%%%%%%%%%%%%%%%%%%%%%%%%%%%%%%%%%%
%%%%%%%%%%%%%%%%%%%%%%%%%%%%%%%%%%%%%%%%%%

\begin{wrapfigure}{r}{0.44\textwidth}
  \vspace{-40pt} % tweak as needed
  \centering
  % \hspace{-1em}
  \begin{subfigure}[c]{0.49\linewidth}
    \centering
    \includegraphics[width=\linewidth]{tikz/da-graph.tikz}
    \caption{Data augmentation}
    \label{fig:da-graph}
  \end{subfigure}
  \hfill
  \begin{subfigure}[c]{0.49\linewidth}
    \centering
    \includegraphics[width=\linewidth]{tikz/transformation-intervention-graph.tikz}
    \caption{Soft-intervention}
    \label{fig:intervention-da-equivalence}
  \end{subfigure}
  \caption{
      For $\sT$-inv. $h_{\target}$, {\captiona} DA on observed data induces an identical distribution to {\captionb} a soft intervention ${\rm do}(X = \smash{\tX})$ on~\cref{eq:canonical-sem}.
  }
  \label{fig:da-intervention-graph}
  \vspace{-0pt}
\end{wrapfigure}
\subsection{Known data symmetries and invariance}
\label{sec:known-data-symmetries-and-invariance}
\emph{Symmetries} are ubiquitous in scientific modeling, describing structural properties that do not change under certain data transformations \citep{bronstein2021geometric}. Formalized as a map $\alpha: \sX\times\sT\rightarrow \sX$ parameterized by $\vtau\in\sT$ with $\vtau\vx\coloneqq\alpha(\vx, \vtau)$ as shorthand, we say that $h_{\target}$ is \emph{invariant} under $\sT$, or \emph{$\sT$-invariant}, if
\begin{align*}
    \func{h_{\target}}{\vtau\vx} = \func{h_{\target}}{\vx} ,
    &&
    \forall
    \;
    (\vx, \vtau)\in \sX\times\sT .
\end{align*}
This captures the notion that certain transformations preserve the causal relationship $X\rightarrow Y$. E.g., spatial translations in physics~\citep{symmetryPhysics}, permutations of exchangeable units~\citep{deepSets}, semantics-preserving text transformations~\citep{symmetryNLP}, and reverse-complements in genomics~\citep{geometricCausalModels}. In image classification, for instance, a rotated image $\vtau\vx$ of a cat $\vx$ remains a cat~\citep{groupCNN}.

In practice, symmetries may hold only approximately, and $h_{\target}$ is \emph{$(\eps, \sT)$-inv.} if~\citep{invarianceError0}
\begin{equation*}
    \sup\nolimits_{\vx, \vtau}\; \abs{
        h_{\target}(\vx) - h_{\target}(\vtau\vx)
    } \leq \eps .
\end{equation*}
Symmetry assumptions are well-established across~many~applications and routinely used to improve statistical estimation. For example, given some probability measure $\distribution{T}$ on $\sT$, \emph{data augmentation} (DA) reduces estimation variance by generating multiple $m$ augmentations $\{ (\vtau_{i, j} \vx_i,\; y_i) \}_{j\in [m]}, \vtau_{i, j}\sim p_T$ for each sample $(\vx_i, y_i, \cdot)\in\sD$, increasing data variability for~statistical~risk~evaluation~\citep{invariance_benefit,group-theory-da-jmlr,savkin2020adversarial},
\begin{align}
    \label{eq:data-augmentation-risk}
    R_{ \term }( h ) \coloneqq \mathbb{E}\sqPr[\big]{
        \ell\pr[\big]{Y, h( \stX ) }
    } ,
\end{align}
where we denote $\stX\coloneq TX$ as random perturbation of $X$ under $T\sim p_T$, such as depicted in~\cref{fig:da-graph}.

While DA is canonically used to mitigate finite-sample estimation variance, our focus is primarily on the infinite-sample setting, and we present~\cref{eq:data-augmentation-risk} and subsequent theoretical results in that context. Nonetheless, increasing sample size via DA also bears on our work, a point we shall briefly discuss.

Despite their success in statistical inference, symmetry priors remain understudied in causal inference~\citep{causal-data-augmentation}. This gap is particularly consequential for~PI, where bounds are often vacuous and practitioners seek to exhaustively impose constraints encoding domain knowledge to sharpen them. In the language of \cref{sec:partial-identification-and-sensitivity-analysis}, $\sT$-inv. is precisely a shape constraint on $h_{\target}$. Symmetry assumptions hence provide a natural and powerful addition to this toolkit, especially in domains where they are already established. In what follows, we provide a first analysis of their application~to~sharpen~PI.

%%%%%%%%%%%%%%%%%%%%%%%%%%%%%%%%%%%%%%%%%%
%%%%%%%%%%%%%%%%%%%%%%%%%%%%%%%%%%%%%%%%%%

\section{A Symmetry Shape Constraint}
\label{sec:explicit-constraint}
Following the shape constraint tradition in PI, the most direct way is to incorporate $\sT$-inv. as a solver feasibility requirement. E.g., impose an \emph{invariance error} constraint with $\eps$ as a sensitivity parameter:
\begin{align*}
    E_{\inv}(h) \coloneqq \mathbb{E}_p\sqPr[\big]{
        \abs{
            \func{h}{X} - h( \stX )
        }^2
    },
    &&
    \sP_{\!\pid+\inv} \coloneqq 
    \br*{
    \;
    \candidateDistribution{} \in \sP_{\!\pid}
    \, : \,
    E_{\inv}( h^{\candidateDistribution{}}_{\target} ) \leq \eps^2
    \;
    } .
\end{align*}
yielding the \emph{sharp} set $\sH_{\pid+\inv} = \sH_{\pid}\cap\sH^{\eps}_{\inv}$, with $\sH^{\eps}_{\inv}\coloneqq\br*{h\in \sH : E_{\inv}(h)\leq\eps^2}$ (\cref{corollary:sharp-inv-constraints}, cf.~\cref{app:explicit-vs-implicit}). Other $E_{\inv}$ functionals may serve equally~\citep{fanny}, as does restricting $\sH$ to be invariant by design~\citep{groupCNN}, though the latter can sit awkwardly with standard PI solvers and may not admit sensitivity parameters.

While most of our experiments focus on settings where $\func{\sH_{\pid+\inv}}{\vx}$ is computable via convex programming, this constraint formulation is compatible with modern deep learning-based PI. Since $E_{\inv}(\cdot)$ is differentiable and amenable to Monte Carlo evaluation, it can be incorporated as a regularizer in augmented Lagrangian or gradient-based solvers~\citep{causal-programming,implicit-g-constrained-pid,gan-pid}. Nevertheless, there are operational limitations: it requires white-box access to the solver's objective, which precludes the specialized or black-box software where the optimization is fixed. Furthermore, optimizing non-convex invariance~constraints for complex, high-dimensional data can be expensive and notoriously unstable~\citep{pid-representation-niki,causal-programming}. This motivates our second approach; an operationalization that touches only the data,~not~the~solver.

%%%%%%%%%%%%%%%%%%%%%%%%%%%%%%%%%%%%%%%%%%
%%%%%%%%%%%%%%%%%%%%%%%%%%%%%%%%%%%%%%%%%%

\section{A Change of Measure under Symmetry Transformations}
\label{sec:implicit-constraint}
\paragraph{DA as a change of measure.}
Where~\cref{sec:explicit-constraint} constrains the feasible set, we instead change the measure $p_{XYZ}\to p_{\stX YZ}$ under which PI is posed while leaving the solver untouched, composing with it as
\begin{align*}
    \func{\sQ_{ \tpid }}{\distribution{XYZ}} \coloneqq \func{\sQ_{\pid}}{\distribution{\stX Y Z}} .
\end{align*}
This change of measure reweights candidate evaluations, pruning previously feasible functions while, crucially, preserving validity across broad families of standard solvers. We now isolate the two mechanisms driving these effects: variance inflation (sharpening) and simulated intervention~(validity).

\paragraph{Variance inflation.}
In IV methods, “strong” IVs that induce more exogenous variation in $X$ yield sharper PI bounds than weak ones \citep{iv_niki,causal-programming}. So we pose~the~question:
{\setlength{\topsep}{0pt}
\begin{quote}
\itshape
\hfill Can the synthetic exogenous variation introduced by DA likewise sharpen PI?\hfill\vphantom{.}
\end{quote}}
We answer yes. Fundamentally, PI uncertainty is an \emph{attribution problem}; the solver must split observed $(X, Y)$ covariation between $h_{\target}$ and $U$. Only variation in $(X, Y)$ whose relation to $U$ is \emph{known}, e.g. an IV $Z$, can break the tie. DA adds exactly such variation since $T$ is exogenous by construction and cannot be attributed to $U$. This reduces ambiguity in PI, giving sharper, more informative bounds. This effect is formalized in~\cref{sec:better-bounds}, here we give a notion of variance inflation~itself.

\emph{Fit vs. evaluation geometries.} Center $Y$ w.l.o.g., then $\mathbb{E}[h_{\target}(X)]=0$. So \emph{estimation} of $h_{\target}$ lives only in the centered \emph{shape space} $\sH_X\coloneqq\br*{ h\in\sH \colon \E[h(X)] = 0 }$, which we equip with covariance forms
\begin{align*}
    \langle f, g \rangle_X \coloneqq \Cov(f(X), g(X))
    ,
    \qquad
    \sqNorm{h}_X
    \coloneqq
    \Var(h(X))
    ,
\end{align*}
and $\pr*{\sH_{\stX}, \langle \cdot, \cdot \rangle_\stX}$ likewise under $p_\stX$. Of course, \emph{evaluation} is still reported in the ambient $L_2(X)$-norm; in particular, the causal excess risk under a sq.-loss $\ell(a, b)\!\coloneqq\!(a\! -\!b)^2$ decomposes orthogonally,
\begin{align*}
    \label{eq:errors}
    E(h) = \bigl\| h - h_{\target} \bigr\|^{2}_{L_2(X)}
    =
    \underbrace{
        \sqNorm{
            h - h_{\target}
        }_X
    }_{\text{shape; lives on $\sH_X$}}
    +
    \;\;\;
    \underbrace{
        \pr[\big]{
            \EE{}{}{
                h(X)
            } - 
            \EE{}{}{
                h_{\target}(X)
            }
        }^2
    }_{\text{level; lives off $\sH_X$}}
    .
\end{align*}
Fitting in $\pr*{\sH_{\stX}, \langle \cdot,\cdot\rangle_{\stX}}$ can change the shape. But $\stX$-centered candidates are not $X$-centered, so a level penalty appears. This forces a separate treatment of covariance and mean, which we do under:

\begin{assumption}
\label{assumption:hypothesis-class}
Let $\sH\subseteq L_2(X)\cap L_2(\stX)$ be closed under constant shifts $\sH + \R1 = \sH$.~Optionally:\footnote{Any class with a free intercept satisfies the base clause. Clause~(a) removes statistical approx. bias to isolate confounding bias, our focus. Clause~(b) is essentially free under DA; sample $\Id\in\sT$, or mix-in $p_{XYZ}$. It lets the $\sH_X$ Hilbert structure carry over to $\sH_{\tX}$. Clause (c) gives bounded~evaluation~$h(\vx)$, without which PI intervals $\sH_{\pid}(\vx)$ for continuous $X$ are known to be vacuous~\citep{iv_niki,vacuous1,vacuous2,vacuous3}. Crucially, finite dimension is only needed to compute set volume/avg.-error in closed~form.}
\\
\phantomsection\label{assumption:hypothesis-class:no-bias}
\noindent
\emph{(a) Realizability:} Suppose $\sH$ contains both $\,h_{\target}(X),\, \mathbb{E}[Y | X]$, and post-DA $\,\mathbb{E}[h_{\target}(X) | \tX],\, \mathbb{E}[Y | \tX]$.
\\
\phantomsection\label{assumption:hypothesis-class:abs-cont}
\noindent
\emph{(b) Absolute continuity:} $\,p_X\, \ll\, p_{\stX}\,$, i.e. the post-DA distribution covers the baseline distribution.
\\
\phantomsection\label{assumption:hypothesis-class:finite-dim}
\noindent
\emph{(c) Finite dim. Hilbert:} $\sH$ is linear with ${\rm dim}(\sH_X) = k < \infty$, and $\langle\cdot,\cdot\rangle_{X}$ positive-definite over~$\sH_X$.
\end{assumption}

\emph{Variance vs. mean-shifts.}
We ask how much a solver-unit direction actually weighs under evaluation, on average. Sum the $L_2(X)$-norms over \emph{any} $\tX$-orthonormal basis $\br*{e_i}\subset\sH_{\stX}$~(basis indep., cf.~\cref{app:proofs}),
\begin{align*}
    \sum\nolimits_i \sqNorm{e_i}_{L_2(X)}
    \quad
    =
    \quad
    \sum\nolimits_i \sqNorm{e_i}_{X}
    +
    \sum\nolimits_i \pr[\big]{ \E[ e_i(X) ] - \E[ e_i(\tX) ] }^2
    \quad
    =
    \quad
    \tr{\varShift} + \IPM^2_\sB(p_X, p_{\stX}).
\end{align*}
Here, $\tr{\varShift} \coloneqq \sum_i \sqNorm{e_i}_X$ is the $X$-variance carried by $k$ directions that are unit-sized under $\tX$. And $D^2_{\sB}$ is the squared \emph{maximum mean discrepancy} (MMD) over a unit $\stX$-ball $\sB \coloneqq \br*{h \in \sH_{\stX} \colon \norm{h}_{\stX} \leq 1}$~\citep{MMD}, collecting the $e_i$ mean-shifts by Pythagoras. Read them as a lever and its price. Whenever $\tr{\varShift}/k \leq 1$, a typical unit $\tX$-direction carries less than unit $X$-variance; this is our notion of \emph{variance inflation}, quantifying the added exogenous variation invoked above. We later connect it directly to bound sharpening in~\cref{proposition:da-sharper-bounds} with a price; the familiar $D_{\sB}$ from \emph{out-of-distribution} (OOD) generalization literature~\citep{domain-generalization}, since we fit under $p_{\stX}$ yet evaluate OOD on $p_X$. However, DA must first be licensed for PI at all, which we discuss next.

\paragraph{Simulated-intervention.}
We now give an observation that licenses DA for PI: for $\sT$-inv. $h_{\target}$, DA realizes a \emph{soft intervention} $\intervention{X}{\stX}\vphantom{\tX}$ on~\cref{eq:canonical-sem}~\citep{causal-data-augmentation,ilse2020selecting,soft-intervention1,soft-intervention2,soft-intervention3}; the structural equation and causal graph are preserved up to relabeling $X\to\tX$, as in~\cref{fig:da-intervention-graph}. Meanwhile, confounding is diluted by the \emph{data processing inequality} (DPI) $\var(\eE{}{}{U | X})\geq \var(\eE{}{}{U | \tX})$ on the chain $U\!\leftrightarrow X\!\to \tX$. Crucially, this dilution is observable through the ratio $\rho\coloneqq \tsigma^2/\sigma^2$ of MMSE $\sigma^2\coloneqq \min_{h\in\sH} R_{\erm}(h)$ and $\tsigma^2\coloneqq \min_{h\in\sH} R_{\term}(h)$, with $\rho\geq 1$ since a noisy $\tX$ cannot predict $Y$ better than $X$. Concretely:

\begin{restatable}[soft intervention with DA]{lemma}{softIntervention}
\label{lemma:da-soft-intervention}
For $\sT$-inv. $h_{\target}$, DA acts on~\cref{eq:canonical-sem} as a \emph{soft-intervention}:
\begin{enumerate}[
    label=(\Alph*),
    leftmargin=*,
    itemsep=2pt,
    topsep=-4pt,
    parsep=0pt
]
    \item \emph{Structural equation:}
    Trivially, post-DA variables $(\residual, U, \tX, Y, \tZ)$ with $\tZ\in\br*{Z, (Z, T)}$ follow:
    \begin{align*}
        Y = h_{\target}(\tX) + U + \residual
        ,
        &&
        \EE{}{}{U}
        =
        \EE{}{}{\residual}
        =
        0
        ,
        &&
        \residual
        \indep
        (U, \tX, \tZ)
        .
    \end{align*}
    \item \emph{Markov transfer:}
    Let $p$ be Markov w.r.t. a DAG $\graph$ over $V \coloneqq \br*{U, X, Y, Z_1, Z_2, \dots}$ where $\NonDesc{X} = V\backslash\br*{X,Y}$. Then, $p_{\stV}$ is Markov w.r.t. $\tgraph$, where $\tV, \tgraph$ relabel $X$ to $\tX$ in $V, \graph$.
    \item \emph{Attenuated confounding:} Under {\asmOneA}, let biases $b\coloneqq h_{\erm} - h_{\target}$ and $\stb\coloneqq h_{\term} - h_{\target}$. Then,
\begin{align*}
    \sqNorm{b}_{X} - \sqNorm{\stb}_{\stX}
    \;=\;
    \sigma^2 (\rho - 1)
    \;\geq\; 0
    .
\end{align*}
\end{enumerate}
\end{restatable}

Hence, broad families of shape constraints~\citep{shape-constraints}, graphical restrictions~\citep{balke1997bounds,nonparametric-graphical} or confounding bounds~\citep{neural-sensitivity-pid} that are valid under $p_{X Y Z}$ continue to remain valid under $p_{\stX Y Z}$. Subsequent sections instantiate two such canonical PI models to study the specific identified set evaluation metrics we introduced in~\cref{sec:partial-identification-and-sensitivity-analysis}.
\paragraph{Fragility of point-estimation.} To motivate the PI application of DA, we first understand its limitations in point-estimation. After all,~\cref{lemma:da-soft-intervention} does beg the question why this should not work.

\begin{minipage}{\columnwidth}
\begin{restatable}[estimation with DA]{proposition}{causalEstimation}
\label{proposition:da-causal-estimation}
Under {\asmOneA}, $\sT$-inv. $h_{\target}$ and \emph{bias direction} $e_{\stb}\coloneqq \left.\tb\right/\norm{\stb}_{\stX}$,
\begin{align*}
    E(h_{\erm})
    -
    E( h_{ \term } )
    \quad
    =
    \quad
    \underbrace{
    \sigma^2
        (
            \rho
            -
            1
        )
        \cdot
        \sqNorm{e_{\stb}}_{L_2(X)}
        \vphantom{D^2_{\br*{\pm e_{\stb}}}(p_X, p_{\stX})}
    }_{
        \text{
            gain
        }
    }
    \;
    +
    \;
    \underbrace{
        (
                1 - \sqNorm{e_{\stb}}_{L_2(X)}
        )
        \cdot
        E(h_{\erm})
    }_{
        \text{wild-card}
    }
    .
\end{align*}
\end{restatable}
\end{minipage}

\Cref{proposition:da-causal-estimation} reveals three fragilities: (\rnum{1})~\textbf{bias-direction dependence:} attenuation $\DPI\geq1$ alone does not
guarantee improvement, which hinges on whether $\sqNorm{e_{\stb}}_{L_2(X)}\leq1$ for an
unobserved bias direction, (\rnum{2})~\textbf{misspecification:} if symmetry is misspecified, the gain term may itself become a loss, and (\rnum{3})~\textbf{identifiability:}  $h_{\target}$ remains unidentifiable regardless.
All three defects are manageable~under~PI.

\subsection{Partial identification with data augmentation}
\label{sec:better-bounds}
We show that DA yields three PI guarantees, each resolving one fragility above, as summarized in~\cref{table:da-point-vs-partial-identification}. \textbf{Validity}~(\cref{theorem:valid-bounds-with-da})~\textbf{:}
Crucially, PI validity follows from $\sT$-invariance alone, addressing~(\rnum{3}). \textbf{Sharpness / informativeness}~(\cref{proposition:da-sharper-bounds,theorem:informative-bounds-with-da})~\textbf{:}
We show that PI sharpening is governed by \emph{global} quantities like $\rho$, $\tr{\varShift}$, $D_{\sB}$, making sufficient conditions observable, unlike~(\rnum{1}). \textbf{Robustness}~(Thm. \hyperlink{thm-general-validity}{3}~and~\cref{remark:robust-bounds-with-da})~\textbf{:} To make these guarantees robust to $\eps$-misspecified
symmetry, we deploy a simple $\eps$-pad-and-intersect post-processing, resolving~(\rnum{2}).

\begin{figure*}[t]
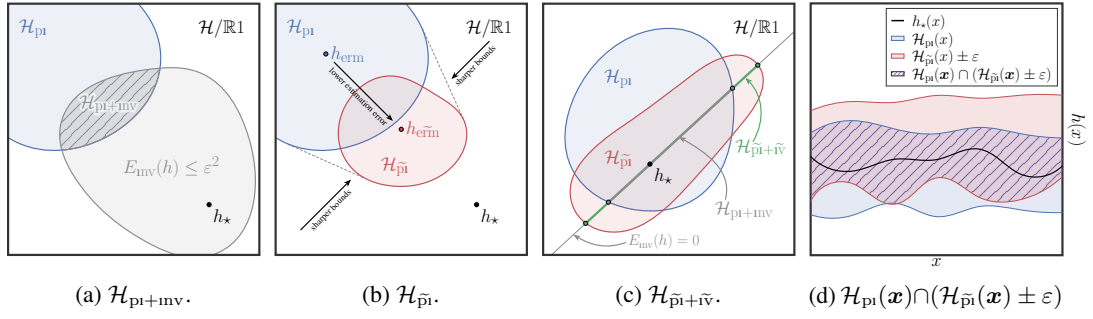

  \centering
  \begin{subfigure}[c]{0.24\textwidth}
    \centering
    \includegraphics[width=1\linewidth]{tikz/inv-pi.tikz}
    \caption{$\sH_{\pid+{\inv}}$.}
    \label{fig:inv-pi-cartoon}
  \end{subfigure}%
  \hfill
  \begin{subfigure}[c]{0.24\textwidth}
    \centering
    \includegraphics[width=1\linewidth]{tikz/da-pi.tikz}
    \caption{$\sH_{\tpid}$.}
    \label{fig:da-pi-cartoon}
  \end{subfigure}
  \!\!\hfill\!
  \begin{subfigure}[c]{0.24\textwidth}
    \centering
    \includegraphics[width=1\linewidth]{tikz/da-iv-pi.tikz}
    \caption{$\sH_{\tpid+\tiv}$.}
    \label{fig:da-iv-pi-cartoon}
  \end{subfigure}%
  \hfill
  \begin{subfigure}[c]{0.24\textwidth}
    \centering
    \includegraphics[width=1.1\linewidth]{tikz/pi-da-pi-intersection.tikz}
    \caption{$\sH_{\pid}(\vx) \cap\pr*{ \sH_{\tpid}(\vx)\pm\eps }$}
    \label{fig:pi-da-pi-intersection-cartoon}
  \end{subfigure}
  \caption{
    Illustration of identified sets: {\captiona}~Invariance shape constraint prunes the baseline set~(\cref{sec:explicit-constraint}). 
    {\captionb}~DA pre-processing yields a sharper (\cref{proposition:da-sharper-bounds}), better-centered (\cref{proposition:da-causal-estimation}) set~(\cref{sec:implicit-constraint}).
    {\captionc}~$T$-as-IV enforces symmetry~(\cref{sec:da-as-iv}, cf.~\cref{app:explicit-vs-implicit}).
    {\captiond}~$\eps$-pad-and-intersect dominates under~arbitrary~DA~(\cref{remark:robust-bounds-with-da}).
  }
  \label{fig:da-pi-cartoon-depictions}
  \vspace{-1\baselineskip}
\end{figure*} %HERE

Beyond these population effects, DA has the familiar statistical-efficiency benefit of growing effective sample size, mitigating sampling variation and finite-sample error, key sources of PI uncertainty~\citep{manski-confidence}.
To make these claims explicit, we instantiate the following PI model:

\begin{assumption}[a bounded confounding sensitivity model]
\label{assumption:partial-r2}
For \emph{confounding budgets} $\gamma_{\vx}, \gamma_y$, let
\begin{align*}
\sP_{\!\pid}
\coloneqq\br*{
\,
\candidateDistribution{}
\,
\colon
\,
\begin{aligned}
    \pR^2_{U\sim X} \leq \gamma_{\vx} ,
    \\
    \pR^2_{U\sim Y | X} \leq \gamma_{y} \phantom{,}
\end{aligned}
\,
}
,
&&
\pR^2_{A\sim B | C}
\coloneqq
1 - \frac{ \EE{}{}{\var(A \given B, C) } }{ \EE{}{}{ \var( A \given C ) } }
\leq 1
,
\end{align*}
where $\smash[t]{\pR^2_{A\sim B \mid C}}$ is the  \emph{partial-$r^2$}, i.e. fraction of variation in $A$ explained by $B$ after controlling~for~$C$.
Further, make the IPM $\IPM_{\sF}$ test all quadratic functions $\sF\coloneqq\br*{ f^2 \colon f\in \Span{\sH} + \R Y }$ of $(X, Y)$.
\end{assumption}

\Cref{assumption:partial-r2} adopts the widely used partial-$r^2$ model by~\cite{making-sensemakr,long-story-short}. While we employ this model in our analysis, we do not necessarily restrict~ourselves to it: several PI families in~\cref{sec:partial-identification-and-sensitivity-analysis} impose equivalent constraints, reducing to ellipsoidal sets of the form

% \vspace{0.1\baselineskip}

\begin{restatable}[characterizing the identified set]{lemma}{characterizingIdentifiedSets}
\label{lemma:characterizing-identified-sets}
Under~\cref{assumption:partial-r2} and~{\asmOneShortA} or~{\asmOneShortC}, $\sH_{\pid}$ is the covariance~ball
\begin{align*}
    \sH_{\pid} = \br*{
    \:\!
    h\in\sH_X
    \:\!
    :
    \:\!
    \norm[\big]{ h - h_{\erm}}^2_X \leq \sigma^2 \gamma
    \:\!
    } ,
\end{align*}
within $\sH_X$, centered at $h_{\erm}$ with radius $\sigma\sqrt{\gamma}$, and an \emph{effective confounding budget} of $\gamma\coloneqq\frac{\gamma_{\vx} \gamma_y}{1 - \gamma_{\vx}}$.
\end{restatable}

The post-DA analysis now reduces to three questions about this ellipsoid: whether it covers $h_{\target}$, does its size contract, and if its center/radius tradeoff lowers avg./worst-case errors. We treat these in~turn.

%%%%%%%%%%%%%%%%%%%%%%%%%%%%%%%%%%%%%%%%%%
%%%%%%%%%%%%%%%%%%%%%%%%%%%%%%%%%%%%%%%%%%

\begin{wraptable}{r}{0.49\textwidth}
  \vspace{-12.5pt}
  \centering
  \caption{Point vs. \emph{partial} identification under DA.}
  \label{table:da-point-vs-partial-identification}
    \resizebox{\linewidth}{!}{%
        \begin{tabular}{@{}l@{}c@{}c@{}c@{}c@{}}
          \toprule
          \makecell[l]{Condition} & \hspace{12pt} & \makecell{Identification metric} & & \makecell{Safeguard} \\
          \midrule
          \\[-5ex]
          \makecell[c]{
              $\sT\!$-inv. $h_{\target}\;\;$
          } & \makecell{
              \\[1ex]
              \llap{$\Longrightarrow$\hspace{-5pt}}
              \\[-1ex]
              \llap{\rotatebox{-25}{$\Longrightarrow$}\hspace{-5pt}}
          } & \makecell{\emph{validity} (\cref{theorem:valid-bounds-with-da}) $\color{deep_green}\bm\checkmark$} & \makecell{$\bm\}\;$} & \makecell{$\eps$-pad (Thm. \hyperlink{thm-general-validity}{3})} \\[-3ex]
          & & \makecell{\\} &  &  \\[-1ex]
          \makecell[c]{
            $\norm{e_{\stb}} \leq 1\;\;$
          }
          & \makecell{
          \llap{$\Longrightarrow$\hspace{-5pt}}
          } & \makecell{{estimation error} $\color{deep_green}\bm\downarrow$\\ (\cref{proposition:da-causal-estimation})} & \multirow{2}{*}{\smash{\raisebox{1.2ex}{$\!\!{\pmb{\left.\rule{0pt}{1.2em}\right\}}}$}}} & {\color{deep_red}\Large\ding{55}} \\[-0ex]
          &  & $\Downarrow$ & & \\
           & & \makecell{\emph{informativeness} $\color{deep_green}\bm\uparrow$\\ (\cref{theorem:informative-bounds-with-da})} & \multirow{3}{*}{\smash{\raisebox{-1ex}{$\!\left.\rule{0pt}{2.75em}\right\}$}}} & \multirow{4}{*}{\makecell{intersect with\\ baseline set\\ (\cref{remark:robust-bounds-with-da})}} \\[-0.25ex]
          &
          & $\Uparrow$ & & \\[-0.25ex]
          \makecell[c]{
            $  \Omega \leq 1\;\;$
          }
          &
          \makecell{
          \llap{$\Longrightarrow$\hspace{-5pt}}
          }
          & \makecell{\emph{sharpness} $\color{deep_green}\bm\uparrow$\\ (\cref{proposition:da-sharper-bounds})} & & \\
          \bottomrule
        \end{tabular}%
    }
  \vspace{-50pt}
\end{wraptable}

\paragraph{Validity.} First and foremost, we establish post-DA validity. In terms of approximation error:

\begin{minipage}{0.485\columnwidth}
\begin{restatable}[valid bounds with DA]{theorem}{validBoundsWithDA}
\label{theorem:valid-bounds-with-da}
Let $h_{\target}$ be $\sT$-inv. under {\asmOneA} and \ref{assumption:partial-r2}. Then, the post-DA set $\sH_{\tpid}$ is exactly valid whenever the baseline set $\sH_{\pid}$ is $\epsilon$-valid with $
\epsilon \leq \sigma^2\pr*{ \sqrt{ \DPI - 1 + \gamma\DPI } - \sqrt{\gamma} }^2
$.
\end{restatable}
\end{minipage}

Consequently, $\sH_{\tpid}$ \emph{repairs} the $\epsilon$-validity of $\sH_{\pid}$ since the previously insufficient budget $\gamma$ can now accommodate the diluted bias. Crucially, this follows from the simulated intervention in \cref{lemma:da-soft-intervention} alone, and does not rely on conditions over unobserved bias like \cref{proposition:da-causal-estimation}. The only untestable requirement of $\sT$-invariance is later relaxed to $(\eps, \sT)$-invariance to allow for sensitivity analysis over the assumed $\eps$-misspecification of symmetry.

%%%%%%%%%%%%%%%%%%%%%%%%%%%%%%%%%%%%%%%%%%
%%%%%%%%%%%%%%%%%%%%%%%%%%%%%%%%%%%%%%%%%%

Having established validity, we now show DA can make bounds \emph{better}, even under relaxed conditions.
% \vspace{-1.5em}
\paragraph{Sharpness.}
While the centroid drift in \cref{proposition:da-causal-estimation} hinges on the bias direction, the identified set's \emph{size} depends on \emph{global} distribution shifts. This yields a milder, data-observable improvement~condition:

\begin{restatable}[sharper bounds with DA]{proposition}{daSharperBounds}
\label{proposition:da-sharper-bounds}
 Denoting $L_2(X)$-volume by $\abs{\cdot}$, under {\asmOneShortBC} and \ref{assumption:partial-r2},
\begin{align*}
    \pr*{
        \frac{ \abs*{\sH_{\tpid}} }{ \abs*{\sH_{\pid}} }
    }^{2/k}
    \leq
    \quad
        \frac{
            \mathbb{E}\bigl[
                \abs*{\func{\sH_{\tpid}}{X}}^2
            \bigr]
        }{
            \mathbb{E}\bigl[
                \abs*{\func{\sH_{\pid}}{X}}^2
            \bigr]
        }
    \quad
    =
    \;\;
    \underbrace{
        \smash[t]{\rho}
        \vphantom{\rho \tr{\varShift} D^2_\sB(p_X, p_\stX)}
    }_{\text{over-budget}}
    \!\!\!\!\!\!\!\!
    \;\;
    \cdot
    \;\;
    \sqPr[\Big]{
    \;
    \underbrace{
        \tr{\varShift}
        \vphantom{\rho \tr{\varShift} D^2_\sB(p_X, p_\stX)}
    }_{\text{lever}}
    \;
    +
    \;
    \underbrace{
        D^2_\sB(p_X, p_\stX)
        \vphantom{\rho \tr{\varShift} D^2_\sB(p_X, p_\stX)}
    }_{\text{penalty}}
    \;
    }
    \left.
    \vphantom{\Big]}
    \right/
    k
    \quad
    \eqcolon
    \quad
    \Omega
    .
\end{align*}
\end{restatable}

\Cref{proposition:da-sharper-bounds} ties sharpening to competing factors. Noise inflation $\DPI\geq 1$ and mean-shift $D^2_{\sB} \geq 0$ strictly expand the set. Thus, bounds tighten when inflated variance $\tr{\varShift}/k\leq 1$ drives $\Omega$ small enough to satisfy $\Omega \leq 1$. Intuitively, \cref{lemma:characterizing-identified-sets} enforces a \emph{variance cap} on allowable bias. DA can stretch admissible candidates beyond this fixed $\gamma$-budget, pruning them from the set. Notably, the soft intervention~\cref{lemma:da-soft-intervention} keeps the true bias $\tb$ within budget even as other candidates~are~pruned. Finally, as the inherited $\gamma$ over-budgets the now diluted confounding by $\DPI$, one may re-calibrate to $\tgamma\coloneqq\gamma/\DPI$ if the solver exposes $\gamma$, keeping validity while reducing~\cref{proposition:da-sharper-bounds}'s term to~$\Omega=(\tr{\varShift}+D^2_{\sB})/k$.
% \vspace{-0.5em}
\paragraph{Informativeness.}
While smaller sets are desirable, we also gauge the ``quality'' of set candidates:

\begin{minipage}{\columnwidth}
\begin{restatable}[informative bounds with DA]{theorem}{informativeBoundsWithDA}
\label{theorem:informative-bounds-with-da}
Under~{\asmOneShortBC} and~\ref{assumption:partial-r2}, the following holds for~$d=k$,
\begin{align*}
    \overline{E}(\sH_{\pid})
    -
    \overline{E}(\sH_{\tpid})
    \quad
    =
    \quad
    \underbrace{
        E(h_{\erm})
        -
        E(h_{\term})
    }_{
        \substack{
            \text{centroid drift (\cref{proposition:da-causal-estimation})}
        }
    }
    \;
    +
    \;
    \;   
    \frac{\sigma^2\gamma k}{d + 2}
    \cdot
    \!
    \underbrace{
        \pr{
            1
            -
            \Omega
        }
    }_{
        \substack{
            \text{set size (\cref{proposition:da-sharper-bounds})}
        }
    }
    \!
    \!
    \quad
    \eqcolon
    \quad
    f(d),
\end{align*}
And interval improvement averaged over queries: 
\smash{$\EE{X\sim p_X}{}{
    \func{\overline{E}_{X}}{
        \func{\sH_{\pid}}{X}
    }
-
    \func{\overline{E}_{X}}{
        \func{\sH_{\tpid}}{X}
    }
}
= f(1)$}.
\end{restatable}
\end{minipage}

The decomposition isolates two mechanisms: the centroid drifts toward $h_{\target}$ and the set shrinks, as illustrated in \cref{fig:da-pi-cartoon}. Crucially, thresholds $\sqNorm{e_{\stb}}_{L_2(X)} \leq 1$ and $\Omega\leq 1$ are \emph{sufficient}, and not necessary for improvement. In general, DA ``wobbles'' the ellipsoid of \cref{lemma:characterizing-identified-sets}, inflating some directions while deflating others. The intersection $\sH_{\pid}(\vx)\cap\sH_{\tpid}(\vx)$ nonetheless guarantees~no~regression.
% \vspace{-0.5em}
\paragraph{Robustness.}
DA encoding misspecified $\sT$-inv. induces endogenous noise $\mNoise \coloneqq h_{\target}(X) - h_{\target}( \tX )$,
\begin{align*}
    Y = h_{\target}(X) + U + \residual,
    &&
    \iff
    &&
    Y = h_{\target}(\tX) + U + \mNoise + \residual ,
\end{align*}
The post-DA intervals $\sH_{\tpid}(\vx)$  may not account for this additional confounding. Therefore, to recover validity we apply $\eps$-padding $[a, b]\pm \eps\coloneqq [a-\eps, b+\eps]$.\footnote{
For non-interval $\sH_{\tpid}(\vx)$, read $\pm\eps$ as the Minkowski sum $\sH_{\tpid}(\vx)\pm\eps\coloneqq \sH_{\tpid}(\vx)+[-\eps,\eps]$. Also, note that re-calibrating to \emph{any} bigger budget $\tgamma\geq\gamma$ may not suffice for validity since $\E[\eta]\neq 0$ is not represented by~$\sH_{\stX}$.
}
Baseline intersection then absorbs any~slack.

\setcounter{theorem}{3}
\begin{minipage}{\columnwidth}
\begin{restatable}[robust bounds with DA]{subtheorem}{aRobustBoundsWithDA}
\label{thm:general-validity}
\hypertarget{thm-general-validity}{Under} {\asmOneShortAB} and \ref{assumption:partial-r2},
let valid baseline set $\sH_{\pid}$ cover an $(\eps, \sT)$-invariant $h_{\target}$. Then, $\eps$-padded post-DA intervals $\sH_{\tpid}(\vx)\pm \eps$ also remain valid on $p_{X}$-a.e.~$\vx$.
\end{restatable}
\end{minipage}

\begin{remark}[dominant bounds with DA]
\label{remark:robust-bounds-with-da}
Under {\asmOneShortAB} and \ref{assumption:partial-r2}, let $\sH_{\pid}$ cover $(\eps, \sT)$-invariant $h_{\target}$. Then, $\sH_{\pid}(\vx) \cap\pr*{ \sH_{\tpid}(\vx)\pm\eps }$ dominates $\sH_{\pid}(\vx)$ on sharpness and worst-case error for $p_X$-a.e. $\vx$.
\end{remark}

\Cref{thm:general-validity} now enables DA-based sensitivity analysis via the parameter $\eps$; validity is preserved by $\eps$-padding bounds, while baseline intersection maintains sharpness and informativeness. This is depicted in \cref{fig:pi-da-pi-intersection-cartoon} where post-DA bounds may arbitrarily change, yet intersection guarantees~dominance. It is worth contrasting this with \cref{sec:explicit-constraint}; unlike $\sH_{\pid+\inv}(\vx)$, $\sH_{\tpid}(\vx)\pm\eps$ solves the entire $\eps$-grid simultaneously.

%%%%%%%%%%%%%%%%%%%%%%%%%%%%%%%%%%%%%%%%%%
%%%%%%%%%%%%%%%%%%%%%%%%%%%%%%%%%%%%%%%%%%

\subsection{Symmetry constraint with $T$-as-IV }
\label{sec:da-as-iv}
The IV problem in \cref{eq:iv-problem} may be incomplete~\citep{niki-targeted-sequential-iv}, or even misspecified~\citep{leaky-iv,pmlr-v258-penn25a,plausibly-exogenous,misspecified-npiv1,misspecified-npiv2,r2-leaky-iv}. Consequently, $h_{\target}$ is unidentifiable. Nevertheless, even invalid instruments retain value for sharpening~PI:

\begin{assumption}[a leaky IV sensitivity model~\citep{misspecified-npiv1}] 
\label{assumption:iv-model}
For instrument violation budget $\gamma_{\vz}$, let
\begin{equation*}
\sP_{\!\pid + \iv} \coloneqq 
\br[\Big]{\;
q\in \sP_{\!\pid}
\,
:
\,
    \sqNorm{
        \EE{p}{}{
            Y 
            -
            h^q_{\target}(X)
            \given Z
        }
    }_{Z}
    \leq
    \sigma^2
    \gamma_{\vz}
\;} .
\end{equation*}
\end{assumption}

Here, $\gamma_{\vz}$ represents a bounded violation in un-confoundedness and/or exclusion of $Z$, while the baseline constraints $\sP_{\!\pid}$ may allow $\sP_{\!\pid+\iv}$ to be sufficiently well-posed to yield non-vacuous bounds. 

Can \Cref{sec:better-bounds} be pushed further, to impose symmetry explicitly within an IV solver? Apparently so! Post-DA, the random transformation $T$ satisfies the IV properties of un-confoundedness $T\indep U$ and exclusion $Y\indep T\,\vert\, \tX, U$ by construction. But $\sT$-invariance renders it incomplete, since $T\indep X$ gives $\E[\;\! h_{\target}(\tX) \;\!\vert\;\! T\;\!] = \EE{}{}{h_{\target}(X)\given T} = \EE{}{}{h_{\target}(X)} = 0$. Because incomplete instruments still constrain PI,  we plug $T$ in as an IV $\tZ\coloneqq (Z, T)$ to enforce symmetry within off-the-shelf or black-box IV solvers,
$$
\sQ_{\tpid+\tiv}(p_{XYZ}) \coloneqq \sQ_{\pid+\iv}(p_{\stX Y\stZ})
.
$$
We defer readers to~\cref{app:misspecified-da-under-iv-constraints} and~\cref{app:explicit-vs-implicit} for additional analysis under~\cref{assumption:iv-model}, including comparison with~\cref{sec:explicit-constraint}.

\section{Experiments}
\label{sec:experiments}
\begin{figure*}[t]
  \centering
  \begin{minipage}[t]{0.63\textwidth}
    \centering
    \includegraphics[width=1.0\linewidth,trim=0pt 0pt 0pt 0pt,
    clip]{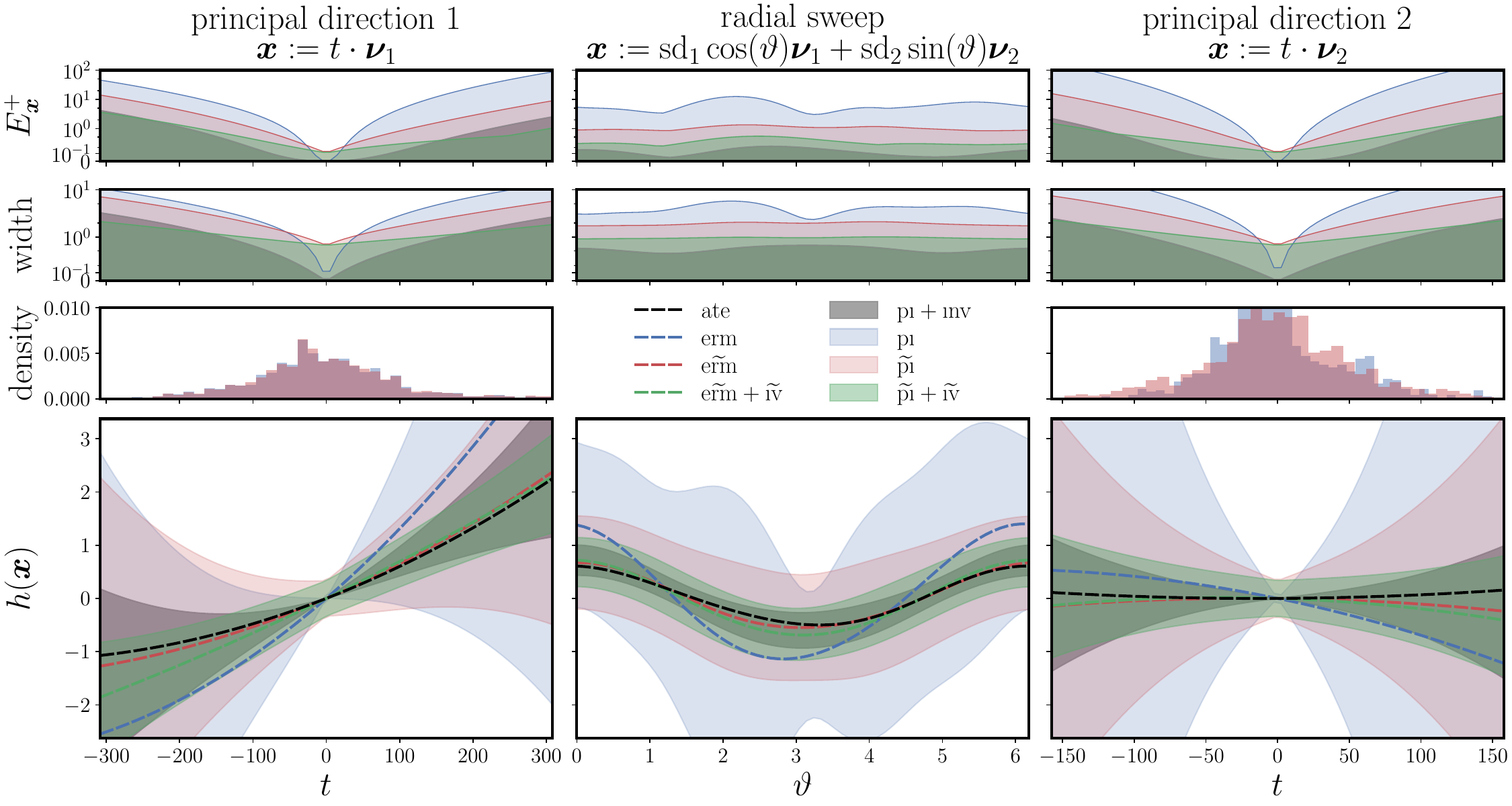}
    \captionof{figure}{
      \textbf{Optical device.}
      Across queries $\vx$ along principal components $\bm\nu_1, \bm\nu_2$ of $\var(X)$ and a radial sweep, symmetries reduce interval width and worst-case error {\figtop} over baseline while maintaining validity {\figbottom}.
    }
    \label{fig:optical_device}
  \end{minipage}
  \hfill%
  \begin{minipage}[t]{0.35\textwidth}
    \centering
    \includegraphics[width=1.0\linewidth,trim=10pt 10pt 10pt 10pt,
    clip]{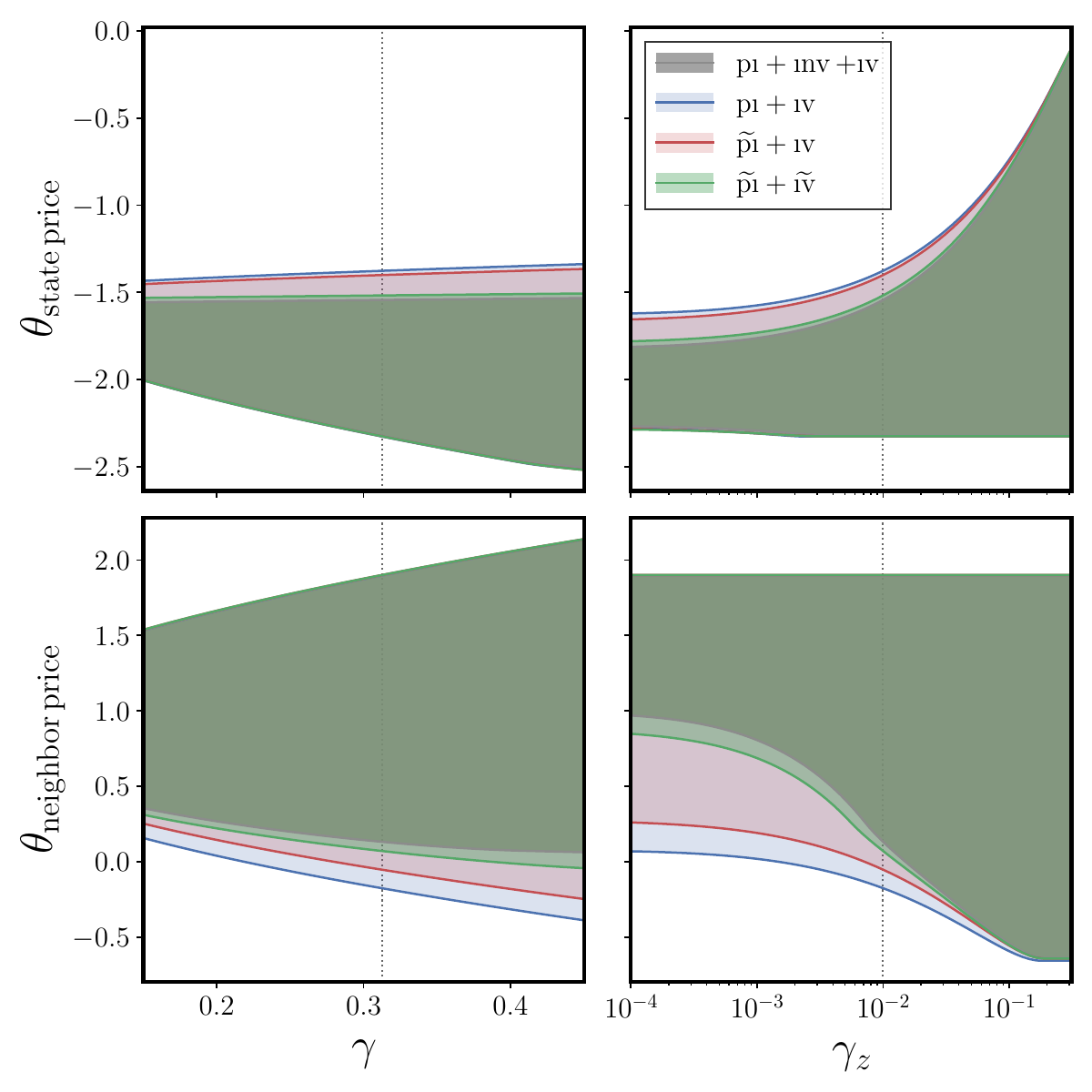}
    \captionof{figure}{
      \textbf{US cigarette demand.}
      State {\figtop} and neighbor {\figbottom} price elasticity across confounding $\gamma$ {\figleft} and IV $\gamma_{z}$ {\figright} budgets.
    }
    \label{fig:cigarettes}
  \end{minipage}
\vspace{-1\baselineskip}
\end{figure*}

We test three claims in finite samples: symmetry-informed PI (\rnum{1}) stays valid, (\rnum{2}) is sharper and more informative than baseline PI, and (\rnum{3}) via DA remains usable where the shape constraint is not. We compare baseline $\color{deep_blue}\pid$ with the shape constraint $\color{deep_grey}\pid+\inv$ (\cref{sec:explicit-constraint}), DA pre-processing $\color{deep_red}\tpid$ (\cref{sec:better-bounds}) and $T$-as-IV $\color{deep_green}\tpid+\tiv$ (\cref{sec:da-as-iv}), each an instance of one common PI program (\cref{table:estimator-grid}).
To isolate symmetry-driven sharpening from mere statistical efficiency, the augmented dataset is the same size as the original. Setup details are in \cref{app:experiment-setup}, and our implementation is publicly available at \url{https://github.com/uzairakbar/symmetry4CausalBounds}.

\textbf{Optical device.}\;
\label{sec:optical-device-dataset}
In \citet{optical-device}'s benchmark, $3\times3$-pixel images $X$ shown on a screen drive voltage readings $Y$ across a photodiode, while a physical latent $U$ drives two LEDs, one lighting the webcam that records $X$, the other the photodiode. Since image orientation should not affect the irradiance, we augment $X$ by flips, $90^{\circ}$ rotations, and sum-zero translations. We obtain $h_{\target}$ by back-door adjustment on $U$, withheld from all methods. \Cref{fig:optical_device} shows that symmetry-informed bounds stay valid while reducing width and worst-case error, sometimes by an order of~magnitude.

\textbf{US cigarette demand.}\;
We give an illustrative example of a textbook symmetry in a canonically causal setting; scaling all prices and income by a common factor (e.g., inflation) should leave demand unchanged. Following~\citet{baltagi-levin}, we take a panel of 48 US states and DC (1970--2019) with $Y$ the log cigarette sales per capita and $X$ the state's own price, cheapest neighbor's price, income and the consumer price index (CPI). We assume the latter two are exogenous, whereas the former two are confounded, with the state's cigarette tax as a (possibly leaky) instrument $Z$.    With $h_{\target}(\cdot)\coloneqq \bm\theta^\top_{\target} \log (\cdot)$, the symmetry implies $\1^\top\bm\theta_{\target}\approx 0$ and DA multiplies $X$ by a random scalar $T$ with $\log T$ Gaussian.
\Cref{fig:cigarettes} shows that imposing the symmetry makes the neighbor-price effect clearly positive unless the IV is heavily leaky, suggestive evidence of cross-border shopping.

\label{sec:do-mnist-dataset}
\begin{wrapfigure}{r}{0.48\textwidth}
  \centering
    \vspace{-\baselineskip}
  \includegraphics[
    width=\linewidth,
    trim=0pt 35pt 15pt 25pt,
    clip
  ]{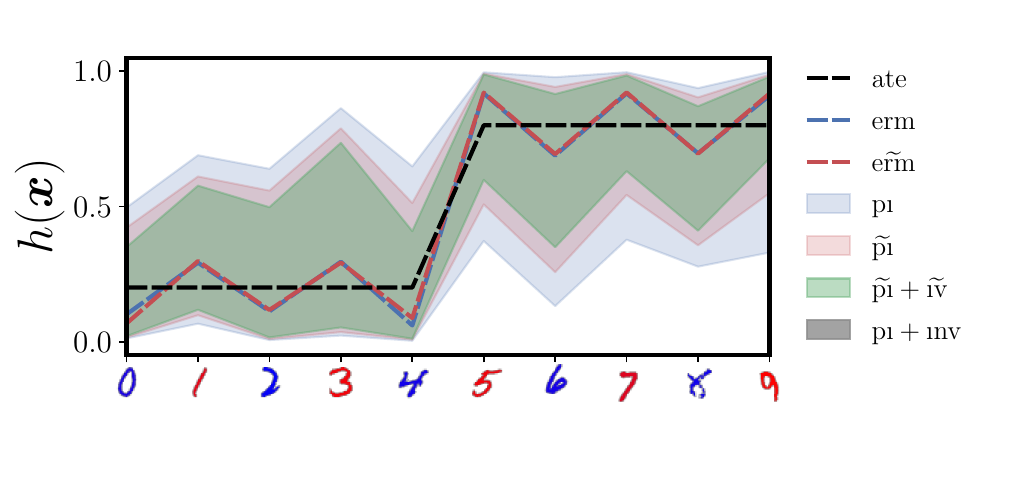}
  \caption{
    \textbf{$\rm do$-MNIST data.}
    Bounds on 10 exemplar images with alternating colors. ERM follows color; symmetries sharpen PI. Invariance shape constraint $\pid+\inv$ is infeasible here~(\cref{app:experiments}).
  }
  \label{fig:colored_mnist}
  \vspace{-\baselineskip}
\end{wrapfigure}
\textbf{$\rm do$-MNIST.}\;
We build a colored-MNIST variant~\citep{arjovsky2020invariant,mnist}, dubbed $\rm do$-MNIST, where a hidden $U$ both tints the digit image $X$ red or blue and corrupts its binary label $Y$, so color spuriously predicts $Y$, while $h_{\target}(\vx)$ depends only on whether the digit is $\geq 5$. We bound $h_{\target}$ using an off-the-shelf CopSens model~\citep{CopSens} with a convolutional network backend and a latent factor $q_{U | X}$. Small translations, rotations and jitter to contrast/saturation/hue defines our DA, none of which alter the digit. In \cref{fig:colored_mnist}, ERM tracks color rather than digit, and symmetries sharpen PI around $h_{\target}$. \Cref{tab:placeholder} aggregates metrics over $2,000$ test queries, showing that DA reduces average width by $19.08\%$ and worst-case error by $34.81\%$, while imposing $T$-as-IV reduces them by $29.05\%$ and $47.10\%$ respectively. The shape constraint $\color{deep_grey}\pid+\inv$, by contrast, is infeasible for any reasonably small $\eps$ as CopSens adapts only the last layer of a frozen, non-invariant feature extractor (\cref{app:experiment-setup}). This is precisely the operational limitation that motivates \cref{sec:implicit-constraint}. Finally, readers may notice that coverage dips by $0.40\%$ and $1.31\%$ for $\color{deep_red}\tpid$ and $\color{deep_green}\tpid+\tiv$ respectively. We blame finite samples and realizability, as with all sins in PI~\citep[\href{https://proceedings.neurips.cc/paper/2020/file/e8b1cbd05f6e6a358a81dee52493dd06-Review.html}{Reviewer 4}]{iv_niki}.

\textbf{Ablations.}\;
Under known $h_{\target}$,~\cref{app:ablations} sweeps each quantity from our theory in turn, aggregating PI metrics over $10$ seeds.
Post-DA sets attain coverage at smaller confounding budgets than baseline, as predicted by \cref{theorem:valid-bounds-with-da}~(\cref{fig:validity}). Width and worst-case error shrink with observable $\Omega$ (\cref{proposition:da-sharper-bounds}, \cref{fig:sharpness-informativeness}). Under deliberately misspecified symmetry, $\eps$-padding restores coverage once $\eps\geq\eps^\star$, and intersecting with the baseline set prevents regression; surprisingly, improvements hold under symmetry violations as large as confounding itself, even with up to $16\times$ over-padding (Thm.~\hyperlink{thm-general-validity}{3}, \cref{remark:robust-bounds-with-da}, \cref{fig:robustness}). Finally, gains persist even without budget re-calibration and across sample sizes $m, n$ (\cref{fig:recalibration,fig:n-efficiency,fig:m-efficiency}).

\begin{table}[t]
\centering
\begin{tabular}{lrrr}
\toprule
 & coverage & width & worst-case error \\
\midrule
$\color{deep_blue}\operatorname{p}\!\textnormal{\i}$ & $\bm{0.993 \pm 0.001}$ & $0.692 \pm 0.001$ & $0.293 \pm 0.001$ \\
$\color{deep_red}\widetilde{\operatorname{p}\!\textnormal{\i}}$ & $\mathit{0.989 \pm 0.001}$ & $\mathit{0.560 \pm 0.001}$ & $\mathit{0.191 \pm 0.001}$ \\
$\color{deep_green}\widetilde{\operatorname{p}\!\textnormal{\i}}+\widetilde{\textnormal{\i}\!\operatorname{v}}$ & $0.980 \pm 0.004$ & $\bm{0.491 \pm 0.019}$ & $\bm{0.155 \pm 0.010}$ \\
\bottomrule
\end{tabular}
\caption{Metrics aggregate over $2,000$ $\rm do$-MNIST test queries $\pm$ one standard error across~10~seeds.}
\label{tab:placeholder}
\vspace{-1\baselineskip}
\end{table}

\section{Discussion}
\label{sec:limitations}
\textbf{Limitations.}
Our proposed framework provides a practical hierarchy of symmetry constraints for PI: should an invariance shape constraint pose implementation challenges, practitioners can resort to the DA operationalization. If they suspect misspecification by some $\eps$, they need only to $\eps$-pad bounds. Which leads us to the main limitation of our work: while the soft intervention framing in~\cref{lemma:da-soft-intervention} makes composability of DA with a broad class of disparate PI solvers plausible under well-specified symmetries, it is yet unclear whether this should also be the case under $\eps$-misspecification. Having addressed this question under two canonical, arguably the most popular PI classes (\cref{assumption:partial-r2} bounded confounding and \cref{assumption:iv-model} leaky IVs), further investigation of additional models is left as future work.

\textbf{Broader applicability.}
In principle, symmetry as a shape constraint is compatible with \emph{all} PI models. Furthermore, since DA dilutes confounding, it also extends to other bounded-confounding models, e.g.\ MSM~\citep{MSM}, Wasserstein~\citep{implicit-g-constrained-pid}, total-variation~\citep{hallucinatedDRO,dro-pid} and $f$-divergence~\citep{f-sensitivity} balls. All of these, like \cref{assumption:partial-r2}, admit conditioning on pre-treatment $Z$~\citep{making-sensemakr}, omitted here only for clarity of exposition. Under well-specified symmetry, DA leaves $h_\target$ unchanged, so other shape constraints remain valid. And as DA preserves the causal graph up to relabeling, \citet{balke1997bounds} style graphical constraints carry over too; with response-function invariance $H_{\target}(\vx) =^d H_{\target}(\vtau\vx)$, this covers IVs with latent, non-additive $U$~\citep{iv_niki}, response-function models~\citep{causal-programming}, proxy variables~\citep{proxy-pi,prox-causal-learning,pmlr-v258-sun25d} and others~\citep{nonparametric-graphical}.

\textbf{Related work.}
\citet{cocycle} use symmetry to improve estimation of an identifiable target, and concurrent work uses symmetry as a point-identification assumption~\citep{geometricCausalModels,symmetryInCausalInference}. Counterfactual DA likewise imposes causal priors via augmentation for point-estimation~\citep{ilse2020selecting,counterfactualResidualDA}, or uses causal structure to inform DA design~\citep{causalResidualBootstrap}. All of these are target-identifiable settings; we consider the weaker setting where point-identification itself fails. Closer to our work is~\citet{causal-data-augmentation}, where additive Gaussian DA under a stylized, linear-Gaussian DGP improves non-identifiable point-estimation.~\Cref{proposition:da-causal-estimation} shows that this is generally fragile, which motivates the PI~lens~taken~here.

\textbf{Future work.}
We aim to generalize our change-of-measure framework to 
(\rnum{1}) general causal queries on $p_{Y | \intervention{\vx}}$~\citep{niki-targeted-sequential-iv}, 
(\rnum{2}) equi-variance in the response function $H_\target(\vtau\vx) =^d \beta(\vx, \vtau) H_\target(\vx)$ under cocycle $\beta$~\citep{cocycle}, and
(\rnum{3}) latent $U$ under a $\wp$-Wasserstein notion of equi-variance misspecification $\sup_{\vx, \vtau} W_\wp(p_{H_\target(\vtau\vx)}, \beta(\vx, \vtau)_\# p_{H_\target(\vx)})\leq\eps$~\citep{group-theory-da-jmlr}. Also, it may be interesting to demonstrate application across causal settings like genomics~\citep{geometricCausalModels} and medicine~\citep{medA,medB,med1,cheema2026necklearn}.

\textbf{Conclusion.}
Researchers are advised to incorporate established domain symmetries as a standard component of their partial-identification toolkit. Integration is effortless with DA pre-processing.

% \subsection*{Reproducibility statement}
% Implementation details, supplementary experiments, and complete proofs are provided in the appendices.

\newpage
{\small
\bibliography{bibliography}
}

%%%%%%%%%%%%%%%%%%%%%%%%%%%%%%%%%%%%%%%%%%%%%%%%%%%%%%%%%%%%

\newpage
\appendix

\section*{Appendix---Symmetry-Informed Causal Partial Identification}
{
    \startcontents[appendices]
    \printcontents[appendices]{l}{1}{\section*{Contents}\setcounter{tocdepth}{3}}
}

\newpage
\section{List of Symbols}
\label{app:notation}
\vspace{-1\baselineskip}
\begin{longtable}[c]{@{}p{0.25\textwidth}p{0.75\textwidth}@{}}
\label{tab:notation} \\
\toprule
\textbf{Symbol} & \textbf{Description} \\
\midrule
\endfirsthead
\multicolumn{2}{c}{\tablename\ \thetable\ -- \textit{Continued from previous page}} \\
\toprule
\textbf{Symbol} & \textbf{Description} \\
\midrule
\endhead
\midrule
\multicolumn{2}{r}{\textit{Continued on next page}} \\
\endfoot
\bottomrule
\endlastfoot

\multicolumn{2}{@{}l}{\textbf{Variables and Spaces}} \\[0.5em]
$X, \sX, \vx$ & Treatment, its support, and a realization \\[0.4em]
$Y, \sY, y$ & Outcome, its support, and a realization \\[0.4em]
$Z, \sZ, \vz$ & Pre-treatment co-variates / instruments (IVs), their support, and realization \\[0.4em]
$U, \residual$ & Unobserved confounder, exogenous noise (mean-zero, $\residual\indep(U,X,Z)$) \\[0.4em]
$\tX, \tZ$ & Post-DA treatment $\tX\coloneqq TX$ and instrument $\tZ\coloneqq (Z, T)$  \\[1em]

\multicolumn{2}{@{}l}{\textbf{Hypotheses and Estimators}} \\[0.5em]
$h, h_{\target}, \sH$ & Candidate hypothesis, causal target, hypothesis class \\[0.4em]
$\sH_X, \sH_{\stX}, k$ & $X$-, $\tX$-centered $\sH$ with covariance forms $\langle\cdot,\cdot\rangle_X, \langle\cdot,\cdot\rangle_{\stX}$ and $k\coloneqq{\rm dim}(\sH_X)$ \\[0.4em]
$h_{\erm}, h_{\term}, b, \tb$ & ERM estimators $h_\bullet\coloneqq \argmin_{h\in\sH}R_\bullet(h)$, their respective biases $h_{\bullet}-h_{\target}$ \\[0.4em]
$h_{\iv}, h^{p}_{\target}$ & IV estimator, and back-door adjustment over joint $p$ \\[0.4em]
$g_{\vx}, \tg_{\vx}$ & Riesz representers for point-evaluation $h\mapsto h(\vx)$ in $\sH_X$ and $\sH_{\stX}$ \\[0.4em]
$\sH_{\inv}, \sI, \sPi$ & $\sPi$ orthogonally projects onto ambient invariant space $\sI$, and $\sH_{\inv}\coloneqq \sH\cap\sI$ \\[1em]

\multicolumn{2}{@{}l}{\textbf{Symmetries, Interventions and Change of Measure}} \\[0.5em]
$T, \sT, \vtau, \alpha$ & Random transformation, support, realization, and defining map $\vtau\vx\!=\!\alpha(\!\!\;\vx,\!\!\;\!\vtau\!\!\;)$ \\[0.4em]
$\intervention{\vx}, \intervention{X}{\stX}$ & Hard intervention $\operatorname{do}(X=\vx)$, and DA-simulated soft intervention \\[0.4em]
$\eps, \mNoise$ & Invariance residual $\mNoise\coloneqq h_{\target}(X)-h_{\target}(\tX)$ uniformly bounded $\abs{\eta}\leq\eps$ a.s. \\[0.4em]
$\shift, \varShift, \meanShiftAdj\meanShift$ & Shift operator on $\sH_{\stX}$, splits into shape/variance $+$ rank-1 level/mean shifts \\[0.4em]
$\sB, D_{\sB}$ & MMD $D_{\sB}(p_X, p_{\stX})$ over a unit $\tX$-ball in $\sH_{\stX}$, with $D^2_{\sB} = \tr{\delta^\ast \delta}$~(\cref{lemma:shift-decomposition}) \\[0.4em]
$\sigma^2, \tsigma^2, \rho$ & Baseline / post-DA MMSE, and the noise-inflation factor $\rho\coloneqq\tsigma^2/\sigma^2 \geq 1$ \\[0.4em]
$\Omega$ & $\Omega \coloneqq \rho \cdot \left.\pr*{\tr{\varShift} + D^2_{\sB}} \right/ k$; post-DA bounds sharpen if $\Omega \leq 1$~(\cref{proposition:da-sharper-bounds}) \\[1em]

\multicolumn{2}{@{}l}{\textbf{Constraints and Partial Identification}} \\[0.5em]
$p, q$ & True observational joint, and a candidate joint distribution in $\sQ_{\pid}$ \\[0.4em]
$\IPM_{\sF}, \sF$ & IPM and its witness-function class defining $\sQ_{\pid}$ \\[0.4em]
$\gamma_{\vx}, \gamma_y, \gamma_{\vz}$ & Assumed budgets for confounding in $X, Y$, and leaky IV $Z$ violation \\[0.4em]
$\gamma, \gamma_{\bullet}^{\target}, \tilde{\gamma}_{\bullet}$ & Effective confounding budget $\gamma\coloneqq\frac{\gamma_{\vx}\gamma_y}{1-\gamma_{\vx}}$, true and re-calibrated budgets \\[0.4em]
$\sP_{\!\pid}, \sP_{\!\pid+\inv}, \sP_{\!\pid+\iv}$ & PI constraint sets: baseline, augmented with invariance / IV constraints \\[0.4em]
$\sQ_{\pid}, \sQ_{\tpid}, \sQ_{\tpid+\tiv}$ & Candidate joints: baseline, post-DA, and DA-as-IV \\[0.4em]
$\sH_{\pid}, \sH_{\tpid}, \func{\sH_{\pid}}{\vx}, \func{\sH_{\tpid}}{\vx}$ & Baseline and post-DA identified sets, their respective intervals at query $\vx$ \\[1em]

\multicolumn{2}{@{}l}{\textbf{Risks and Evaluation Metrics}} \\[0.5em]
$\Risk{\erm}{}{h}, \Risk{\term}{}{h}$ & Observational risks, pre- and post-DA \\[0.4em]
$R_{\vx}(y), \Risk{}{}{h}$ & Causal / interventional risks, point-wise and averaged \\[0.4em]
$E_{\vx}(y), E(h)$ & Causal excess risk/ estimation error $E(h)\!=\!\Risk{}{}{h}-\!\;\!\Risk{}{}{h_{\target}\!\;\!}$, pointwise~via~$R_{\vx}$ \\[0.4em]
$E^{-}(\sH), \overline{E}(\sH), E^{+}(\sH)$ & Approximation (min), averaged, and worst-case (max) errors $E$ over set $\sH$ \\[0.4em]
$E^-_{\vx}(\sY), \overline{E}_{\vx}(\sY), E^+_{\vx}(\sY)$ & Point-wise versions of the above metrics via $E_{\vx}$ over interval $\sY$ \\[0.4em]
$E_{\inv}(h)$ & Invariance error $\E_p\sqPr[\big]{ \abs{ h(X) - h(\tX) }^2 }$ \\[0.4em]
\end{longtable}

\newpage
\section{Misspecified DA under IV Constraints}
\label{app:misspecified-da-under-iv-constraints}

For completeness, we explore the limits of the $\eps$-padding approach from \cref{thm:general-validity} under the IV constraint in \cref{assumption:iv-model}. Notice that $\sT$-invariance violation manifests as violated IV exclusion / un-confoundedness. This violation may be accommodated by either DA reducing confounding via the DPI $\sigma^2\leq \tsigma^2$, or a sufficient slack in the baseline IV budget $\gamma_{\vz}$.

\begin{minipage}{\columnwidth}
\begin{restatable}[robust bounds with DA]{subtheorem}{bRobustBoundsWithDA}
\label{thm:general-validity-b}
Let a valid model following \cref{assumption:iv-model} yield intervals~$\sH_{\pid+\iv}(\vx)$ under the setting of \cref{thm:general-validity}. Then, $\sH_{\tpid+\tiv}(\vx)\pm\eps$ is also valid for $\tZ\in\{Z, (Z, T)\}$ on $p_{X}$-a.e. $\vx$ if
\begin{align*}
    \eps \leq \tsigma \sqrt{\gamma_{\vz}} - \sigma \sqrt{\gamma_{\vz}^{\target}}
    ,
    &&
    \text{
        where $\gamma_{\vz}^{\target}\leq \gamma_{\vz}$ is the true IV budget.
    }
\end{align*}
\end{restatable}
\end{minipage}

The above inequality allows for some tolerance to $\eps$-misspecified symmetry under DA so long as $\gamma_{\vz}$ is positive. In practice this is often true even for perfect instruments to make the optimization problem more tractable. Importantly, practitioners can use the following observable diagnostic for \cref{thm:general-validity-b}:
\begin{align*}
    \eps
    \leq
    (\tsigma - \sigma)
    \sqrt{
        \gamma_{\vz}
    }
    .
\end{align*}
Alternatively, if the sensitivity parameter $\gamma_{\vz}$ is exposed, one can simply re-calibrate it to $\tilde{\gamma}_{\vz}$ satisfying
\begin{align*}
    \tgamma_{\vz}(\eps)
    =
    \pr*{
        \eps / \tsigma + \sqrt{\gamma_{\vz} / \rho}
    }^2
    .
\end{align*}
Meaning the new budget $\tilde{\gamma}_{\vz}$ should accommodate the IV leak as well as the misspecified symmetry. If one plausibly assumes symmetry violation is ``orthogonal'' to confounding~(\cref{remark:conservative-iv}), this tightens~to
\begin{align*}
    \tgamma_{\vz}(\eps)
    =
    \eps^2 / \tsigma^2
    +
        \gamma_{\vz} / \rho
    .
\end{align*}

\newpage

\section{Shape Constraint vs.\ Change of Measure}
\label{app:explicit-vs-implicit}

In this section we aim to compare the two operationalizations of symmetry informed PI; via an invariance shape constraint in~\cref{sec:explicit-constraint}, and via change-of-measure through DA pre-processing in~\cref{sec:implicit-constraint}. Throughout, we assume no observable co-variates $Z = \varnothing$ for simplicity, and a mean-centered $Y$ w.l.o.g. with the mean absorbed into the intercept of $h_{\target}$. Define the ambient $\sT$-inv. subspace $\sI$ as
\begin{align*}
    \sI\coloneqq\br[\Big]{
    \,
    h
    \in
    L^0_2(X)
    \cap
    L^0_2(\stX)
    \;
    \colon
    \;
    h(X)=h(\stX)\ \text{a.s.}
    \,
    }
    ,
    \qquad
    \sH_{\inv}
    \coloneqq
    \sH\cap\sI
    .
\end{align*}
Notice that, by construction, all forms $\langle\cdot, \cdot\rangle_X$, $\langle\cdot, \cdot\rangle_{\stX}$ $\langle\cdot, \cdot\rangle_{L_2(X)}$ and $\langle\cdot, \cdot\rangle_{L_2(\stX)}$ coincide on $\sI$, $\sH_{\inv}$.
Next, define the operator $\sE_{A | B}: L^0_2(A) \rightarrow L^0_2(B)$ as conditional expectation on centered subspaces
\begin{align*}
    (\sE_{A | B}\, h)(\cdot) = \EE{}{}{h(A) \given B = \cdot} .
\end{align*}
Well defined because $\EE{}{}{\EE{}{}{h(A) \given B}} = \EE{}{}{h(A)}=0$ by the tower property, and $\sE_{A | B} h = 0$ iff $\EE{}{}{h(A) \given B}=0$ a.s..
Of particular interest are $\sE_{X | \stX}$ and $\sE_{\stX | T}$, both of which we discuss below.

Now, under a $\sT$-inv. $h_{\target}$, $\eps = 0$ and the invariance shape constraint in~\cref{sec:explicit-constraint} yields $\sH_{\pid+\inv} = \sH_{\pid}\cap\sH_{\inv}$ (see~\cref{corollary:sharp-inv-constraints}). By contrast, DA pre-processing yields $\sH_{\tpid}$ and/or $\sH_{\tpid+\tiv}$. To make their comparison tractable in observables, we make two standard assumptions. First, we take $\sT$ to be a compact group:

\vspace{1pt}
\begin{assumption}[compact group action DA]
\label{assumption:compact-group}
Let $\sT$ be a compact group with action $\alpha : \sX\times\sT\rightarrow \sX$. Furthermore, $\sT$ is endowed with a (unique) normalized Haar measure $p_T$, i.e. $p_T$ is uniform over $\sT$.
\end{assumption}

\Cref{assumption:compact-group} encompasses several standard DA transformations, including some used in our experiments such as rotations, translations, etc. Under~\cref{assumption:compact-group},  $p_{\smash{\tX}}$ itself becomes $\sT$-invariant. Importantly, $\sE_{X | \stX}$ is now the Reynolds operator; an $L_2(X)$-orthogonal projector onto~$\sI$~\citep{reynolds}. Under~{\asmOneShortAB}, the post-DA bias $\tb$ is then a projection of $b$ onto $\sI$~(\cref{lemma:haar-group-da}). This now allows us to compare identification metrics entirely in observables, e.g. in~\cref{proposition:da-causal-estimation},
$$
\sqNorm{e_{\tb}}_{L_2(X)}
=
\frac{\sqNorm{\stb}_{L_2(X)}\vphantom{\tb}}{\sqNorm{\stb}_{\tX}\vphantom{\tb}}
=
\frac{\sqNorm{\stb}_{\tX}\vphantom{\tb}}{\sqNorm{\stb}_{\tX}\vphantom{\tb}}
=
1 \vphantom{\tb}
.
$$
Second, since $T$-as-IV imposes~\cref{eq:iv-problem} as a constraint, it filters functions with $\mathbb{E}[ h(\tX)|T]$~a.s.~constant. This holds for all $h\in\sH_{\inv}$, since $\mathbb{E}[h(\tX) | T] = \mathbb{E}[h(X) | T] = \mathbb{E}[h(X)]=0$ a.s. following $T\indep X$. Hence $\sH_{\inv}\subseteq \Null{ \sE_{\smash{\tX} | T} }$. We say that DA is \emph{relevant}, or \emph{complete}, when this holds with equality:

\vspace{1pt}
\begin{assumption}[relevant / complete DA]
\label{assumption:da-relevance-completeness}
    Let $T$ be \emph{relevant / complete}:\hfill $\Null{ \sE_{\smash{\tX} | T} \given_{\sH_{\stX}} } = \sH_{\inv}$.\hfill\vphantom{.}
\end{assumption}

\Cref{assumption:da-relevance-completeness} is akin to the standard injective $\sE_{X | Z}$ assumption for relevant/complete IVs $Z$~\citep{niki-targeted-sequential-iv,pmlr-v258-sun25d}. Similarly,~\cref{assumption:da-relevance-completeness} assumes $\sE_{\stX | T}$ injective on $\tX$-orthogonal complement of $\sH_{\inv}$ within $\sH_{\stX}$. Meaning, the only directions whose data-averaged response is zero across augmentations are the invariant~ones.

The inv. shape constraint and change-of-measure routes are now comparable under \cref{assumption:compact-group,assumption:da-relevance-completeness}:

\begin{restatable}[shape constraint vs. change-of-measure]{proposition}{invConstraintVsDA}
\label{proposition:inv-constraint-vs-da}
Under {\asmOneShortAB}, and~\cref{assumption:partial-r2,assumption:iv-model,assumption:compact-group,assumption:da-relevance-completeness} with $\sT$-invariant $h_{\target}$ and $Z = \varnothing$. 
Set $\eps = \gamma_{\vz}=0$, and suppose the baseline PI set is valid $h_{\target}\in\sH_{\pid}$.
Then, $h_{\term} = \mathcal{\Pi}h_{\erm} \in \sH_{\inv}$, where $\mathcal{\Pi}$ is the $L_2(X)$-orthogonal projection onto $\sI$~(cf.~\cref{app:proofs}).~Further,
\begin{align*}
    \sH_{\tpid} &= \br*{
        h\in
        \sH_{\stX}
        \;\;\!
        \colon
        \sqNorm{\vphantom{\big|}
            h - h_{\term}
        }_{\stX}
        \leq
        \tsigma^2\gamma
    }
    ,
    \\
    \sH_{\tpid+\tiv} &= \br*{
        h\in
        \sH_{\inv}
        \colon
        \sqNorm{\vphantom{\big|}
            h - h_{\term}
        }_{\stX}
        \leq
        \tsigma^2\gamma
    }
    ,
    \\
    \sH_{\pid+\inv} &= \br*{
        h\in
        \sH_{\inv}
        \colon
        \sqNorm{\vphantom{\big|}
            h - h_{\term}
        }_{X}
        \leq
        \sigma^2(1 + \gamma - \rho)
    }
    .
\end{align*}
All identified sets are balls centered about $h_{\term}$. Both $\sH_{\tpid+\tiv}$ and $\sH_{\pid+\inv}$ live on the slice $\sH_{\inv}$, the latter with a smaller radius $\sigma \sqrt{ 1 + \gamma - \rho } \leq \tsigma \sqrt{\gamma}$ with equality iff $\rho = 1$. It trivially follows that,
\begin{align*}
    h_{\target}
    \in
    \sH_{\pid+\inv}\subseteq\sH_{\tpid+\tiv}\subseteq\sH_{\tpid}
    ,
\end{align*}
and the same dominance hierarchy holds for sharpness and worst-case errors over $\sH_\bullet$~and~$\sH_\bullet(\vx)$.
\end{restatable}

Under~{\asmOneC}, closed form volume, width and avg. errors for the balls in~\cref{proposition:inv-constraint-vs-da} can now be derived in observables similar to~\cref{proposition:da-sharper-bounds} and~\cref{theorem:informative-bounds-with-da}. We leave it as an exercise for interested readers.

\newpage
\section{Implementation Details}
\label{app:implementation-details}
\paragraph{Identification program.}

\setcounter{table}{2}
\begin{table}[t]
  \centering
  \footnotesize
  \setlength{\tabcolsep}{3.5pt}
  \caption{Every object in the paper as one setting of the same program~\eqref{eq:P2};
  labels are those used in the legends of \cref{fig:optical_device,fig:cigarettes,fig:colored_mnist,fig:linear_simulation,fig:validity,fig:sharpness-informativeness,fig:robustness,fig:recalibration,fig:n-efficiency,fig:m-efficiency}. The label may drop $\dots+\iv$ whenever $Z = \varnothing$ in the respective experiment. $\gamma_{\min}$ represents the minimum feasible budget, whereas $\tgamma_\bullet$ represents a post-DA, re-calibrated budget. And $\color{deep_blue}\erm+\iv$ is the anchor regression~\citep{anchor}.
}
  \label{table:estimator-grid}
\begin{tabular}{@{}cc|lccc|lcccl@{}} \toprule 
    & & & \multicolumn{3}{c|}{Sensitivity params.} & & \multicolumn{3}{c}{Sensitivity params.} & \\ 
    \cline{4-6} \cline{8-10} 
    \rule{0pt}{1.1em}Dataset & Instruments & Point estimate & $\gamma$ & $\gamma_{\vz}$ & $\eps$ & PI interval & $\gamma$ & $\gamma_{\vz}$ & $\eps$ & Post-proc. \\
\midrule
    $\sD$             & $Z$      & ${\color{deep_grey}\iv}$
                      & $\infty$ & $0$ & $\infty$
                      & ${\color{deep_grey}\pid+\inv+\iv}$
                      & $\gamma$ & $\gamma_{\vz}$ & $\eps$ & \\

    $\sD$             & $Z$      & ${\color{deep_blue}\erm+\iv}$
                      & $\gamma_{\min}$ & $0$ & $\infty$
                      & ${\color{deep_blue}\pid+\iv}$
                      & $\gamma$ & $\gamma_{\vz}$ & $\infty$ & \\

    $\color{deep_red}\widetilde{\sD}$ & $Z$
                      & ${\color{deep_red}\term+\iv}$
                      & $\gamma_{\min}$ & $0$ & $\infty$
                      & ${\color{deep_red}\tpid+\iv}$
                      & $\tgamma$ & $\stgamma_{\vz}$ & $\infty$ & $\pm\eps$ \\

    $\color{deep_red}\widetilde{\sD}$
                      & ${\color{deep_green}\tZ} \coloneqq (Z, {\color{deep_green}T})$
                      & ${\color{deep_green}\term+\tiv}$
                      & $\gamma_{\min}$ & $0$ & $\infty$
                      & ${\color{deep_green}\tpid+\tiv}$
                      & $\tgamma$ & $\stgamma_{\vz}$ & $\infty$ & $\pm\eps$ \\
    \bottomrule
  \end{tabular}
\end{table}

Write the finite-$n$ sample version of the abstract PI program from~\cref{sec:partial-identification-and-sensitivity-analysis} as:
\begin{equation}\tag{P1}\label{eq:P1}
\begin{aligned}
\min_{q} / \max_{q} \ & h^q_\star(\vx) && \text{// interval lower / upper bounds} \\
\text{s.t.} \quad & D_{\sF}(q_{XYZ},\ \widehat{p}_{XYZ}) \le \gamma_n, && \text{// observational consistency} \\
& q \in \mathcal{P}_{\pid}, \quad h^q_\star \in \sH, \quad \widehat{E}_{\inv}(h^q_\star) \le \eps^2. && \text{// structural constraints}
\end{aligned}
\end{equation}
Here, $\widehat{p}_{XYZ}$ is the empirical distribution of $n$ observations. The IPM $D_{\sF}$ matches $q_{XYZ}$ to $\widehat{p}_{XYZ}$ within tolerance $\gamma_n\geq 0$ to account for sampling error. Further, $\widehat{E}_{\inv}$ is a Monte Carlo approx.~of~$E_{\inv}$:
\begin{equation*}
\widehat{E}_{\inv}(h) := \frac{1}{m n} \sum_{i = 1}^n \sum_{j = 1}^m \big\lvert h(\vx _i) - h(\vtau_{i,j}\, \vx _i) \big\rvert^2 ,
\qquad
(\vx _i, \cdot)\in\sD,
\quad
\vtau_{i, j}\sim p _T .
\end{equation*}
In practice, the partial-$r^2$ model~(\cref{assumption:partial-r2}) directly imposes~\cref{lemma:characterizing-identified-sets} via $\widehat{R}_{\erm}(h) - \widehat{R}_{\erm}( \widehat{h}_{\erm} ) \le \widehat{\sigma}^2 \gamma$,
\begin{equation*}
\widehat{R}_{\erm}(h) - \widehat{R}_{\erm}( \widehat{h}_{\erm} )
=
\underbrace{
    \frac{1}{n}\sum_{i=1}^{n}
    \bigl(
    h(\vx_i)
    -
    \widehat{h}_{\erm}(\vx_i)
    \bigr)^2
}_{\leq \widehat{\sigma}^{2}\gamma\,\text{ by \cref{lemma:characterizing-identified-sets}}}
,
\qquad
\widehat{\sigma}^{2}
:=
\frac{1}{n - k}\sum_{i=1}^{n}
\bigl(y_i-\widehat{h}_{\erm}(\vx_i)\bigr)^2
.
\end{equation*}
For leaky-IV, we
implement a generalized method of moments (GMM) relaxation of the conditional moment in~\cref{assumption:iv-model}~\citep{causal-data-augmentation,niki-targeted-sequential-iv}. For IV test functions
$\phi_{\vz}:\mathcal{Z}\to\mathbb{R}^{d_{\phi_{\vz}}}$,~define
\begin{equation*}
\widehat{R}_{\iv}(h)
:=
\left\|
\frac{1}{n}\sum_{i=1}^{n}
\bigl(y_i-h(\vx_i)\bigr)
\cdot \phi_{\vz}(\vz_i)
\right\|_{\bm W}^{2},
\qquad
\lVert\boldsymbol{v}\rVert_{\bm W}^{2}
:=
\boldsymbol{v}^{\top}{\bm W}\boldsymbol{v}
.
\end{equation*}
We use $\phi_{\vz}(\vz)=\vz$ and
$\bm W=\widehat{\mSigma}_{\vz}^{\dagger}$~\citep{deep-gmm}, where $\dagger$ denotes the Moore--Penrose pseudoinverse. For~point-estimation $\widehat{h}_{\iv} \coloneqq \argmin\limits_{h\in\sH} \widehat{R}_{\iv}(h)$, whereas for PI, impose~$\widehat{R}_{\iv}(h) \le \widehat{\sigma}^2 \gamma_{\vz}$.

We instantiate the hypothesis class by a parameterized family $\sH:=\{\Phi \circ \bm\theta^\top \phi_{\vx} :\vtheta\in\mathcal{\Theta}\}$ with some basis features $\phi_{\vx}:\sX\to\R^k$, e.g. polynomial features for optical device, log transform in cigarette demand, and penultimate layer of a convolutional neural network in $\rm do$-MNIST~\citep{iv_niki,causal-programming,CopSens}. Program \eqref{eq:P1} under~\cref{assumption:partial-r2,assumption:iv-model} now reduces~to:
\begin{equation}\tag{P2}\label{eq:P2}
    \begin{aligned}
        \min_{\vtheta \in \mathcal{\Theta}} / \max_{\vtheta \in \mathcal{\Theta}} \ & h_{\vtheta}(\vx) \\
        \text{s.t.}
        \quad & 
        \widehat{R}_{\erm}(h_{\vtheta}) - \widehat{R}_{\erm}(\widehat{h}_{\erm}) \le \widehat{\sigma}^2 \widehat{\gamma}, && \text{// bounded confounding (\cref{lemma:characterizing-identified-sets})} \\
        & 
        \widehat{R}_{\iv}(h_{\vtheta}) \le \widehat{\sigma}^2 \widehat{\gamma}_{\vz}, && \text{// leaky IV constraints (\cref{assumption:iv-model})} \\
        & \widehat{E}_{\inv}(h_{\vtheta}) \le \eps^2 ,\qquad && \text{// invariance  constraint (\cref{sec:explicit-constraint})}
        \\
        & 
        \pr*{ 
            \frac{1}{n} \displaystyle\sum_i h_{\vtheta}(\vx_i) - y_i 
        }^{\!\! 2 } \leq \widehat{\sigma}^2 \widehat{\gamma}_{\mu} , && \text{// center / mean-match}
    \end{aligned}
\end{equation}
where we relax population budgets $\gamma_\bullet$ into the finite-sample tolerances $\widehat{\gamma}_\bullet$ elaborated below. 
Finally,~\cref{table:estimator-grid} lists how programs~\eqref{eq:P1} or~\eqref{eq:P2} are invoked under different approaches of symmetry-informed PI.

\paragraph{Finite-sample tolerances $\widehat{\gamma}_\bullet$.}
Evaluated at $h_{\target}$, every empirical constraint in~\eqref{eq:P2} has the form $\|\widehat{\bm{v}}_n\|^2 \le \sigma^2\gamma_\bullet$,
where $\widehat{\bm{v}}_n\in\mathbb R^d$ is a sample statistic (table below). The population counterpart $\bm v \coloneqq \operatorname*{plim}_{n\to\infty} \widehat{\bm{v}}_n$ satisfies $\|\bm v\|^2 \le \sigma^2\gamma_\bullet$. At finite $n$, $\widehat{\bm{v}}_n$ fluctuates about $\bm v$ at
rate $n^{-1/2}$, so $h_\star$ may violate the empirical constraint. We therefore relax each budget $\gamma_\bullet$ just enough that $h_\star$ satisfies the corresponding empirical
constraint with asymptotic probability $\geq 1-\alpha/3$.

\begin{assumption}[asymptotic normality]
    \label{assumption:asymptotic-normal}
    At $h = h_{\target}$, let $\sqrt{n} ( \widehat{\bm{v}}_n  - \bm{v} ) \rightsquigarrow \sN(0, s_0^2 \mI_d)$ and $0 < s^2_0 \leq s^2$.
\end{assumption}

\begin{restatable}[finite-sample budgets]{proposition}{proptolerances}\label{prop:tolerances} Under~\cref{assumption:asymptotic-normal} and the population budget $\|\bm v\|^2 \le \sigma^2\gamma_\bullet$, set
\[
\widehat{\gamma}_\bullet := \frac{s^2}{\sigma^2}\,\Gamma_n\Bigl(d;\, \frac{\sigma^2}{s^2}\gamma_\bullet\Bigr),
\qquad
\Gamma_n(d;\lambda) := \frac{1}{n}\, F^{-1}_{\chi^2_d(n\lambda)}\Bigl(1-\frac{\alpha}{3}\Bigr),
\]
where $\chi^2_d(\kappa)$ is the noncentral $\chi^2$ distribution with $d$ degrees of freedom and noncentrality value $\kappa$. Then $\liminf_{n\to\infty}\Pr\bigl(\|\widehat{\bm{v}}_n\|^2 \le \sigma^2\widehat{\gamma}_\bullet\bigr) \ge 1-\alpha/3$, such that $0 < \widehat{\gamma}_\bullet - \gamma_\bullet = O\bigl(d/n + \sqrt{\gamma_\bullet/n}\bigr)$.
\end{restatable}

Now, simply plug the respective $(s, d, \gamma_\bullet)$ values to get finite-sample budgets for each~\eqref{eq:P2} constraint:
\begin{center}
\small
\begin{tabular*}{\textwidth}{@{\extracolsep{\fill}}lclccl@{}}
\toprule
constraint & $\widehat{\bm v}_n$ & $d$ & $s_0^2$ & $s^2/\sigma^2$ & $\gamma_\bullet$ \\
\midrule
ERM \citep{seber2003linear}
& $\widehat{\mSigma}^{1/2}_{\phi_{\vx}}\big(\widehat{\vtheta}_{\erm} - \vtheta_{\target}\big)$
& $k$
& $\sigma^2$
& $1 + \gamma$
& $\gamma$ \\

IV \citep{anderson1949estimation}
& $\widehat{\bm{W}}^{1/2} \frac{1}{n} \sum_i \pr*{y_i - h_{\target}(\vx_i)} \vz_i$
& $d_{\vz}$
& $\var(U + \xi)$
& $1 + \gamma$
& $\gamma_{\vz}$ \\

mean-match
& $\frac{1}{n}\sum_i (y_i - h_{\target}(\vx_i))$
& $1$
& $\var(U + \xi)$
& $1 + \gamma$
& $0$ \\
\bottomrule
\end{tabular*}
\end{center}
where $s^2 / \sigma^2$ column follows from~\cref{lemma:characterizing-identified-sets}: $\var(U + \residual) = \sigma^2 + \sqNorm{h_{\erm} - h_{\target}}_X \leq \sigma^2(1 + \gamma)$. Note that the invariance budget $\eps$ does not require finite-sample adjustment since it is assumed to hold a.s. Therefore, the final budgets adjusted for finite-samples that~\eqref{eq:P2} uses in our experiments are
\begin{align*}
    % \boxed{
        \widehat{\gamma} \coloneqq (1+\gamma)\,\Gamma_{\! n}\Big(k; \frac{\gamma}{ 1+\gamma }\Big),
        &&
        \widehat{\gamma}_{\vz}&\coloneqq (1+\gamma)\,\Gamma_{\! n}\Big(d_z;\,  \frac{ \gamma_z }{ 1+\gamma } \Big),
        &&
        \widehat{\gamma}_{\mu} \coloneqq (1+\gamma)\,\Gamma_{\! n}(1;0),
    % }
\end{align*}
Consequently,~\eqref{eq:P2} is asymptotically valid with probability $\geq 1-\alpha$ by a union bound over the three constraints; we use $\alpha = 0.05$.
Post-DA, $\tilde{\gamma}, \tilde{\gamma}_{\vz}, \tilde{\sigma}$ replace
$\gamma, \gamma_{\vz}, \sigma$, and we keep $n = |\mathcal{D}|$ throughout since augmented
copies are dependent~\citep{cameron2015practitioner}.

\paragraph{Solvers and link $\Phi$.}
The simulation, optical device and cigarette demand experiments solve~\eqref{eq:P2}~via~QCQP under an identity link $\Phi = \Id$, while ${\rm do}$-MNIST uses SLSQP with $5$ multi-restarts across seeds to account for a non-linear probit link $\Phi$. The PI interval then reports the $\max$ over restarts for the upper-bound, and the $\min$  over restarts for the lower-bound.

\paragraph{Change of measure.}
Take $m$ samples $\vtau_{i, j}\sim p _T$ for each $i\in[n]$ with $n = \lvert\sD\rvert$. Construct dataset $\widetilde{\sD} \coloneqq \{(\tilde{\vx}_{i, j},\, y _i,\, \tilde{\vz}_{i, j})\}_{i\in[n], j\in[m]}$ with $\tilde{\vx}_{i, j} \coloneqq \vtau_{i,j}\vx_i$ and $\tilde{\vz}_{i, j} \coloneqq (\vz_{i}, \vtau_{i, j})$.

\paragraph{Budget re-calibration.} The post-DA budgets $\tgamma_\bullet$ are set as below per the attenuated confounding~(\cref{lemma:da-soft-intervention}) and symmetry misspecification~(\cref{thm:general-validity-b}), and swept across $[\tgamma_\bullet, \gamma_\bullet]$ in \cref{app:ablations} ablations,
\begin{align*}
    \tgamma = \left.\gamma\right/ \widehat{\rho}
    ,
    &&
    \tgamma_{\vz}(\eps)
    &=
    \pr*{
        \eps / \tsigma + \sqrt{\gamma_{\vz} / \rho}
    }^2
    .
\end{align*}

\clearpage

\section{Experiment Supplement}
\label{app:experiments}

\subsection{Experiment Setup}
\label{app:experiment-setup}
% \paragraph{Simulation.}
\textbf{Simulation.}\;
\label{sec:simulation-experiment}
\begin{figure*}[t]
    \centering
    \includegraphics[width=1.0\linewidth]{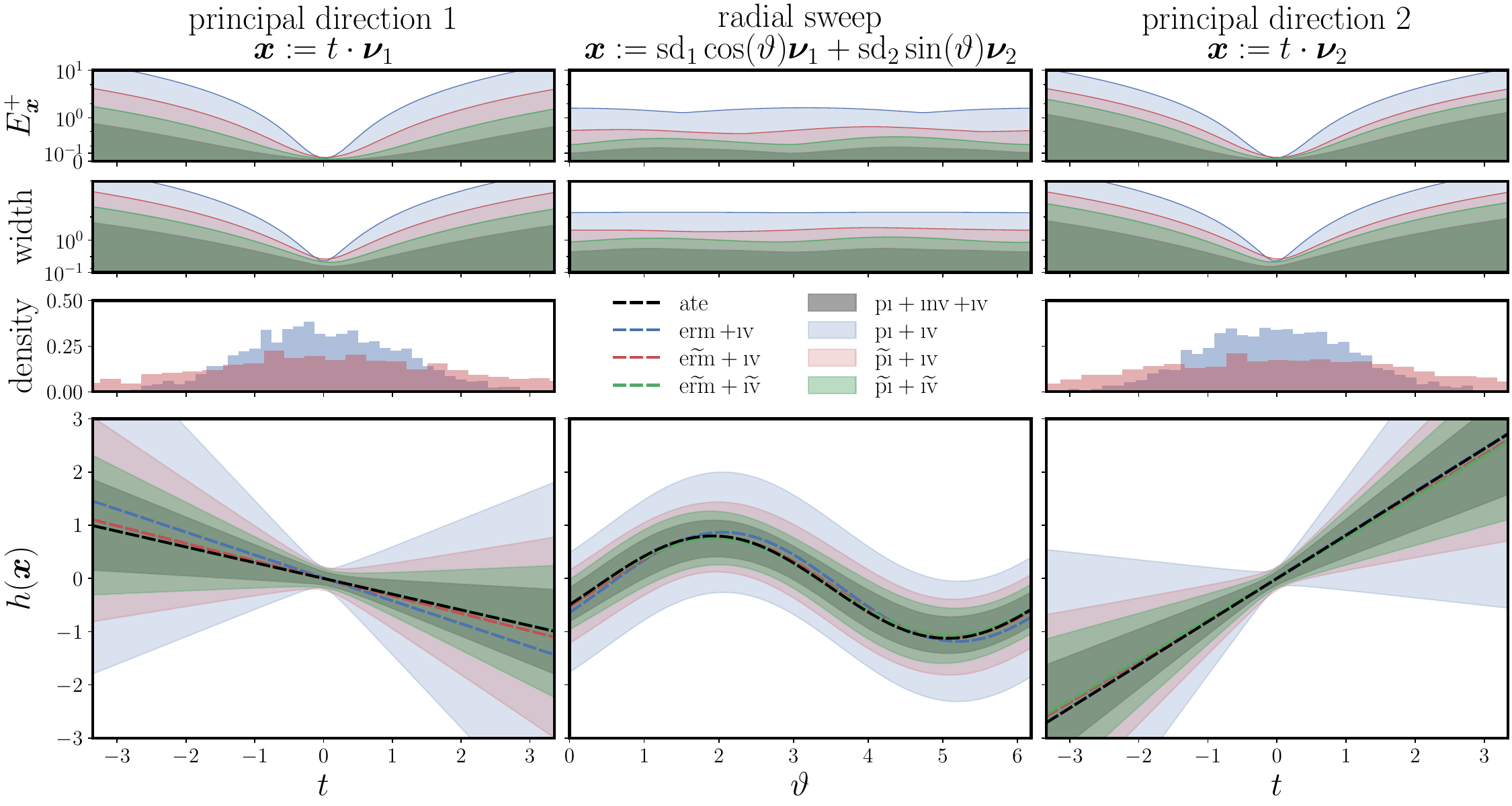}
    \captionof{figure}{
      \textbf{Simulation data.}
        Across queries $\vx$ along principal components $\bm\nu_1, \bm\nu_2$ of $\var(X)$, interval width and worst-case error {\figtop} decrease over baseline while maintaining validity {\figbottom}.
    }
    \label{fig:linear_simulation}
    \vspace{-1\baselineskip}
\end{figure*}
Fix $\vtheta_{\target}\sim\sN( {\0{}}, {\mI{}} )$ and sample $\mSigma_{u \vx}$ uniformly over a sphere of radius $\sqrt{\gamma^\target}$. Generate
\begin{align*}
    h_{\target}(X) \coloneqq \vtheta_{\target}^\top X,
    \qquad
    (U, X, Z) \sim \sN\pr*{{{\0{}}},{
         \begin{bmatrix}
            1 + \gamma^\target & \mSigma_{u \vx} & 0
            \\
            \mSigma_{u \vx}^\top & \mI & \mSigma_{\vx z}
            \\
            0 & \mSigma_{\vx z}^\top & 1
        \end{bmatrix}
    }},
\end{align*}
where $\mSigma_{\vx z}$ is sampled on a unit sphere orthogonal to $\mSigma_{u \vx}$, and $Y$ follows~\cref{eq:canonical-sem} with $\residual = 0$. Hence $Z$ is a valid but an incomplete instrument, so $h_{\target}$ is not point-identified. Following~\citet{causal-data-augmentation}, we keep each basis of $\Null{\vtheta_{\target}}$ with probability $0.8$, stack the $d_{\vtau}$ basis into $\mA\in\R^{d_{\vtau}\times d_{\vx}}$, and set
\begin{align*}
    \smash[t]{\tX} = \alpha(X, T) = X + \smash[t]{\mA^\top} T ,
    &&
    T\sim\mathcal{N}({\0{d_{\vtau}}}, {\mI_{d_{\vtau}}}) .
\end{align*}
$\smash[t]{\tX}$ varies only in $\Null{\vtheta_{\target}}$, so $h_{\target}$ is exactly $\sT$-invariant. Although this construction uses $\vtheta_{\target}$, $\mA$ stands in for domain knowledge of plausible symmetries; alone, it does not identify $\vtheta_{\target}$ without observing~$U$.

We take $d_{\vx}=k=32$ and $n=2048$. PI uses the program~\eqref{eq:P2} with an oracle budget $\gamma = \gamma^\target = 1$, and $\eps = \gamma_{\vz} = 2^{-8}$ as numerical tolerances. \Cref{fig:linear_simulation} shows the resulting intervals, where $\color{deep_blue}\erm+\iv$ is the anchor regression~\citep{anchor}, i.e. least squares within an IV solution set.

% \paragraph{Optical device.}
\textbf{Optical device.}\;
The benchmark of~\citet{optical-device} comprises 12 datasets, each with $n = 1000$ samples of $3\times3$ pixel images $X$ displayed on a laptop screen and the resulting photodiode voltages $Y$. A physical latent $U$ drives two LEDs, one affecting the webcam that captures $X$ and the other the photodiode that measures $Y$. $U$ is recorded but withheld from all methods; we use it only to estimate the target $h_{\target}(\vx)=\vtheta_{\target}^\top \phi_{\vx}(\vx)$ by back-door adjustment on $(\phi_{\vx}(X), U)$, where $\phi_{\vx}$ are polynomial features of the degree ${\rm deg}(h_{\target})\in\{1, \cdots, 5\}$ that best explains the data.\footnote{Lowest $5$-fold cross-validation error of back-door adjustment, e.g.\ ${\rm deg}(h_{\target})=2$ on dataset~$\#8$.} Post-DA features are $\phi_{\vx}(\tX)$, with DA composing random vertical/horizontal flips, $90^{\circ}$ rotations, each with probability $0.25$, and zero-sum translations that preserve the pixel sum $\bm{1}^\top\tilde{\vx} = \bm{1}^\top\vx$ via
\begin{align*}
    \tilde{\vx} = \vx + \vtau_{\rm 0},
    &&
    \vtau_{\rm 0}
    \sim
    \sN\pr*{
        0,
        \;
        0.25
        \,
        \var(x_{\rm pixels})
        \pr*{\mI - \tfrac{1}{9}\bm{1}\cdot\bm{1}^\top}
    }
    .
\end{align*}
The DA parameters $\vtau = [ \tau_{\rm vflip}, \tau_{\rm hflip}, \tau_{\rm rot}, \vtau_0^\top ]^\top$ with $\tau_{\rm vflip}, \tau_{\rm hflip}, \tau_{\rm rot}\in\{0, 1\}$ record the composition. As $Z=\varnothing$, the IV constraint enters only through $T$-as-IV with $\tZ = T$.~\Cref{fig:optical_device} uses dataset $\#8$, whereas~\cref{app:ablations} ablations aggregate all metrics across 10~datasets, each using oracle budgets $\gamma,\eps$ for~\eqref{eq:P2} unless stated otherwise. We exclude dataset $\#6$ (no confounding) and $\#11$ (pure confounding).

\newpage
\textbf{US cigarette demand.}
We use annual data for the 48 contiguous states and DC, 1970--2019 ($n=2450$), from the CDC, BEA and FRED~\cite{orzechowski2019taxburden,bea2025sainc1,bls_cpiaucsl}.
$Y$ is log cigarette packs sold per capita; $X$ is the state's cigarette price, the cheapest neighboring state's price, income, and the consumer price index (CPI), with $\phi_{\vx}=\log$ elementwise, so $h_{\target}(\vx)=\vtheta_{\target}^\top\phi_{\vx}(\vx)$.
State and neighbor price are confounded; income and CPI are assumed exogenous and, with the state cigarette tax, form the instrument set $Z$.
We partial out state fixed effects and a cubic time trend.
Symmetry assumes demand is unchanged when all prices and income scale by a common factor, i.e.\ $h_{\target}(\tau\vx)=h_{\target}(\vx)$ for $\tau>0$, which holds iff $\1^\top\vtheta_{\target}=0$ (imposed with $\eps=2^{-8}$); DA draws $\tX=TX$ with $\log T$ Gaussian.
\Cref{fig:cigarettes} pins $\gamma = 0.31$ per \citet{making-sensemakr}, benchmarking
against last year's sales, an omitted confounder capturing addiction \citep{becker-addiction},
and sweeps $\gamma \in [0.15, 0.45]$ at $\gamma_z = 0.01$. It then sweeps
$\gamma_z \in [10^{-4}, \gamma]$, as a plausible IV should leak less than the confounder
\citep{nevo-rosen}. Ablations (\cref{app:ablations}) keep the real $(X, Z)$ and
simulate $Y$ from a known $h_\star$ (the 2SLS fit projected onto $\bm{1}^\top\vtheta = 0$)
plus a synthetic confounder orthogonal to $Z$ with $\gamma^\star = 0.25$.

\begin{figure}[t]
    \centering
    \begin{subfigure}[b]{0.32\linewidth}
        \centering
        \begin{tikzpicture}[baseline=(Y.center)]
    \pgfmathsetmacro{\s}{1.5}
        \node[latent, thick, label={[font=\footnotesize]left:color}] at (0, \s) (C) {$C$};
        \node[latent, thick] at (1*\s, \s) (U) {$U$};
        \node[obs, thick, label={[font=\footnotesize]left:image}] at (0, 0) (X) {$X$};
        \node[obs, thick, label={[font=\footnotesize]right:label}] at (1*\s, 0) (Y) {$Y$};
    \path (X) edge [->, >={triangle 45}, thick, deep_red] node[above] {$h_{\target}\;$} (Y);
      \edge[thick] {U} {C};
      \edge[thick] {C} {X};
      \edge[thick] {U} {Y};
\end{tikzpicture}
        \caption{${\rm do}$-MNIST}
        \label{fig:do-mnist-sem}
    \end{subfigure}
    \hfill
    \begin{subfigure}[b]{0.32\linewidth}
        \centering
        \begin{tikzpicture}[baseline=(Y.center)]
    \pgfmathsetmacro{\s}{1.5}
        \node[latent, thick, label={[font=\footnotesize]left:camera}] at (0, \s) (C) {$C$};
        \node[latent, thick, label={[font=\footnotesize]right:LED}] at (1*\s, \s) (U) {$U$};
        \node[obs, thick, label={[font=\footnotesize]left:image}] at (0, 0) (X) {$X$};
        \node[obs, thick, label={[font=\footnotesize]right:photodiode}] at (1*\s, 0) (Y) {$Y$};
    \path (X) edge [->, >={triangle 45}, thick, deep_red] node[above] {$h_{\target}\;$} (Y);
      \edge[thick] {U} {C};
      \edge[thick] {C} {X};
      \edge[thick] {U} {Y};
\end{tikzpicture}
        \caption{Optical device}
        \label{fig:optical-sem}
    \end{subfigure}
    \hfill
    \begin{subfigure}[b]{0.32\linewidth}
        \centering
        \begin{tikzpicture}[baseline=(U.center)]
    \pgfmathsetmacro{\s}{1.5}
        \node[obs, thick, label={[font=\footnotesize]left:prices}] at (0, 0.5*\s) (X1) {$X_1$};
        \node[obs, thick, label={[font=\footnotesize]left:tax}] at (0.9*\s, \s) (Z) {$Z$};
        \node[latent, thick] at (0.9*\s, 0) (U) {$U$};
        \node[obs, thick, label={[font=\footnotesize]above:sales}] at (1.8*\s, 0.5*\s) (Y) {$Y$};
        \node[obs, thick, label={[font=\footnotesize]right:income}] at (2.7*\s, \s) (X2) {$X_2$};
        \node[obs, thick, label={[font=\footnotesize]right:CPI}] at (2.7*\s, 0) (X3) {$X_3$};
    \path (X1) edge [->, >={triangle 45}, thick, deep_red] node[fill=white, inner sep=1pt] {$h_{\target}$} (Y);
    \path (X2) edge [->, >={triangle 45}, thick, deep_red] (Y);
    \path (X3) edge [->, >={triangle 45}, thick, deep_red] (Y);
      \edge[thick] {Z} {X1};
      \edge[thick, dashed] {Z} {Y};
      \edge[thick] {U} {X1};
      \edge[thick] {U} {Y};
\end{tikzpicture}
        \caption{US cigarette demand}
        \label{fig:cigarette-sem}
    \end{subfigure}
    \caption{Causal graphs of the three real / semi-synthetic datasets used in our experiments.}
    \label{fig:sems}
    \vspace{-1\baselineskip}
\end{figure}
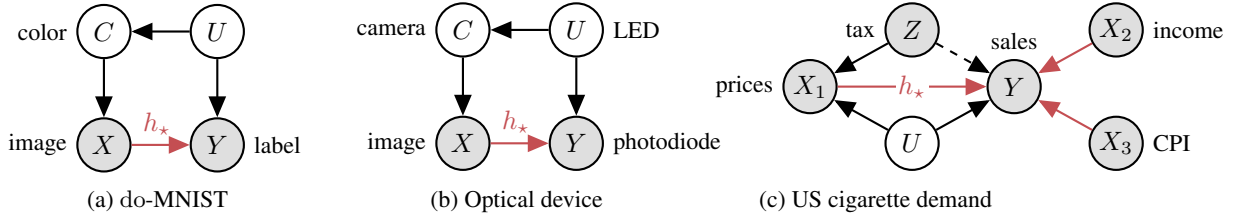
\textbf{${\rm do}$-MNIST.}\,
For a MNIST image $D$, we generate color $C$, colored image $X$ and label $Y\in\{0,1\}$ as
\begin{align*}
    C \coloneqq U \oplus {\rm Bern}(\mu_C),
    \qquad
    U \sim {\rm Bern}(\mu_U),
    \qquad
    \residual \sim {\rm Bern}(\mu_\xi),
    \\
    X \coloneqq {\rm color}(D, C),
    \qquad
    Y \coloneqq (1 - \residual)\cdot \1[{\rm digit}(X)\geq5] + \residual \cdot U,
\end{align*}
with $\mu_C=0.25,\mu_\residual=0.4,\mu_U = 0.5$, and ${\rm color}(\cdot, c)$ tints the $28\times 28$ digit red ($c=1$) or blue ($c=0$). Then, the ERM $\EE{}{}{Y\given X}\in\{0.1, 0.3, 0.7, 0.9\}$ tracking the color whereas the target ATE is
\begin{align*}
    h_{\target}(\vx)
    \;
    =
    \;
    0.6\cdot\1[{\rm digit}(\vx)\geq5] + 0.2
    \;
    \in
    \;
    \{0.2, 0.8\}
    .
\end{align*}
ERM uses the network architecture of~\citet{arjovsky2020invariant}, trained on $40$k images, and serves as the centroid of the CopSens PI model~\citep{CopSens}. CopSens sets the radius from a latent factor model $q_{U | X}$ (dimension $32$), fit on a further $10$k images. All methods share a single $\gamma=0.062$: the smallest budget at which baseline $\color{deep_blue}\pid$ covers $h_{\target}$ on $99.5\%$ of 10k samples, i.e.\ the baseline's tightest valid budget. DA composes $\pm10\%$ translations, $\pm10^{\circ}$ rotations, and contrast, saturation and hue jitter. The symmetry is exact, so $\eps=0.0475$ serves only as finite-sample slack. $5\%$ of augmented rows are left untransformed so that~\asmOneB{} holds. \Cref{tab:placeholder} reports mean $\pm$ one standard error over 10 seeds. % TODO label

\emph{Infeasibility of the shape constraint.} $\color{deep_grey}\pid+\inv$ is infeasible on $\rm do$-MNIST for any reasonably small $\eps$. This is because CopSens fixes all but the last layer of the ERM network~\citep{CopSens}, a standard device in neural PI for keeping bounds non-vacuous~\citep{causal-programming,iv_niki}. The frozen features are not invariant, so no last-layer weights can satisfy the invariance constraint. Unfreezing earlier layers would restore feasibility, but at the cost of vacuity and compute. DA, in contrast, needs no access to the solver, which is precisely the operational case made for it in~\cref{sec:implicit-constraint}.

\subsection{Ablations}
\label{app:ablations}
The following sweeps are repeated over 10 independent runs, each sampling from a different DGP. We report coverage of $h_\star(\bm x)$, interval width, and worst-case error $E^+_{\bm x}$, averaged over test queries $\bm x\sim p_X$; width and worst-case error are normalized by the baseline $\color{deep_blue}\pid+\iv$. Lines are aggregated over runs (arithmetic mean for coverage, geometric mean for width / worst-case error); bands span the $95\%$ CI across 1,000 bootstrapped replicates over runs. When one quantity is swept, all others are held at their oracle values $\gamma=\gamma^\star$, $\gamma_z=\gamma_z^\star$, $\varepsilon=\varepsilon^\star$. Here $\varepsilon^\star$ is the 95th percentile of $\abs{h_\star(X)-h_\star(\widetilde X)}$ on a fixed calibration sample, a heuristic since the sup is unbounded for linear $h_\star$. Unless stated otherwise, $m=1$, $n$ is the full dataset, and post-DA budgets are re-calibrated to $\tilde\gamma=\gamma/\widehat\rho$~and~$\tilde\gamma_z(\varepsilon)$~per~\cref{app:implementation-details}.

\newpage

\begin{figure}[!htbp]
    \centering
    \includegraphics[width=1.0\linewidth]{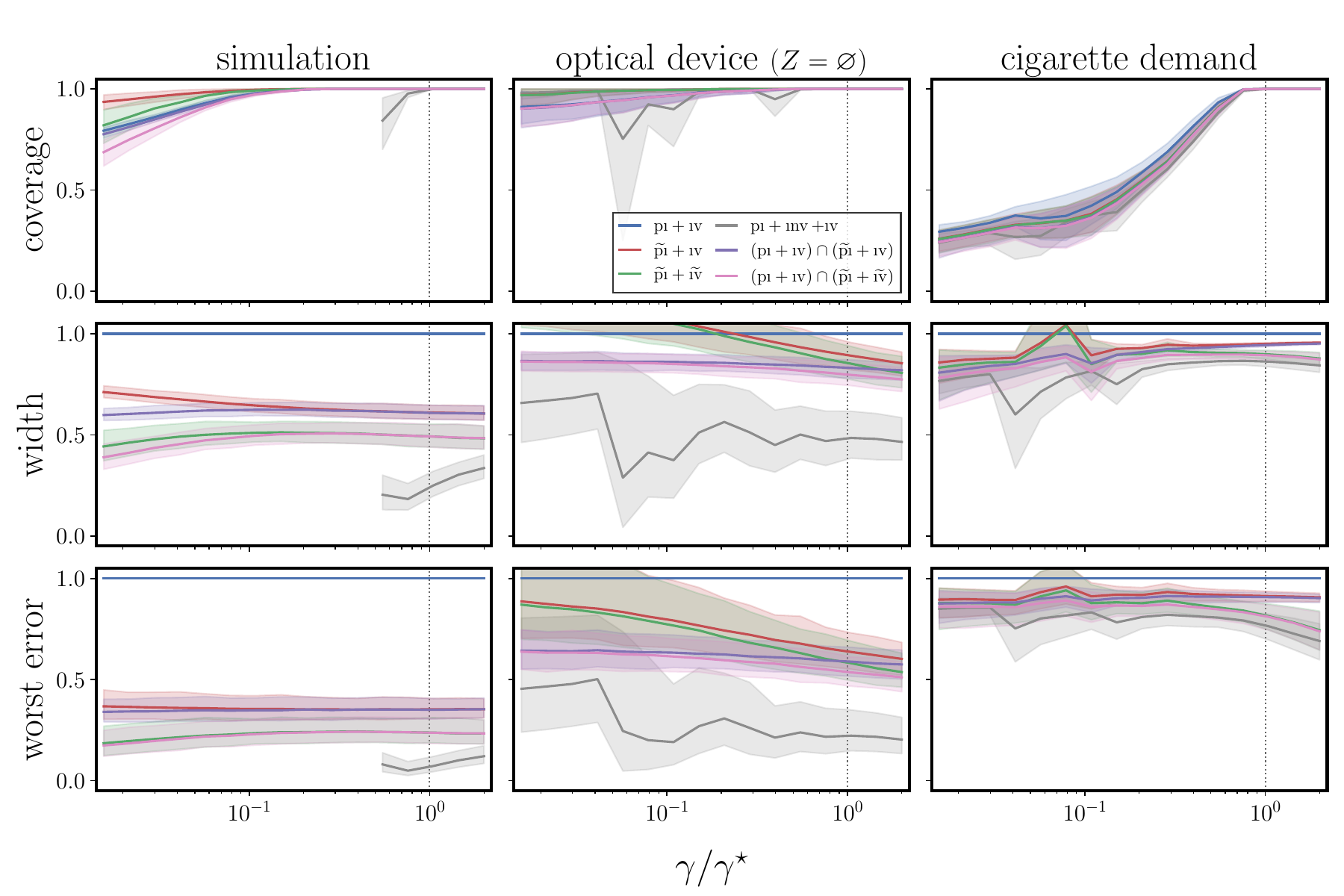}
    \caption{\textbf{Validity~(\cref{theorem:valid-bounds-with-da}).} Coverage, width, and worst-case error as the confounding budget $\gamma$ varies relative to its oracle $\gamma^\star$ (dotted). Post-DA sets attain coverage at smaller budgets than baseline.}
    \label{fig:validity}
    \vspace{-1\baselineskip}
\end{figure}

\begin{figure}[!htbp]
    \centering
    \includegraphics[width=1.0\linewidth]{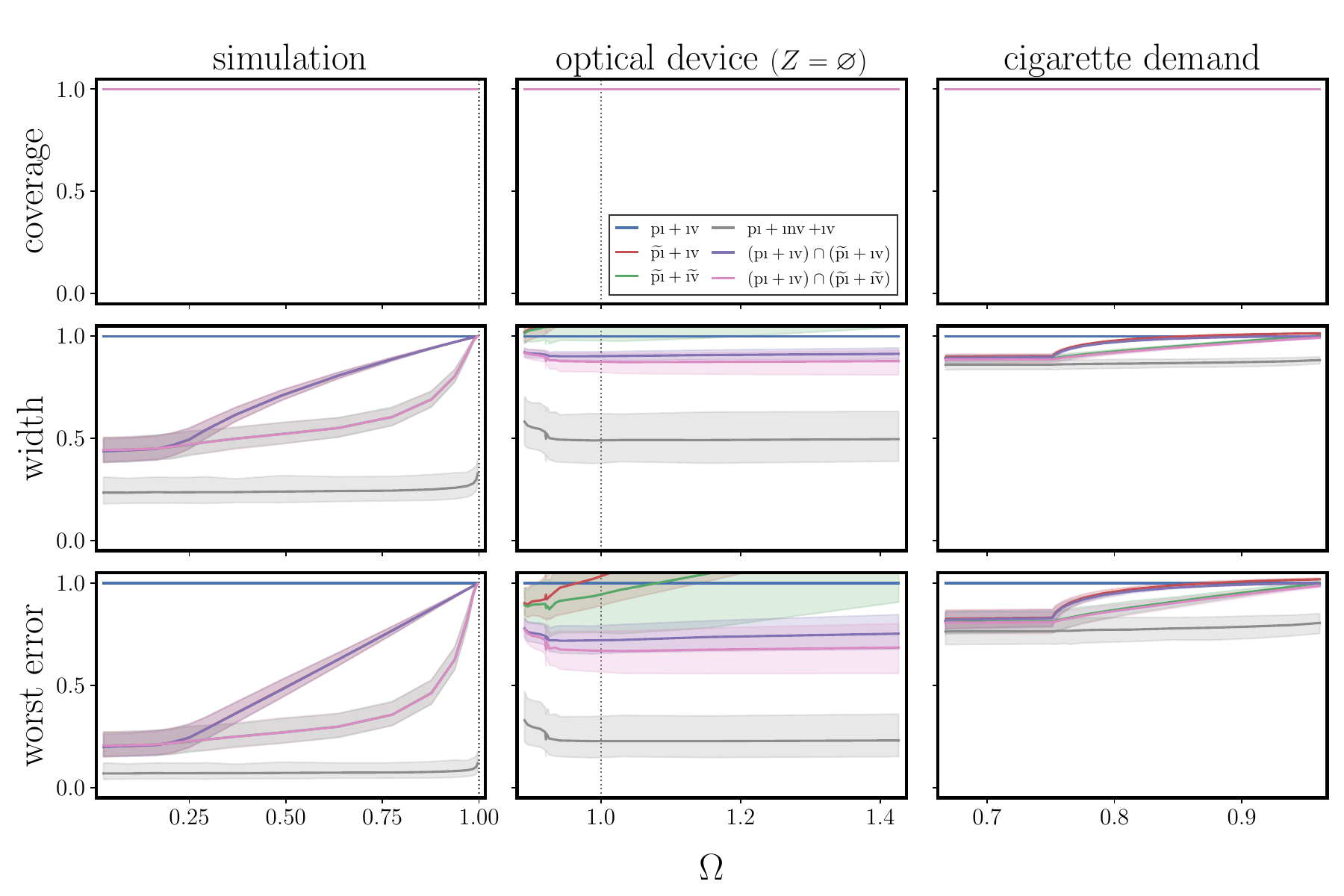}
    \caption{\textbf{Sharpness / informativeness~(\cref{proposition:da-sharper-bounds,theorem:informative-bounds-with-da}).} $\Omega$ is varied by applying augmentations of increasing strength, then computing $\widehat\Omega$ post hoc and sorting along the horizontal axis. For basis features $\phi_{\vx}$, let $\widehat\mSigma_V=\widehat{\operatorname{Cov}}(\phi_{\vx}(V))$ and $\widehat\delta=\widehat{\mathbb E}[\phi_{\vx}(X)]-\widehat{\mathbb E}[\phi_{\vx}(\widetilde X)]$. Then, we compute $\widehat\Omega$ by summing $\operatorname{tr}(\widehat\varShift) / k = \operatorname{tr}(\widehat\mSigma_{\widetilde X}^{-1}\widehat\mSigma_X) / k$ and $\widehat D_B^2 / k = \widehat\delta^\top\widehat\mSigma_{\widetilde X}^{-1}\widehat\delta / k$. Dotted: the threshold $\Omega=1$.}
    \label{fig:sharpness-informativeness}
    \vspace{-1\baselineskip}
\end{figure}

\begin{figure}[!htbp]
    \centering
    \includegraphics[width=1.0\linewidth]{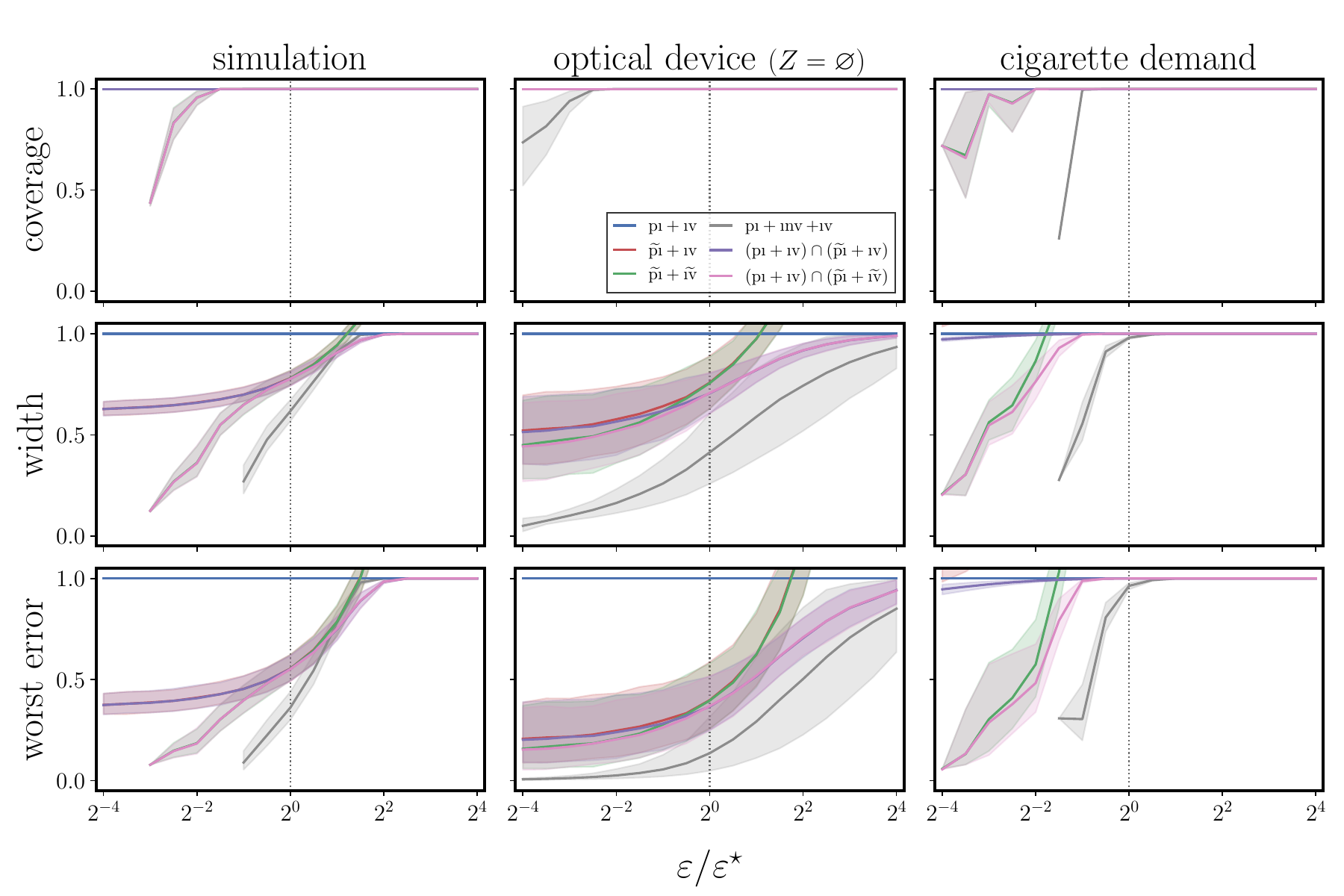}
    \caption{
        \textbf{Robustness (Thm. \protect\hyperlink{thm-general-validity}{3}).} We deliberately set $\varepsilon^\star$ to $\sigma\sqrt{\gamma^\star}$ to make symmetry misspecification as large as confounding itself. To this end, we add a symmetry-breaking corruption of strength $s$ to the DA (Gaussian noise for simulation and optical device, an extra translation off $\bm 1$ for cigarette demand), and bisect $s$ until the 95th percentile of $\abs{h_\star(X)-h_\star(\widetilde X)}$ matches $\sigma\sqrt{\gamma^\star}$. The assumed $\varepsilon$ is swept relative to $\varepsilon^\star$ (dotted); $\varepsilon$-padding restores coverage for $\varepsilon\ge\varepsilon^\star$, and intersecting with the baseline set prevents regression (\cref{remark:robust-bounds-with-da}). Surprisingly, all symmetry-informed bounds maintain improvement on both simulation and optical device up to $\eps = 16 \eps^\target$. For cigarette demand, $\color{deep_grey}\pid+\inv+\iv$ still improves at $\eps = \eps^\target$, but note that this is an abnormally large violation for a demand~system.
    }
    \label{fig:robustness}
    \vspace{-1\baselineskip}
\end{figure}

\begin{figure}[!htbp]
    \centering
    \includegraphics[width=1.0\linewidth]{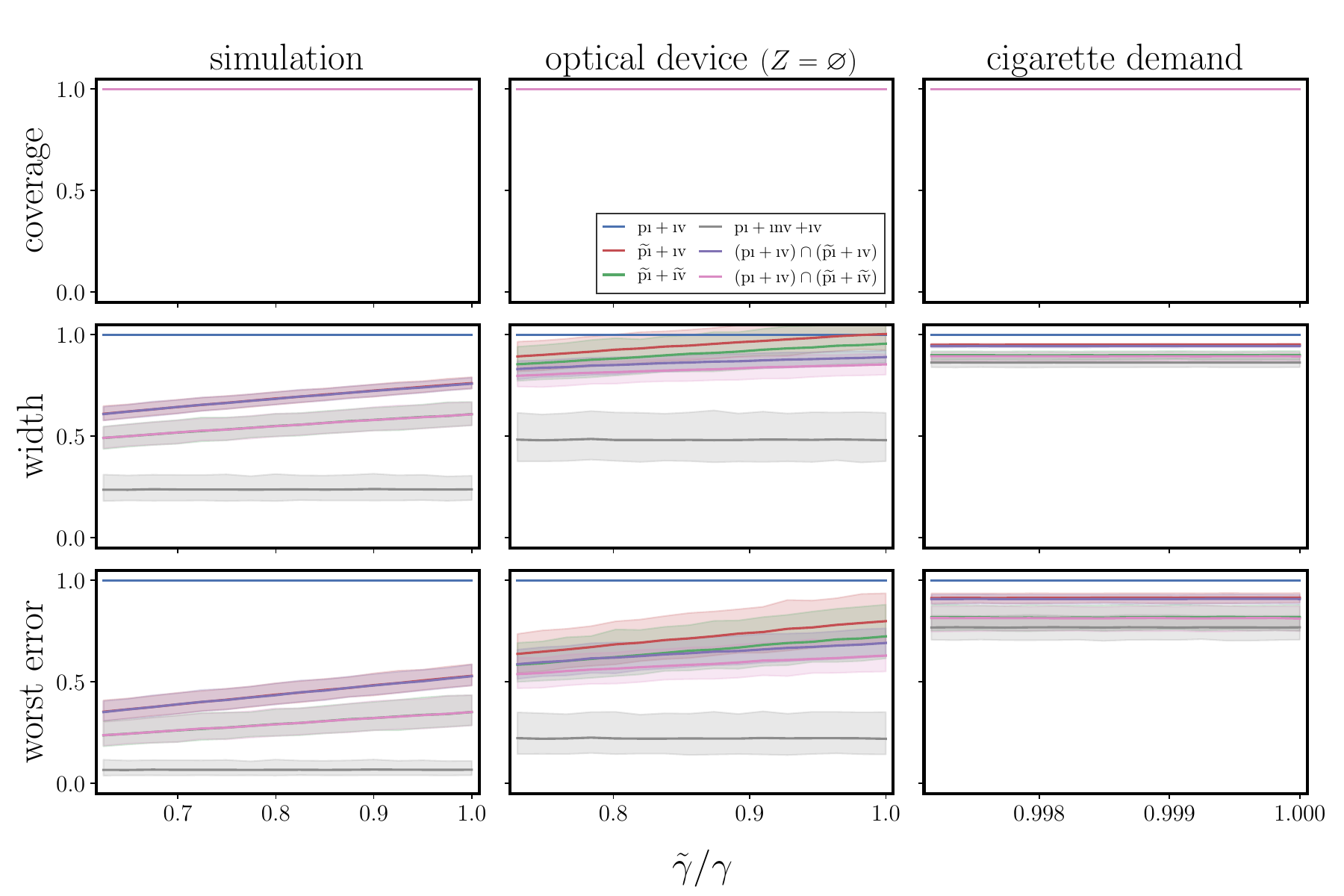}
    \caption{\textbf{Re-calibration.} The post-DA budget is swept over $\tgamma\in[\gamma/\widehat\rho,\,\gamma]$. Gains persist without re-calibration, particularly for intersected bounds. For cigarette demand $\widehat\rho\approx 1$, hence~the~narrow~range.}
    \label{fig:recalibration}
    \vspace{-1\baselineskip}
\end{figure}

\begin{figure}[!htbp]
    \centering
    \includegraphics[width=1.0\linewidth]{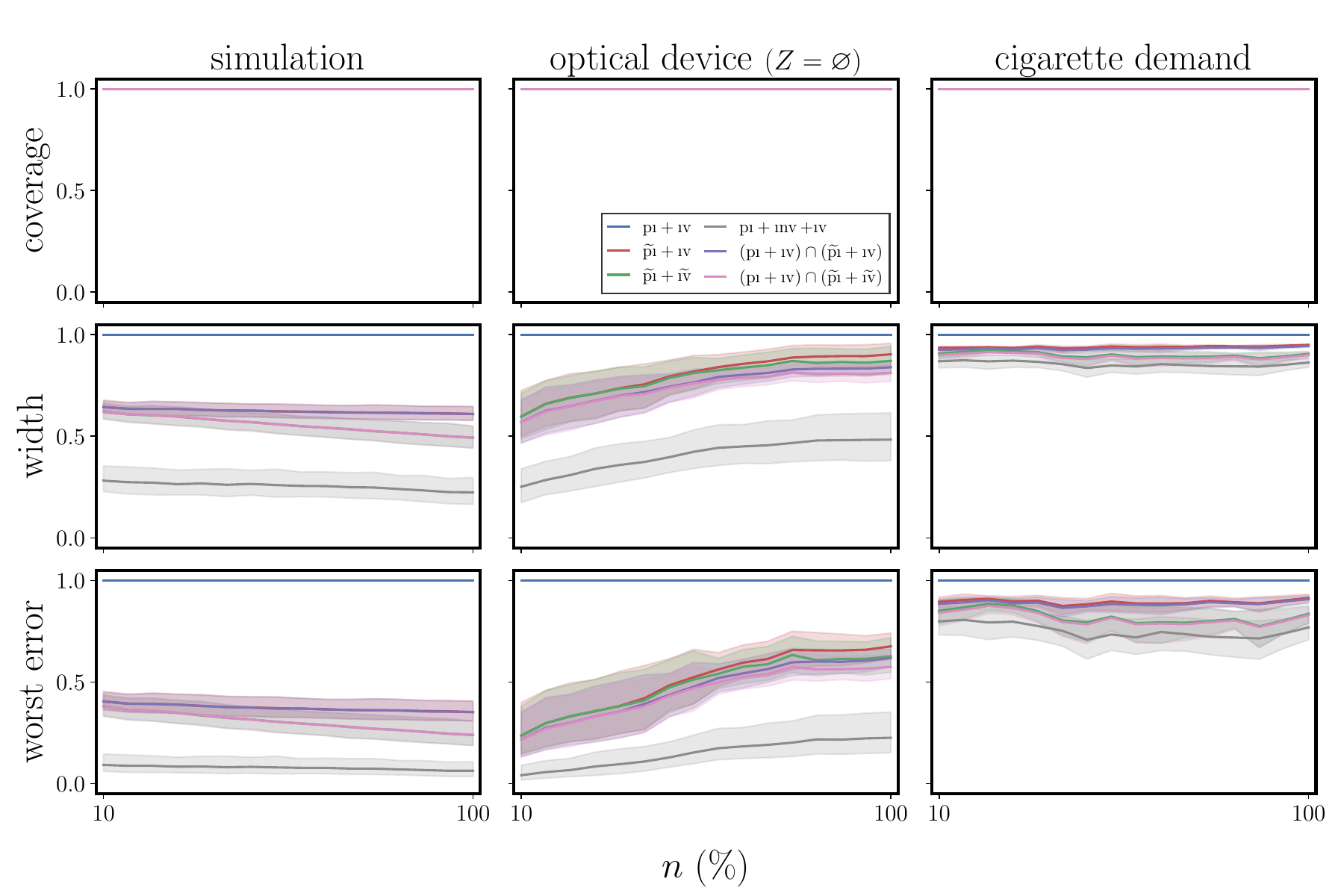}
    \caption{
        \textbf{Statistical efficiency in $n$.} The dataset is subsampled to $n\%$ of its full size; the augmented size matches $n$ since $m=1$. Symmetry-informed PI gives sharper bounds across sample sizes.
    }
    \label{fig:n-efficiency}
    \vspace{-1\baselineskip}
\end{figure}

\begin{figure}[!htbp]
    \centering
    \includegraphics[width=1.0\linewidth]{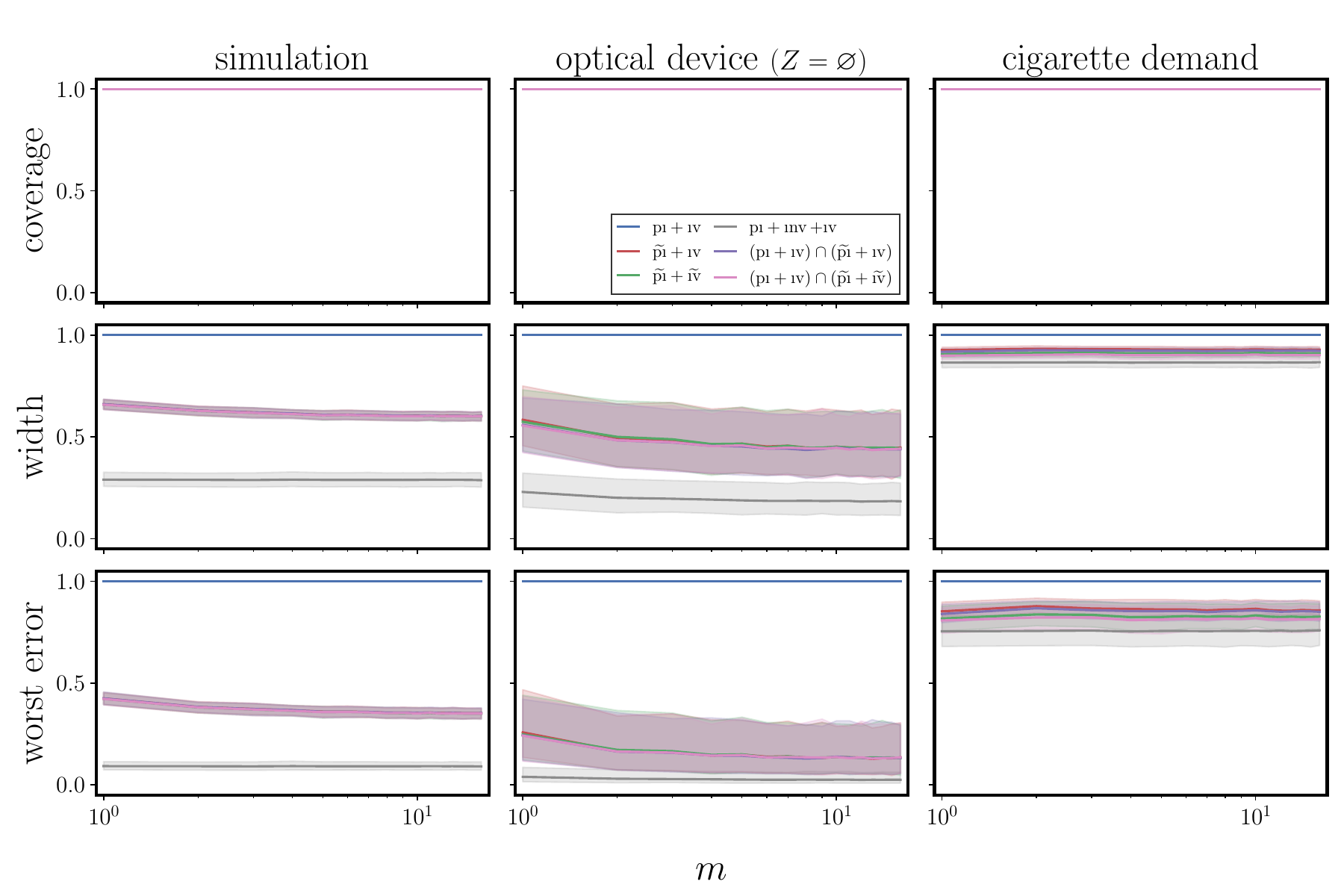}
    \caption{
        \textbf{Statistical efficiency in $m$.} Augmentation folds $m$ at $n=10\%$ of the data. Unlike elsewhere, the augmented dataset grows to $mn$, isolating the finite-sample benefit of DA. This effect is noticeable, albeit insignificant for optical device. Significant, but modest for simulation. And invisible in cigarette demand; we conjecture it scales with the number of directions excited by the symmetry.
    }
    \label{fig:m-efficiency}
\end{figure}

\clearpage
\section{Proofs}\label{app:proofs}

\section*{Technical Preliminaries}

\paragraph{Finite dimension.}
For $V\in\br*{X, \stX}$, define the centering operator $\sC_V\coloneqq h - \E[h(V)]$. Now, since $\sH + \R1 = \sH$~(\cref{assumption:hypothesis-class}), therefore $\sC_V h \in \sH_V$. These maps define a bijection between $\sH_X$ and $\sH_{\stX}$,
\begin{align*}
    \sH_X
    \overset{\sC_{\tilde{X}}}{\underset{\sC_X}{\rightleftarrows}}
    \sH_{\stX}
\end{align*}
therefore, $\dim(\sH_{\stX})=\dim(\sH_X)=k$~({\asmOneC}).

\paragraph{Non-degeneracy.}
$p_X \ll p_{\stX}$~({\asmOneB}) transfers Hilbert structure~({\asmOneC}) to $\sH_{\stX}$. For~$h\in\sH_{\stX}$,
\begin{align*}
    \tag{$0$ is the only a.s. constant $h\in\sH_{\stX}$.}
    \norm{h}_{\stX}
    =0
    \;\;
    &\implies
    \;\;
    h(\stX) = 0 \;\text{a.s.}
    \\
    % \tag{{\asmOneB}.}
    &\overset{\text{\asmOneShortB}}{\implies}
    \;\;
    h(X) = 0 \;\text{a.s.}
    \;\;
    \implies
    \;\;
    \norm{h}_X = 0
    \;\;
    \overset{\text{\asmOneShortC}}{\implies}
    \;\;
    h \in \R1
    \;\;
    \overset{h\in\sH_{\tX}}{\implies}
    \;\;
    h = 0
    .
\end{align*}
Hence both \(\langle\cdot,\cdot\rangle_X\) and \(\langle\cdot,\cdot\rangle_{\stX}\) are positive-definite Hilbert norms on \(\mathcal H_{\stX}\), but so is \(\langle\cdot,\cdot\rangle_{L_2(X)}\) because
\begin{align*}
    \sqNorm{ h }_{L_2(X)}
    =
    \sqNorm{h}_X
    +
    \E[h(X)]^2
    % (\E[h(X)] - \E[h(\tX)] )^2
    \geq 
    \sqNorm{h}_X
    .
\end{align*}

\paragraph{Shift operator $+$ decomposition.}
We give a basis-free form of the var.-/mean-shift identity in~\cref{sec:implicit-constraint}.

Under~{\asmOneShortBC}, Riesz gives a unique, positive-definite, self-adjoint \emph{shift operator} $\shift : \sH_{\stX}\to \sH_{\stX}$ that pulls the evaluation $L_2(X)$-inner product back into the solver's native $\pr*{\sH_{\stX}, \langle \cdot,\cdot\rangle_{\stX}}$-geometry,
\begin{align*}
    \langle f, \shift g\rangle_{\stX} = \langle f, g\rangle_{L_2(X)}
    ,
    \quad
    \forall
    f, g
    \in
    \sH_{\stX}
    .
\end{align*}
$\shift$ captures the fit vs. evaluation gap. And like $L_2(X)$-norm, its pull-back $\shift$ also admits a shape / level split into operator $\varShift\colon\sH_{\stX}\to\sH_{\stX}$ and linear functional $\meanShift\colon\sH_{\stX}\to\R$ defined respectively as
\begin{align*}
    \langle f, \varShift g\rangle_{\stX} = \langle f,g\rangle_X,
    \qquad
    \meanShift h = \EE{}{}{h(X)} - \EE{}{}{h(\stX)},
    \qquad
    \forall
    f, g, h
    \in
    \sH_{\stX}
    .
\end{align*}

\begin{minipage}{\columnwidth}
\begin{restatable}[shift decomposition]{lemma}{shiftDecomposition}
\label{lemma:shift-decomposition}
Under {\asmOneShortBC}, $\shift = \varShift + \meanShiftAdj\meanShift$,~with~$\varShift\mgt 0$,~$\meanShiftAdj\meanShift\mgeq~0$~s.t. $\func{\rm rank}{\meanShiftAdj\meanShift}\!\leq\! 1$, and $\tr{\meanShiftAdj\meanShift}\!=\!D_{\sB}^2(p_X, p_{\stX})$ is $\text{MMD}^2$ over the unit ball~$\sB\!\coloneqq\!\br*{h\in\sH_{\stX} :\norm{h}_\stX\leq 1}$.
\end{restatable}
\end{minipage}

The \emph{variance-shift operator} $\varShift$ governs how the estimation geometry changes between $\pr*{\sH_X, \langle \cdot,\cdot\rangle_X}$ and $\pr*{\sH_{\stX}, \langle \cdot,\cdot\rangle_{\stX}}$. It differs from $\shift$ precisely by the rank-1 \emph{mean-shift operator} $\meanShiftAdj\meanShift$ with trace $D^2_{\sB}$.

\paragraph{Invariant slice.}
Recall from~\cref{app:explicit-vs-implicit},
\[
\sI\coloneqq
\br*{\; f\in L^0_2(X)\cap L^0_2(\stX) \;\;\colon\;\; f(X)=f(\tX)\ {\rm a.s.} \;}
.
\]
Here \(\sI\) is a closed linear subspace, irrespective of any linearity assumed of
\(\sH\). Now define
\begin{align*}
\sH_{\inv}\coloneqq\sH\cap\sI
.
\end{align*}
By construction, all four forms coincide on $\sI$ and $\sH_{\inv}$,
\begin{align}
\label{eq:iv-vs-da-star}
    % \tag{$\square$}
    \langle f,g\rangle_{L_2(X)}
    \;\;=\;\;
    \langle f,g\rangle_{L_2(\stX)}
    \;\;=\;\;
    \langle f,g\rangle_X
    \;\;=
    \;\;
    \langle f,g\rangle_{\stX}
    ,
    &&
    \forall\,f,g\in\sI
    .
\end{align}
Covariance forms agree with $L_2$ as $\sI$ is both $X$- and $\tX$-centered. And $X$-, $\tX$-forms agree by~$\sT$-inv.

Throughout, we denote \(\sPi\) as the \(L_2(X)\)-orthogonal (equivalently, $X$-orthogonal)
projection onto \(\sI\).

\newpage
\subsection{Proof of \cref{theorem:valid-bounds-with-da}---Valid bounds with DA}
\label{app:theorem-valid-bounds-with-da}
{
\allowdisplaybreaks

\validBoundsWithDA*
\begin{proof}
We proceed as: (\rnum{1}) express baseline approx. error as point-to-ball distance \eqref{eq:point-to-ball}, (\rnum{2}) certify post-DA membership \eqref{eq:ball-def} by comparing bias magnitudes \eqref{eq:da-soft} using~\cref{lemma:da-soft-intervention}.

First, write the~\cref{lemma:characterizing-identified-sets} pre- and post-DA partially identified sets in their respective $L_2$ geometries as:
\begin{align*}
    \sH_{\pid} = \br*{
        h\in\sH_X : \sqNorm{
            h - h_{\erm}
        }_{L_2(X)} \leq \gamma \sigma^2
    }
    ,
    \\
    \sH_{\tpid} = \br*{
        h\in\sH_{\stX} : \sqNorm{
            h - h_{\term}
        }_{L_2(\stX)} \leq \gamma \tsigma^2
    }
    .
\end{align*}
Note that target $h_{\target}$ lies in both centered classes: $\EE{}{}{h_{\target}(X)} = \EE{}{}{Y}=0\implies h_{\target}\in\sH_X$, and by $\sT$-inv., $\E[ h_{\target}(\tX) ]=\EE{}{}{h_{\target}(X)}=0\implies h_{\target}\in\sH_{\stX}$. Define the pre- and post-DA bias magnitudes:
\begin{equation*}
    A\coloneqq\norm{h_{\erm}-h_{\target}}_{L_2(X)},
    \qquad
    B\coloneqq\norm{h_{\term}-h_{\target}}_{L_2(\stX)}.
\end{equation*}
For an $L_2(X)$-round ball, approx. error $E^-(\bullet) = \min\limits_{h\in\bullet}\Big\|h - h_{\target}\Big\|^2_{L_2(X)}$
\!\!\!\!\!\!\!\!\!\!\!\!\!\! reduces to point-to-ball~distance. 
Since $\sH_{\pid}$ is essentially $\sH_X \cap L_2(X)$-ball, approx. error is lower-bounded by point-to-ball distance,
\begin{equation}
    \tag{$\ast$}
    \label{eq:point-to-ball}
    \func{E^-}{\sH_{\pid}}
    \geq \pr[\big]{A-\sigma \sqrt{\gamma}}_+^2 ,
    \qquad
    h_{\target}\in\sH_{\pid}
    \iff A\le\sigma \sqrt{\gamma} .
\end{equation}
As for $\sH_{\tpid}$, the following validity criterion gives $\func{E^-}{\sH_{\tpid}}=0$ exactly,
\begin{equation}
    \tag{$\dagger$}
    \label{eq:ball-def}
    h_{\target}\in\sH_{\tpid}
    \iff B\le\tsigma \sqrt{\gamma}.
\end{equation}
For $C^2\coloneqq \tsigma^2 - \sigma^2$,~\cref{lemma:da-soft-intervention}~(\asmOneA, $\sT$-inv. $h_{\target}$) supplies the following two identities with $C^2\geq 0$,
\begin{align*}
    \tag{$\ddagger$}
    \label{eq:da-soft}
    A^2 = B^2 + C^2
    ,
    &&
    \tsigma^2
    =
    \sigma^2
    +
    C^2
    .
\end{align*}
Now suppose $\sH_{\pid}$ is $\epsilon$-valid with
\begin{equation*}
    \epsilon \;\le\; \epsilon^+
    \;\coloneqq\; \pr[\Big]{\sqrt{C^2 + \gamma\tsigma^2} - \sigma\sqrt{\gamma}}^2
    \;=\; \sigma^2 \pr[\Big]{\sqrt{\rho - 1 + \gamma\rho} - \sqrt{\gamma}}^2 ,
\end{equation*}
the second equality substituting $C^2 + \gamma\tsigma^2 = \sigma^2(\rho - 1 + \gamma\rho)$ from~\cref{eq:da-soft}. Note $\epsilon^+ > 0$ iff $\rho > 1$,~and
$$
\sqrt{\epsilon^+} = \sqrt{C^2 + \gamma\tsigma^2} - \sigma\sqrt{\gamma}
\geq 0
,
$$
since the difference defining $\epsilon^+$ is non-negative $C^2 + \gamma\tsigma^2 - \gamma\sigma^2 = (1+\gamma)C^2 \geq 0$. From~\cref{eq:point-to-ball},
\begin{align*}
    \tag{$\epsilon$-validity.}
    (A - \sigma\sqrt{\gamma})_+
    &\leq \sqrt{\epsilon},
    \\[2pt]
    \tag{$x\leq (x)_+$ and $\sqrt{\epsilon}\leq\sqrt{\epsilon^+}$.}
    \implies\quad
    A - \sigma\sqrt{\gamma}
    &\leq \sqrt{\epsilon^+} ,
    \\[2pt]
    \tag{Def. of $\epsilon^+$, and $\sqrt{\epsilon^+}\geq 0$.}
    \implies\quad
    A - \sigma\sqrt{\gamma}
    &\leq \sqrt{C^2 + \gamma\tsigma^2} - \sigma \sqrt{\gamma},
    \\[2pt]
    \tag{Add $\sigma\sqrt{\gamma}$, then square.}
    \implies\quad
    A^2
    &\leq C^2 + \gamma\tsigma^2 .
    \\[2pt]
    \tag{From \cref{eq:da-soft}.}
    \implies\quad
    B^2 &= A^2 - C^2
    \leq \gamma\tsigma^2 .
\end{align*}
Because $B\le\tsigma \sqrt{\gamma}$, thus by~\cref{eq:ball-def} $h_{\target}\in\sH_{\tpid}$, and therefore
\begin{equation*}
    \func{E^-}{\sH_{\tpid}}
    = 0
    \le \func{E^-}{\sH_{\pid}} .
\end{equation*}
Finally, equality holds precisely when $\func{E^-}{\sH_{\pid}}=0$, which by~\cref{eq:point-to-ball} is equivalent to $h_{\target}\in\sH_{\pid}$.
\end{proof}

% %SEND MEETS LINK

}

\newpage
\subsection{Proof of \cref{theorem:informative-bounds-with-da}---Informative bounds with DA}
\label{app:theorem-informative-bounds-with-da}
{
\allowdisplaybreaks

\informativeBoundsWithDA*
\begin{proof}
Gist of the proof: (\rnum{1}) reduce avg. error to second moment of a uniform law over a centered-ball, (\rnum{2}) connect it to post-DA ball via operator $\shift$, and (\rnum{3}) repeat for intervals via representers and~\cref{lemma:avg-representer-norms}.

\paragraph{Setup.}
Write the~\cref{lemma:characterizing-identified-sets,corollary:characterizing-identified-sets} pre- and post-DA PI sets in their respective $L_2$ geometries:
\begin{align*}
    \sH_{\pid} &= \br*{
        h\in\sH_X : \sqNorm{
            h - h_{\erm}
        }_{L_2(X)} \leq \sigma^2\gamma
    }
    ,
    \qquad
    \sH_{\pid}(\vx) = \br*{h_{\erm}(\vx)}
    \pm
    \sigma\sqrt{\gamma}\norm{g_{\vx}}_{L_2(X)},
    \\
    \sH_{\tpid} &= \br*{
        h\in\sH_{\stX} : \sqNorm{
            h - h_{\term}
        }_{L_2(\stX)} \leq \tsigma^2\gamma
    }
    ,
    \qquad
    \sH_{\tpid}(\vx) = \br*{h_{\term}(\vx)}
    \pm
    \tsigma\sqrt{\gamma}\norm{\tg_{\vx}}_{L_2(\stX)}
    ,
\end{align*}
with $L_2(X)$-representers $g_{\vx},\tg_{\vx}$ well defined on Hilbert $\sH_X$~({\asmOneC}) and $\sH_{\stX}$~({\asmOneB}, cf.~\cref{app:proofs}). Now, suppose $u$ is uniform on a centered $d$-dim. ball of radius $r$ w.r.t. \emph{any} inner product $\langle\cdot,\cdot\rangle$, then
\begin{align*}
    \EE{}{}{u} = 0,
    &&
    \EE{}{}{\langle u, \sA u\rangle} = \frac{r^2}{d+2} \tr{\sA} ,
\end{align*}
for any self-adjoint operator $\sA$ in $\langle\cdot,\cdot\rangle$. We show that avg. errors reduce to these identities~up~to~offset.

\paragraph{Set version ($d=k$).}
Recall $\overline{E}(\cdot) = \EE{h\sim{\rm Unif}(\cdot)}{}{ \sqNorm{ h - h_{\target}}_{L_2(X)} }$. Take the re-parameterizations
\begin{align*}
    h &= h_{\erm} + u, \qquad
    u\sim{\rm Unif}\br*{
        u\in\sH_X : \norm{u}_{L_2(X)}\leq\sigma\sqrt{\gamma}
    },
    \\
    h &= h_{\term} + \tilde{u}, \qquad
    \tilde{u}\sim{\rm Unif}\br*{
        u\in\sH_{\stX} : \norm{u}_{L_2(\stX)}\leq\tsigma\sqrt{\gamma}
    },
\end{align*}
for $h\sim{\rm Unif}(\sH_{\pid})$ and $h\sim{\rm Unif}(\sH_{\tpid})$ respectively. Substituting back into avg. error $\overline{E}(\sH_{\pid})$ def.:
\begin{align*}
    \overline{E}(\sH_{\pid}) &= \EE{u}{}{
        \sqNorm{
            h_{\erm} - h_{\target} + u
        }_{L_2(X)}
    },
    \\
    \tag{$\tr{\Id}=k$.}
    &= \sqNorm{
        h_{\erm} - h_{\target}
    }_{L_2(X)}
    + \EE{u}{}{
        \sqNorm{
            u
        }_{L_2(X)}
    }
    =
    E(h_{\erm}) + \frac{\sigma^2\gamma k}{k+2},
\end{align*}
where $\EE{}{}{u}=0$ kills the cross-term.
Likewise, for post-DA avg. error it follows from the def. of $\shift$:
\begin{align*}
    \overline{E}(\sH_{\tpid}) &= \EE{\tilde{u}}{}{
        \sqNorm{
            h_{\term} - h_{\target} + \tilde{u}
        }_{L_2(X)}
    },
    \\
    \tag{
        $\sqNorm{h}_{L_2(X)} = \langle h, \shift h \rangle_{L_2(\stX)}$ for $h\in\sH_{\stX}$.
    }
    &= \sqNorm{
        h_{\term} - h_{\target}
    }_{L_2(X)}
    + \EE{\tilde{u}}{}{
        \langle
            \tilde{u},
            \shift
            \tilde{u}
        \rangle_{L_2(\stX)}
    },
    \\
    &=
    E(h_{\term}) + \frac{\tsigma^2\gamma}{k+2} \tr{\shift}.
\end{align*}
Again, cross-terms vanish because $\EE{}{}{\tilde{u}}=0$. Subtracting, and with $\rho = \left. \tsigma^2 \right/ \sigma^2$ gives the~result~$f(k)$.

\paragraph{Interval version $(d = 1)$.}
A uniform draw from $[-w,w]$ has mean zero and variance $\left. w^2 \right/ 3$. Thus,
\begin{align*}
\func{\overline{E}_{\vx}}{\sH_{\pid}(\vx)} = \func{E_{\vx}}{h_{\erm}(\vx)} + \frac{1}{3}\sigma^2\gamma\sqNorm{g_{\vx}}_{L_2(X)}
,
\\
\func{\overline{E}_{\vx}}{\sH_{\tpid}(\vx)} = \func{E_{\vx}}{h_{\term}(\vx)} + \frac{1}{3}\tsigma^2\gamma\sqNorm{\tg_{\vx}}_{L_2(\stX)}.
\end{align*}
Averaging over $X\sim p_X$ via~\cref{lemma:avg-representer-norms} gives $\EE{X}{}{\sqNorm{g_X}_{L_2(X)}}=k$ and $\EE{X}{}{\sqNorm{\tg_X}_{L_2(\stX)}}=\tr{\shift}$. Substituting into pointwise approx. errors and subtracting gives the desired result $f(1)$.
\qedhere
\end{proof}
}

\newpage
\subsection{Proof of \cref{thm:general-validity,thm:general-validity-b}---Robust bounds with DA}
\label{app:general-validity}
{
\allowdisplaybreaks

\setcounter{theorem}{2}
\begin{theorem}[robust bounds with DA]
\label{thm:robust-da-iv}
Under {\asmOneShortAB}, suppose valid PI models following \cref{assumption:partial-r2,assumption:iv-model} yield respective intervals $\sH_{\pid}(\vx)$ and $\sH_{\pid+\iv}(\vx)$ covering an $(\eps, \sT)$-invariant $h_{\target}$. Then,
\begin{enumerate}[label=(\Alph*), leftmargin=*]
    \item $\eps$-padded post-DA intervals $\sH_{\tpid}(\vx)\pm\eps$ continue to remain valid for all $\eps\geq 0$ on $p_{X}$-a.e. $\vx$.
    \item $\eps$-padded $T$-as-IV intervals $\sH_{\tpid+\tiv}(\vx)\pm\eps$ are also valid for $\smash[t]{\tZ}\in\{Z, (Z, T)\}$ on $p_{X}$-a.e. $\vx$ if
\begin{align*}
    \eps \leq \tsigma \sqrt{\gamma_{\vz}} - \sigma \sqrt{\gamma_{\vz}^{\target}}
    ,
    &&
    \text{
        where $\gamma_{\vz}^{\target}\leq \gamma_{\vz}$ is the true IV budget.
    }
\end{align*}
\end{enumerate}
\end{theorem}

\begin{proof}
Sketch: (\rnum{1}) build a surrogate target $h_{\dagger}$ within $\eps$ of $h_{\target}$, (\rnum{2}) certify $h_{\dagger}\in\sH_{\tpid}$ by DPI (Part A), and (\rnum{3}) verify the IV moment via Minkowski (Part B); $\eps$-pad to convert $h_{\dagger}$ membership into validity

\paragraph{Surrogate.}
Define the surrogate as the smoothed target $h_{\dagger}\coloneqq \EE{}{}{h_{\target}(X)\given \stX = \cdot}\in\sH$~(\asmOneA).
Define $\mNoise\coloneqq h_{\target}(X)-h_{\target}(\tX)$, with $\abs{\mNoise}\leq\eps$ a.s.\ by $(\eps, \sT)$-invariance. Now, $h_{\dagger}$ obeys two identities:
\begin{align*}
    \tag{Take $\E[\,\cdot\, | \tX]$ in~\cref{eq:canonical-sem}.}
    h_{\dagger}(\tX) 
    &=h_{\term}(\tX)-\EE{}{}{U\given\stX}\vphantom{\tX}
    ,
    \\
    \tag{Take $\E[\,\cdot\, | \tX]$ in def. of $\eta$.}
    &=h_{\target}(\tX) + \EE{}{}{\mNoise\given\stX}\vphantom{\tX}
    ,
    \\
    \implies
    \tag{$\P$}
    \label{eq:eps-approx-inv}
    \big\| h_{\dagger} &- h_{\target} \big\|_{L_\infty(X)}
    \overset{\text{\asmOneB}}{\leq}
    \big\| h_{\dagger}-h_{\target} \big\|_{L_\infty(\stX)}
    \overset{\text{Jensen}}{\leq}
    \eps
    .
\end{align*}
Furthermore, by tower property $\E[h_{\dagger}(\tX)]=\EE{}{}{h_{\target}(X)}=0$, so surrogate is $\tX$-centered $h_{\dagger}\in\sH_{\stX}$.

\paragraph{Part (A): $h_{\dagger}\in\sH_{\tpid}$.}
The surrogate's offset from the post-DA center is $h_{\term} - h_{\dagger} = \EE{}{}{U\given\stX}$, so
\begin{align*}
    \tag{$\tX$-centered $h_{\dagger},\ h_{\term}$, $\Rightarrow \sqNorm{\cdot}_{L_2} = \Var(\cdot)$.}
    \sqNorm{h_{\term}-h_{\dagger}}_{L_2(\stX)}
    &=
    \var\big(\EE{}{}{U\given\stX}\big)
    ,
    \\%[-1ex]
    \tag{Contraction over Markov $U\!-\!X\!-\!\stX$.}
    &\leq \var\big(\EE{}{}{U\given X}\big)
    ,
    \\
    \tag{Baseline valid $h_{\target}\in\sH_{\pid}$~(\asmOneA,~\ref{assumption:partial-r2}).}
    &\leq \gamma\sigma^2
    ,
    \\
    \tag{DPI $\rho\geq 1$.}
    &\leq \gamma\tsigma^2
    .
\end{align*}
Post-DA,~\cref{lemma:characterizing-identified-sets} gives $\sH_{\tpid}$ as the $\sH_{\stX}\cap L_2(\tX)$-ball of sq.-radius $\tsigma^2\gamma$ about $h_{\term}$. So $h_{\dagger}\in\sH$ ({\asmOneA}) gives $h_{\dagger}\in\sH_{\tpid} \Rightarrow h_{\dagger}(\vx)\in\sH_{\tpid}(\vx)$. Hence $h_{\target}(\vx)\in\sH_{\tpid}(\vx)\pm\eps$ for $p_{X}$-a.e. $\vx$ by~\cref{eq:eps-approx-inv}.
% \vspace{-1.5em}
\paragraph{Part (B): $h_{\dagger}\in\sH_{\tpid+\tiv}$.}
To verify membership, we need to test the surrogate's residual $Y - h_{\dagger}(\tX)$ against the IV moment constraint in~\cref{assumption:iv-model}.
Note that this residual is mean-zero and reads as
\begin{align*}
    Y - h_{\dagger}(\stX)
    =
    \zeta + U + \residual,
    &&
    \zeta\coloneqq
    \eta - \EE{}{}{\eta\given\stX},
    &&
    \sqNorm{\zeta}_{L_2}
    =
    \EE{}{}{\var(\eta\given\stX)}
    \leq\eps^2
    .
\end{align*}
For any choice of IV $\stZ\in\br*{Z, (Z, T)}$, the conditional $\EE{}{}{U + \residual \given Z} = \EE{}{}{U + \residual \given \stZ}$ is conserved since $T\indep (\residual, U, Z)$. Writing the true budget $\left. \gamma^{\target}_{\vz}\coloneqq\sqNorm{\EE{}{}{U+\residual\given Z}}_{L_2} \right/ \sigma^2 \leq \gamma_{\vz}$~(\cref{assumption:iv-model}), we have
\begin{align*}
    \tag{Minkowski.}
    \norm*{\E[ Y-h_{\dagger}(\tX)\mid\tZ]}_{L_2}
    &\leq \norm*{\E[\zeta\mid\tZ]}_{L_2} + \norm*{\E[ U+\residual\mid\tZ]}_{L_2},
    \\
    \tag{Contraction; $T\indep(\residual,U,Z)$.}
    &\leq \eps + \sigma\sqrt{\gamma_{\vz}^{\target}},
    \\
    \tag{$\eps\leq\tsigma\sqrt{\gamma_{\vz}}-\sigma\sqrt{\gamma_{\vz}^{\target}}$.}
    &\leq \tsigma\sqrt{\gamma_{\vz}} .
\end{align*}
This is exactly the leaky-IV constraint of~\cref{assumption:iv-model} under the augmented law, so together with Part (A), $h_{\dagger}\in\sH_{\tpid+\tiv}$, and $h_{\target}(\vx)\in\sH_{\tpid+\tiv}(\vx)\pm\eps$ for $p_{X}$-a.e. $\vx$ uniformly over $\tZ\in\br*{Z, (Z, T)}$.
\end{proof}
\vspace{1\baselineskip}
\begin{remark}[conservativeness of~\cref{thm:general-validity-b}.]
\label{remark:conservative-iv}
    Minkowski assumes the worst-case where the two residual components are perfectly aligned. Instead, one~can~define
    $$
    \kappa := \left. \Big\langle \mathbb E[\zeta|\widetilde Z],\, \mathbb E[U{+}\xi|\widetilde Z]\Big\rangle_{\!\smash{L_2}} \right/ \pr*{\norm[\big]{\mathbb E[\zeta|\widetilde Z]}\cdot\norm[\big]{\mathbb E[U{+}\xi|\widetilde Z]} } \in [-1,1] ,
    $$
    as an alignment sensitivity parameter. The sufficient condition for post-DA leaky-IV validity is now
    \[
\varepsilon^2
+
2\kappa
\,
\varepsilon
\,
\sigma
\sqrt{\gamma_z^\star}
+
\sigma^2
\gamma_z^\star
\;\le\;
\tilde\sigma^2\gamma_z.
\]
    A conservative $\kappa = 1$ then recovers~\cref{thm:general-validity-b}, while a plausible $\kappa = 0$ (misspecification uncorrelated to confounding) gives the Pythagorean budget $\varepsilon \le \sqrt{\tilde\sigma^2\gamma_z - \sigma^2\gamma_z^\star}$, strictly dominating the smaller $\tilde\sigma\sqrt{\gamma_z} - \sigma\sqrt{\gamma_z^\star}$. The observable diagnostic $\varepsilon \le (\tilde\sigma - \sigma)\sqrt{\gamma_z}$ of~\cref{app:misspecified-da-under-iv-constraints} remains assumption-free.
\end{remark}

}

\newpage
\subsection{Proof of \cref{proposition:da-causal-estimation}---Estimation with DA}
\label{app:proposition-da-causal-estimation}
{
\allowdisplaybreaks

\causalEstimation*
\begin{proof}
Under {\asmOneA}, $h_{\erm}=\EE{}{}{Y|X}$ and $h_{\term}=\EE{}{}{Y|\stX}$. Taking conditional means in \cref{eq:canonical-sem},
\begin{align*}
    b = \EE{}{}{Y|X}-h_{\target} = \EE{}{}{U|X},
    \qquad\qquad
    \tb = \EE{}{}{Y|\stX}-h_{\target} = \EE{}{}{U|\stX},
\end{align*}
where the latter follows from $\sT$-inv. of $h_{\target}$. Note that $b,\tb\in\sH$ since $h_{\erm},h_{\term},h_{\target}\in\sH$~({\asmOneA}). Further,
$\EE{}{}{U}=0$ implies $b\in\sH_X$ and $\tb\in\sH_{\stX}$.
Now, re-write the error terms in respective~biases:
\begin{align*}
    E(h_{\erm}) = \sqNorm{b}_{L_2(X)},
    &&
    E(h_{\term}) = \sqNorm{\stb}_{L_2(X)}.
\end{align*}
On $\tX$-centered functions $h\in L_2^0(\tX)$, define the shift ratio $s(h)\coloneqq \left.\sqNorm{h}_{L_2(X)}\right/\! \sqNorm{h}_{\stX}$ ($s(0)\coloneqq 0$). Substitute into $E(h_{\term})$:
\begin{align*}
    E(h_{\term})
    &=
    \sqNorm{\stb}_{L_2(X)},
    \\
    \tag{Multiply and divide by $\sqNorm{\stb}_{\stX}$.}
    &=
    s(\stb) \cdot \sqNorm{\stb}_{\stX},
    \\
    \tag{$\tb\in\sH_{\stX}\implies \sqNorm{\stb}_{L_2(\stX)} = \sqNorm{\stb}_{\stX}$.}
    &=
    s(\stb) \cdot \sqNorm{\stb}_{L_2(\stX)} ,
    \\
    \tag{\cref{lemma:da-soft-intervention}.}
    &=
    s(\stb) \cdot \pr*{
        \sqNorm{b}_{L_2(X)} - (\tsigma^2 - \sigma^2)
    },
    \\
    \tag{Substitute $E(h_{\erm}) = \sqNorm{b}_{L_2(X)}$ and $\rho=\tsigma^2 / \sigma^2$.}
    &=
    s(\stb) \cdot \pr*{
        E(h_{\erm}) - \sigma^2 \cdot (\rho - 1)
    },
\end{align*}
Finally, subtract this from $E(h_{\erm})$ and re-arrange to get a general form of the proposition statement:
\begin{align*}
    E(h_{\erm})
    -
    E(h_{\term})
    =
    \sigma^2 s(\stb) \cdot (\rho - 1)
    +
    E(h_{\erm})
    \cdot
    (
        1 - s(\stb)
    )
    .
\end{align*}

\paragraph{Operator version under {\asmOneShortBC}.}
Recall from~\cref{app:proofs} that {\asmOneShortBC} gives $\sqNorm{h}_{L_2(X)} = \langle h, \shift h\rangle_{\stX}$ where operator $\shift$ has a Rayleigh quotient $\left.\shift[h] \coloneqq \langle h, \shift h\rangle_{\stX}^{\phantom{2}}\right/ \sqNorm{h}_{\stX}
\;(\shift[0]\coloneqq 0)$ and decomposition $\shift = \varShift + \meanShiftAdj\meanShift$~by~\cref{lemma:shift-decomposition}.

Trivially, substituting $s(\stb) = \shift[\stb]$ into the claimed identity now gives an operator version where $\shift[\stb] = \varShift[\stb] + (\meanShiftAdj\meanShift)[\stb]$ (\cref{lemma:shift-decomposition} and Rayleigh quotient is linear in the operator).
\end{proof}

}

\newpage
\subsection{Proof of \cref{proposition:da-sharper-bounds}---Sharper bounds with DA}
\label{app:proposition-da-sharper-bounds}
{
\allowdisplaybreaks

\daSharperBounds*
\begin{proof}
Sketching our steps: (\rnum{1}) \textbf{volume:} write the post-DA ball in an $L_2(X)$-orthonormal basis of $\sH_{\stX}$ and compare via AM-GM; (\rnum{2})~\textbf{width:}~write widths as representer norms and average via~\cref{lemma:avg-representer-norms}.
\paragraph{Setup.}
Write~\cref{lemma:characterizing-identified-sets,corollary:characterizing-identified-sets} pre- and post-DA identified sets in their respective $L_2$ geometries,
\begin{align*}
    \sH_{\pid} &= \br*{
        h\in\sH_X : \sqNorm{
            h - h_{\erm}
        }_{L_2(X)} \leq \sigma^2\gamma
    }
    ,
    \qquad
    \sH_{\pid}(\vx) = \br*{h_{\erm}(\vx)}
    \pm
    \sigma\sqrt{\gamma}\norm{g_{\vx}}_{L_2(X)},
    \\
    \sH_{\tpid} &= \br*{
        h\in\sH_{\stX} : \sqNorm{
            h - h_{\term}
        }_{L_2(\stX)} \leq \tsigma^2\gamma
    }
    ,
    \qquad
    \sH_{\tpid}(\vx) = \br*{h_{\term}(\vx)}
    \pm
    \tsigma\sqrt{\gamma}\norm{\tg_{\vx}}_{L_2(\stX)}
    .
\end{align*}
We measure both sets in ambient $L_2(X)$-volume and -width by bridging with $L_2(\stX)$ using operator~$\shift$.

\paragraph{Volume.}
Define the $L_2(X)$-volume of a set inside a $k$-dim. subspace as the Lebesgue volume of its coordinates in \emph{any} $L_2(X)$-orthonormal basis of that subspace; in particular $V_k\coloneqq \left.\pi^{k/2}\right/ \Gamma(k/2+1)$, the Euclidean unit-ball volume, is the $L_2(X)$-volume of a $k$-dim. unit ball in any $L_2(X)$-orthonormal basis. The baseline set is an $L_2(X)$-round ball of radius $\sigma\sqrt{\gamma}$ within $\sH_X$, so directly
\begin{align*}
    \abs{
        \sH_{\pid}
    } = V_k (\sigma\sqrt{\gamma})^k .
\end{align*}
For the post-DA set, take an $L_2(\stX)$-orthonormal eigenbasis $\br*{\psi_i}_{i=1}^k$ of $\shift$ on $\sH_{\stX}$ with eigenvalues $\lambda_i > 0$~({\asmOneShortBC}), and rescale $e_i\coloneqq \psi_i / \sqrt{\lambda_i}$. First, notice $\br*{e_i}\subset\sH_{\stX}$ is $L_2(X)$-orthonormal:
\begin{align*}
    \tag{Def.~of $\shift$.}
    \langle e_i, e_j\rangle_{L_2(X)}
    &=
    \langle e_i, \shift e_j\rangle_{L_2(\stX)},
    \\
    \tag{Def. of $e_i$, and $\shift\psi_j=\lambda_j\psi_j$.}
    &=
    \frac{\lambda_j}{\sqrt{\lambda_i\lambda_j}}
    \langle \psi_i, \psi_j\rangle_{L_2(\stX)}
    = \Id_{ij}
    .
\end{align*}
Therefore we can measure $L_2(X)$-volume in the coordinates $h - h_{\term} = \sum_i v_i e_i$. Second, we have:
\begin{align*}
    \tag{$L_2(\stX)$-orthonormality of $\psi_i$.}
    \langle e_i, e_j\rangle_{L_2(\stX)}
    &=
    \frac{1}{\sqrt{\lambda_i\lambda_j}}
    \langle \psi_i, \psi_j\rangle_{L_2(\stX)}
    = \frac{1}{\lambda_i}\,\Id_{ij}
    .
\end{align*}
Using this identity and linearity of inner-product, substitute $h - h_{\term} = \sum_i v_i e_i$ into the def. of $\sH_{\tpid}$:
\begin{align*}
    \sqNorm{
        h - h_{\term}
    }_{L_2(\stX)}
    \;=\;
    \left\langle
        \sum_i v_i e_i
        ,\,
        \sum_j v_j e_j
    \right\rangle_{\!\!\! L_2(\stX)}
    \!\!\!\!\!\!\!\!=
    \;
    \sum_i \sum_j v_i v_j
    \langle
        e_i
        ,
        e_j
    \rangle_{L_2(\stX)}
    \;=
    \boxed{
    \sum_{i=1}^k \frac{v_i^2}{\lambda_i}
    \leq \tsigma^2\gamma
    .
    }
\end{align*}
Out pops an ellipsoid with semi-axes $\tsigma \sqrt{\gamma}\cdot\sqrt{\lambda_i}$. Now, taking its volume and subsequently the ratio:
\begin{equation*}
\begin{gathered}
    \abs*{\sH_{\tpid}}
    \;
    =
    \;
    V_k(\tsigma\sqrt{\gamma})^k\prod_i\sqrt{\lambda_i}
    \;
    =
    \;
    V_k(\tsigma\sqrt{\gamma})^k(\det\shift)^{1/2},
    \\
    \implies
    \;
    \frac{\abs*{\sH_{\tpid}}}{\abs*{\sH_{\pid}}}
    \;
    =
    \;
    \rho^{\frac{k}{2}}\,(\det\shift)^{\frac{1}{2}}
    \;
    =
    \;
    \rho^{\frac{k}{2}}\pr[\Big]{\prod_i\lambda_i}^{\frac{1}{2}}
    \overset{\text{AM--GM}}{\leq}
    \rho^{\frac{k}{2}}\pr[\Big]{\frac{1}{k}\sum_i\lambda_i}^{\frac{k}{2}}
    \;
    =
    \;
    \pr[\Big]{\rho\cdot\frac{\tr{\shift}}{k}}^{\frac{k}{2}}
    ,
\end{gathered}
\end{equation*}
with equality iff eigenvalues $\lambda_i$ coincide, i.e. iff $\shift = \lambda \cdot \Id$ for some $\lambda > 0$.
\paragraph{Width.}
\cref{corollary:characterizing-identified-sets} gives half-widths at a query $\vx$ as $\sigma\sqrt{\gamma}\norm{g_{\vx}}_{L_2(X)}$ and $\tsigma\sqrt{\gamma}\norm{\tg_{\vx}}_{L_2(\stX)}$. Averaging their squares over $X\sim p_X$ using~\cref{lemma:avg-representer-norms},
\begin{align*}
    \EE{X}{}{\abs*{\func{\sH_{\pid}}{X}}^2}
    &= 4\sigma^2\gamma\,\EE{X}{}{\sqNorm{g_X}_{L_2(X)}} = 4\sigma^2\gamma\,k ,
    \\
    \EE{X}{}{\abs*{\func{\sH_{\tpid}}{X}}^2}
    &= 4\tsigma^2\gamma\,\EE{X}{}{\sqNorm{\tg_X}_{L_2(\stX)}} = 4\tsigma^2\gamma\,\tr{\shift} .
\end{align*}
Splitting $\tr{\shift}=\tr{\varShift}+\tr{\meanShiftAdj\meanShift}$ (\cref{lemma:shift-decomposition}), their ratio gives the claimed identity.
\end{proof}

}

\newpage
\subsection{Proof of \cref{proposition:inv-constraint-vs-da}---Shape constraint vs. change-of-measure}
\label{app:inv-constraint-vs-da}
{
\allowdisplaybreaks

\invConstraintVsDA*
\begin{proof}
Fix the common center and geometry, characterize the three sets as balls, then compare radii:

\paragraph{Shared center.} Under {\asmOneShortAB} and~\ref{assumption:compact-group}, \cref{lemma:haar-group-da} gives:\hfill $h_{\term} = \sPi h_{\erm} \in \sH_{\inv}$.\hfill\vphantom{.}

\paragraph{The post-DA set.} $\sH_{\tpid} = \br*{h\in\sH_{\stX} \colon \sqNorm{h - h_{\term}}_{\stX} \leq \tsigma^2\gamma}$ immediate from~\cref{lemma:characterizing-identified-sets} under $p_{\stX Y}$.
\paragraph{The $T$-as-IV set.} For $\tZ=T$, $\gamma_{\vz}=0$, the IV constraint of~\cref{corollary:sharp-iv-constraints} under $p_{\stX Y \stZ}$ filters $h\in\sH_{\stX}$~via
\begin{align*}
          \!\!\!\!\!\!\func{\Var}{\EE{}{}{Y - h(\stX) \given T }} &= 0 
             \;\iff{} \func{\Var}{\E[Y] - \EE{}{}{h(\stX) \given T }} = 0 
             \tag{$T\!\!\indep\! Y$ by construction.} \\
    \!\!\!\!\!\!\iff{}\quad\;\;\,\func{\Var}{\EE{}{}{h(\stX) \given T }}     &= 0 
             \;\iff{} \sE_{\stX | T} h                                     = 0 
             \qquad\qquad\qquad\qquad\qquad\quad
             \tag{Def. of $\sE_{\stX | T}$.} \\
    &\qquad\;\!\iff{}h\in\Null{\sE_{\stX | T} \given_{\sH_{\stX}}}  = \sH_{\inv} ,
             \tag{\cref{assumption:da-relevance-completeness}.}
    \\
    &\qquad\;\implies{}\sH_{\tpid+\tiv} = \sH_{\tpid} \cap \sH_{\inv}
    \tag{\cref{corollary:sharp-iv-constraints} under $p_{\stX Y \stZ}$.}
    .
\end{align*}

\paragraph{The inv. constraint set.} By \cref{corollary:sharp-inv-constraints}, $\sH_{\pid+\inv}=\sH_{\pid}\cap\sH_{\inv}$, with $\sH_{\pid}$ the $X$-ball of radius $\sigma\sqrt{\gamma}$ about $h_{\erm}$ (\cref{lemma:characterizing-identified-sets}). For invariant $h\in\sH_{\inv}$, using Pythagoras along $\sPi$ followed by $h=\sPi h$,
\begin{align*}
    \sqNorm{
        h - h_{\erm}
    }_X
    &=
    \sqNorm{
        \sPi (h -  h_{\erm})
    }_X
    +
    \sqNorm{
        (\Id - \sPi)(h - h_{\erm})
    }_X,
    \\
    \label{eq:iv-vs-da-dagger}
    \tag{$\triangle$}
    &=
    \sqNorm{
         h -  \sPi h_{\erm}
    }_X
    +
    \sqNorm{
        (\Id - \sPi)h_{\erm}
    }_X
    .
\end{align*}
Writing pre- and post-DA biases $b\coloneqq h_{\erm} - h_{\target}$ and $\tb\coloneqq h_{\term} - h_{\target}$, the second term reduces to:
\begin{align*}
    \tag{Def. of bias $b$, $\sT$-inv. $h_{\target}$.}
    \sqNorm{
        (\Id - \sPi)h_{\erm}
    }_X
    &=
    \sqNorm{
        (\Id - \sPi)(h_{\target} + b)
    }_X
    =
    \sqNorm{
        (\Id - \sPi)b
    }_X
    ,
    \\
    \tag{\cref{lemma:haar-group-da} $\tb=\sPi b$, Pythagoras.}
    &=
    \sqNorm{
        b - \stb
    }_X
    =
    \sqNorm{b}_X - \sqNorm{\stb}_X
    ,
    \\
    \tag{$\tb\in\sI\Rightarrow$~\cref{eq:iv-vs-da-star},~\cref{lemma:da-soft-intervention}.}
    &=
    \sqNorm{b}_X - \sqNorm{\stb}_{\stX}
    =
    \tsigma^2 - \sigma^2
    .
\end{align*}
Substituting into~\eqref{eq:iv-vs-da-dagger}, together with $\sPi h_{\erm} = h_{\term}$ and baseline ball constraint $\sqNorm{h-h_{\erm}}_{X}\le\sigma^2\gamma$,
\begin{align*}
    \implies
    \sqNorm{h - h_{\term}}_{X} \;\le\; \sigma^2\gamma - (\tsigma^2-\sigma^2) \;=\; \sigma^2(1+\gamma-\rho) .
\end{align*}

\paragraph{Inclusion and dominance hierarchy.} Both $T$-as-IV and inv.-constrained sets $\sH_{\tpid+\tiv},\,\sH_{\pid+\inv}$ are balls in $\sH_{\inv}$ about the common center $h_{\term}$. Since $h - h_{\term}\in\sI$ for invariant $h$, their norms are measured in the common geometry of~\cref{eq:iv-vs-da-star}, with squared radii obeying
\begin{align*}
    \tsigma^2\gamma - \sigma^2(1+\gamma-\rho) = \sigma^2(1+\gamma)(\rho-1) \;\ge\; 0 ,
\end{align*}
where $\rho\ge 1$ by the DPI (\cref{lemma:da-soft-intervention}). Hence $\sH_{\pid+\inv}\subseteq\sH_{\tpid+\tiv}$, while $\sH_{\tpid+\tiv}\subseteq\sH_{\tpid}$ is immediate.

Finally, set inclusion directly implies a dominance hierarchy in sharpness and worst-case errors.
\end{proof}

}

\newpage
\subsection{Proof of \cref{prop:tolerances}---Finite-sample budgets}
\label{app:tolerances}
{
\allowdisplaybreaks
\proptolerances*

\begin{proof}
Sketch: (i) rewrite the constraint as a Gaussian ball probability (Part A), (ii) sandwich the noncentral $\chi^2$ quantile between its median and Cantelli's bound (Part B).

\textbf{Setup.} Standardize by the sampling scale and write the empirical constraint as a ball,
\begin{gather*}
\lambda := \frac{\sigma^2 \gamma_\bullet}{s^2}, \qquad
t := \frac{s_0}{s} \in (0, 1], \qquad
\bm{u}_n := \frac{\sqrt{n}\, \bm{v}}{s}, \qquad
W_n := \frac{\sqrt{n}\,(\widehat{\bm{v}}_n - \bm{v})}{s} \rightsquigarrow t\zeta, \\
q_n := n\, \Gamma_n(d; \lambda), \qquad
B_n := \big\{ \bm{w} \in \mathbb{R}^d : \|\bm{w} + \bm{u}_n\|^2 \le q_n \big\}, \qquad
\zeta \sim \mathcal{N}(0, I_d),
\end{gather*}
with the limit given by Asm.~\ref{assumption:asymptotic-normal}. The ball $B_n$ is convex and contains the origin, since
\begin{align*}
\|\bm{u}_n\|^2 &\le n\lambda, && \text{(Population budget.)} \\
&< q_n. && \text{($\|\zeta + \sqrt{n\lambda}\, \bm{e}_1\|^2 > n\lambda$ whenever $\zeta_1 > 0$, i.e.\ w.p.\ $\ge \tfrac{1}{2} > \tfrac{\alpha}{3}$.)}
\end{align*}

\textbf{Part (A): coverage.}
\begin{align*}
\Pr\big(\|\widehat{\bm{v}}_n\|^2 \le \sigma^2 \widehat{\gamma}_\bullet\big)
&= \Pr(W_n \in B_n), && \text{(Scale by $n/s^2$; def.\ of $\widehat{\gamma}_\bullet$.)} \\
&= \Pr(t\zeta \in B_n) + o(1), && \text{(Uniform over convex sets; $t\zeta$ abs.\ cont.\ \citep{rao1962relations}.)} \\
&\ge \Pr(\zeta \in B_n) + o(1), && \text{($B_n$ convex, $0 \in B_n$, $t \le 1 \Rightarrow B_n \subseteq B_n / t$.)} \\
&\ge \Pr\big(\chi^2_d(n\lambda) \le q_n\big) + o(1), && \text{($\chi^2_d(\cdot)$ stoch.\ increasing; $\|u_n\|^2 \le n\lambda$.)} \\
&= 1 - \tfrac{\alpha}{3} + o(1). && \text{(Def.\ of $q_n$.)}
\end{align*}
Taking the $\liminf$ gives the coverage claim.

\textbf{Part (B): vanishing slack.} The law $\chi^2_d(n\lambda)$ has mean $d + n\lambda$ and variance $2(d + 2n\lambda)$, so Cantelli caps its quantile $q_n$ from above, while the Setup bounds it from below,
\begin{align*}
0 < \widehat{\gamma}_\bullet - \gamma_\bullet
&= \tfrac{s^2}{\sigma^2 n}\, (q_n - n\lambda), && \text{(Def.\ of $\widehat{\gamma}_\bullet$, $\lambda$; Setup.)} \\
&\le \tfrac{s^2}{\sigma^2 n} \big(d + c_\alpha \sqrt{2(d + 2n\lambda)}\big), && \text{(One-sided Cantelli, $c_\alpha := \sqrt{3/\alpha - 1}$.)} \\
&= O\big(\tfrac{d}{n} + \sqrt{\gamma_\bullet / n}\big). && \text{($\sqrt{a + b} \le \sqrt{a} + \sqrt{b}$; $\lambda \propto \gamma_\bullet$.)}
\end{align*}
\end{proof}

}

\newpage
\begin{figure}[th]
\centering
\[
\begin{tikzcd}[column sep=huge,row sep=large]
\sQ_{\pid}
  \arrow[r, dashed, "\Psi(\Phi(q);p_0)"]
  \arrow[d, maps to, "\Phi"']
&
\sQ_{\pid}\cap\sP_{\Psi}
  \arrow[d, maps to, "\Phi"]
\\
\sH_{\pid}
  \arrow[r, dashed, "\Psi(h;p_0)"']
&
\sH_{\pid}\cap\sH_\Psi .
\end{tikzcd}
\]
\caption{
A picture of~\cref{proposition:sharp-fiberwise-constraints}: Filtering baseline models $q\in\sQ_{\pid}$ by the predicate $\Psi\pr*{\Phi(q);p_0}$ and then mapping through $\Phi$ yields the same hypothesis set as mapping $\sH_{\pid}=\Phi(\sQ_{\pid})$ first and then filtering by the corresponding
function-level predicate $\Psi(h;p_0)$. Dashed arrows denote set restriction.
}
\label{fig:sharp-fiberwwise-constraints}
\end{figure}
\subsection{Proof of \cref{proposition:sharp-fiberwise-constraints}---Sharp bounds with structural-function constraints}
\label{app:sharp-fiberwise-constraints}
{
\allowdisplaybreaks

\begin{proposition}[sharp bounds with structural-function constraints]
\label{proposition:sharp-fiberwise-constraints}
Fix some distribution $p_0$. Define
\begin{align*}
    \Phi
    \colon
    \sP_{\!\pid}
    \to
    \sH
    ,
    \;
    q
    \mapsto
    h^q_{\target}
    .
\end{align*}
For a set of baseline PI candidate distributions $\sQ_{\pid}\subseteq\sP_{\!\pid}$, their candidate hypotheses are given by
\begin{align*}
    \sH_{\pid}
    \coloneqq
    \func{\Phi}{\sQ_{\pid}}
    .
\end{align*}
Let $\Psi(h;p_0)$ be a predicate depending only on functions $h\in\sH$ and the fixed law $p_0$. Define the sets
\begin{align*}
    \sH_\Psi\coloneqq \br*{
        h\in\sH
        \colon
        \Psi(h;p_0)
    }
    ,
    &&
    \sP_\Psi
    \coloneqq
    \br*{
        q\in\sP_{\!\pid}
        \colon
        \func{\Psi}{
            \func{\Phi}{q}
            ;
            p_0
        }
    }
    .
\end{align*}
Then $\sP_\Psi=\func{\Phi^{-1}}{\sH_\Psi}$ is the pre-image of $\sH_{\Psi}$ under $\Phi$, and the constrained identified set is \emph{sharp}:
\begin{align*}
    \func{\Phi}{
        \sQ_{\pid}
        \cap
        \sP_\Psi
    }
    =
    \sH_{\pid}\cap\sH_\Psi
    .
\end{align*}
The same identity holds for any finite collection of such predicates.~\cref{fig:sharp-fiberwwise-constraints} states the prop.~in~a~picture.
\end{proposition}

\begin{proof}
Re-write the claimed identity by substituting $\sH_{\pid} = \Phi(\cdot)$, $\sP_\Psi = \Phi^{-1}(\cdot)$,  $Q = \sQ_{\pid}$, $H = \sH_\Psi$,
\begin{align*}
    \func{\Phi}{
        Q\cap\func{\Phi^{-1}}{H}
    }
    =
    \Phi(Q)\cap H.
\end{align*}
\paragraph{$\Longrightarrow$ direction.}
Trivially, for any $h\in\func{\Phi}{
    Q\cap\func{\Phi^{-1}}{H}
}$, we have $h=\Phi(q)$ for some candidate $q\in Q\cap\func{\Phi^{-1}}{H}$. This directly implies $h\in\Phi(Q)$ and $h\in H$. Therefore, it holds that LHS $\subseteq$~RHS.

\paragraph{$\Longleftarrow$ direction.}
Let $h\in \Phi(Q)\cap H$. Then there exists $q\in Q$ such that $\Phi(q) = h$, which is in $H$. Since $\Phi(q)\in H$, then $q\in\Phi^{-1}(H)$. Thus, $q\in Q \cap \Phi^{-1}(H)$, giving $h = \Phi(q) \in \func{\Phi}{
    Q\cap\func{\Phi^{-1}}{H}
}$.

\end{proof}

\vspace{\baselineskip}

\begin{corollary}[sharp bounds with inv. constraints]
\label{corollary:sharp-inv-constraints}
Define hypotheses via the inv.-error functional~$E_{\inv}$
\begin{align*}
    \sH_{\inv}^{\eps} \coloneqq \br*{ h\in\sH : E_{\inv}(h)\le\eps^2 }
    .
\end{align*}
Then, the invariance shape constraint PI model from~\cref{sec:explicit-constraint} is \emph{sharp}, meaning $\sH_{\pid+\inv} = \sH_{\pid}\cap\sH_{\inv}^{\eps}$.
\end{corollary}
\vspace{-1\baselineskip}
\begin{proof}
    Immediate from~\cref{proposition:sharp-fiberwise-constraints} using $p_0 = p_{X\stX}$ and $\Psi(h;p_{0})= \mathbf{1}\br*{E_{\inv}(h)\leq \eps^2}$.
\end{proof}

\vspace{\baselineskip}

\begin{corollary}[sharp bounds with IV constraints]
\label{corollary:sharp-iv-constraints}
Let $\sigma_p^2$ be the ERM noise floor under true $p$. Define
\begin{align*}
    \sH_{\iv} \coloneqq \br*{
        h\in\sH :
        \var_{p}\pr[\Big]{
            \EE{p}{}{Y - h(X) \given Z }
        }
        \leq \sigma^2_p\gamma_{\vz}
    } ,
\end{align*}
as the leaky-IV hypotheses. Then, the PI model of~\cref{assumption:iv-model} is \emph{sharp}, meaning $\sH_{\pid+\iv} = \sH_{\pid}\cap\sH_{\iv}$.
\end{corollary}
\vspace{-1\baselineskip}
\begin{proof}
    Immediate from~\cref{proposition:sharp-fiberwise-constraints} using $p_0 = p_{XYZ}$ and the predicate
    \begin{equation*}
    \Psi(h;p_{0})= \mathbf{1}\br[\Big]{
        \var_{p_0}\pr[\Big]{ \EE{p_0}{}{ Y - h(X) \given Z} } \leq \sigma^2_{p_0} \gamma_{\vz}
    }.
    \qedhere
    \end{equation*}
\end{proof}

}

\newpage
\subsection{Proof of \cref{lemma:da-soft-intervention}---Soft intervention with DA}
\label{app:da-soft-intervention}
{
\allowdisplaybreaks

\softIntervention*
\begin{proof}
Part \emph{(a)} is immediate from def. of $\sT$-inv. and $T\indep (\residual, U, X, Y, Z)$. We prove parts \emph{(b)}--\emph{(c)}.

\paragraph{Part (b): Markov transfer.}
Write $W\coloneqq V\backslash\br*{X, Y}$, so $\tV = W\cup \br*{\stX, Y}$. Now, disintegrate $p_{\stV}$,
\begin{align*}
    p_{W\stX Y}
    =
    p_W
    \otimes
    p_{\stX | W}
    \otimes
    p_{Y | \stX W}
    ,
\end{align*}
and reduce each kernel to its $\tgraph$-kernel,
\begin{flalign*}
    \tag{$W$ ancestral in $\graph$, so its $\graph$-marginal factorizes.}
    & p_W = \bigotimes\nolimits_{V\in W} p_{V | \Pa{V}} , & 
    \\
    \tag{Integrate out $\vtau$.}
    & p_{\stX | W}\pr*{\,\cdot\,\given \vw} & 
    \\
    \tag{Use $T\indep (W, X)$ and $\tX = \alpha(X, \vtau)$ on $\br*{T=\vtau}$.}
    & =
    \int \alpha(\cdot, \vtau)_{\#}\, p_{X | W}\pr*{\,\cdot\, \given \vw}\; p_T(\diff \vtau) , & 
    \\
    \tag{$W$ non-desc., so $X\!\indep\! W\backslash \Pa{X}\, \big\vert\, \Pa{X}$.}
    & =
    \int \alpha(\cdot, \vtau)_{\#}\, p_{X | \Pa{X}}\pr*{\,\cdot\, \given \pa{\vx} }\; p_T(\diff \vtau) , & 
    \\
    & \eqqcolon
    p_{\stX | \Pa{X}}\pr*{\,\cdot\, \given \pa{\vx} } , & 
    \\
    \tag{Part \emph{(a)}: $Y = h_{\target}(\tX) + U + \residual$ with exogenous $\residual\indep (W, \tX)$.}
    & p_{Y | \stX W}\pr*{\,\cdot\,\given \vx, \vw} 
    &
    \\
    &
    =
    p_{\residual}\pr*{\,\cdot\, - h_{\target}(\vx) - u} \eqqcolon
    p_{Y | \stX U}\pr*{\,\cdot\,\given \vx, u}. & 
\end{flalign*}
Substitute back to get the desired factorization along $\tgraph$, i.e. all parent sets of $\graph$ with $X$ relabeled $\tX$:
\begin{align*}
    p_{W\stX Y}
    =
    \bigotimes\nolimits_{V\in W} p_{V | \Pa{V}}
    \;
    \otimes
    \;
    p_{\stX | \Pa{X}}
    \;
    \otimes
    \;
    p_{Y | \stX U}
    .
\end{align*}

\paragraph{Part (c): attenuated confounding.}
Under {\asmOneA}, $h_{\erm}=\EE{}{}{Y|X}$ and $h_{\term}=\EE{}{}{Y|\stX}$. Taking conditional means in \cref{eq:canonical-sem},
\begin{align*}
    b = \EE{}{}{Y|X}-h_{\target} = \EE{}{}{U|X},
    \qquad\qquad
    \tb = \EE{}{}{Y|\stX}-h_{\target} = \EE{}{}{U|\stX},
\end{align*}
where the latter follows from $\sT$-inv. of $h_{\target}$.

\textit{The identity:}
Since $\residual\indep(U,X)$, the noise floor $\sigma^2$ decomposes as
\begin{align*}
    \tag{{\asmOneA}.}
    \sigma^2
    &= \EE{}{}{\var(Y|X)}
    ,
    \\
    \tag{$\residual\indep(U,X)$.}
    &= \var(\residual) + \EE{}{}{\var(U|X)}
    ,
    \\
    \tag{Law of total variance.}
    &= \var(\residual) + \var(U) - \sqNorm{b}_{L_2(X)}
    ,
\end{align*}
and identical steps under $\tX$ ($\sT$-inv.\ at the first step) gives $\tsigma^2=\var(\residual)+\var(U)-\sqNorm{\stb}_{L_2(\stX)}$. Subtracting the two noise floors gives $\sqNorm{b}_{L_2(X)}-\sqNorm{\stb}_{L_2(\stX)}=\tsigma^2-\sigma^2$.

\textit{Non-negativity:}
As $\tX=\alpha(X, T)$ with $T\indep(U,X)$, the chain $U-X-\tX$ is Markov, hence:
\begin{align*}\tag{Tower rule.}
\tb=\EE{}{}{U\given\stX}=\EE{}{}{\EE{}{}{U|X}\given\tX}=\EE{}{}{b(X)\given\stX}
\end{align*}
Because conditional expectation is an $L_2$-contraction, we have:
\begin{align*}
    \sqNorm{\stb}_{L_2(\stX)} = \func{\Var}{\EE{}{}{U\given\stX}} \le \func{\Var}{\EE{}{}{U\given X}} = \sqNorm{b}_{L_2(X)} ,
\end{align*}
i.e.\ $\tsigma^2-\sigma^2=\sqNorm{b}_{L_2(X)}-\sqNorm{\stb}_{L_2(\stX)}\ge 0$.
\end{proof}

}

\newpage
\subsection[Proof of \cref{lemma:characterizing-identified-sets}---Characterizing the identified set]%
  {Proof of \cref{lemma:characterizing-identified-sets}---Characterizing the identified set%
   \protect\footnote{By convention, set $\gamma\coloneqq+\infty$ when $\gamma_{\vx}=1$ in~\cref{assumption:partial-r2}, and $r^2 \coloneqq 0$ whenever its denominator vanishes.}}
\label{app:characterizing-identified-sets}
{
\allowdisplaybreaks

\restatingtheoremtrue
\characterizingIdentifiedSets*
\restatingtheoremfalse
\begin{proof}
With $\sB$ the ball in the statement, the proof has three parts. (\rnum{1}) \textbf{Anchor:} all $q\in\sQ_{\pid}$ share the center $h_{\erm}$ and radius $\sigma$. (\rnum{2}) \textbf{Necessity $(\sH_{\pid}\subseteq \sB)$:} $h_{\target}^q\in\sB$ for every $q$ by bounding on budgets $\gamma_\vx,\gamma_y$; (\rnum{3}) \textbf{Sufficiency $(\sH_{\pid}\supseteq \sB)$:} Any $h\in\sB$ is achieved by~some~Gaussian certificate~$q^h\in\sQ_{\pid}$.
\paragraph{Anchor.} Denote $\sV\coloneqq \Span{\sH}$. The IPM test class $\sF = \br*{f^2 : f\in \sV + \R Y}$ in~\cref{assumption:partial-r2} matches $p_{XY}$ and $q_{XY}$ up to second moments of $(v(X), Y)$ for all $v\in\sV$. Hence, all $q$ induce a common risk
\begin{align*}
    R_{\erm}(h) = \E[Y^2] - 2\,\E[h(X)Y] + \E[h(X)^2] ,
    \qquad
    h\in\sH,
\end{align*}
with common minimizer $h_{\erm}$ and minimum $\sigma^2$.
Moreover,
\begin{align}
    \tag{$\clubsuit$}
    \label{eq:orthogonal-residual}
    \EE{p}{}{
        v(X)
        \cdot
        (Y - h_{\erm}(X))
    }
    &=
    0
    \qquad
    \forall v\in\sV
    .
\end{align}
Under~{\asmOneA}, this holds because $h_{\erm} = \EE{p}{}{ Y\given X }$; under~{\asmOneC}, $\sV = \sH$ and $h_{\erm} = \sPi_{\sH} Y$.

For $v = 1\in\sV$,~\cref{eq:orthogonal-residual} gives $\E[h_{\erm}(X)]=\E[Y]$, confirming that $h_{\erm}\in\sH_X$ under mean-zero $Y$.
\paragraph{Necessity $(\sH_{\pid}\subseteq \sB)$.}
Write~\cref{eq:canonical-sem} under some fixed $q\in\sQ_{\pid}$
\begin{align*}
    Y = h_{\target}^q(X) + U_q + \residual_q
    ,
    &&
    \residual_q
    \indep
    (U_q, X, Z)
    ,
\end{align*}
omitting the sub-/superscript $q$  where unambiguous. Define
\begin{align*}
    b^q
    \coloneqq
    h_{\erm}
    -
    h_{\target}^q
    \in
    \sH
    -
    \sH
    \subseteq
    \sV
    ,
    &&
    r_q
    \coloneqq
    \EE{q}{}{Y\given X}
    -
    h_{\erm}
    \perp_{L_2(X)}
    \sV
    ,
\end{align*}
where $r_q\perp_{L_2(X)} \sV$ holds because for every \(v\in\sV\), tower rule, moment matching, and
\cref{eq:orthogonal-residual} give
\[
\begin{aligned}
\E_q[v(X)r_q(X)]
&=
\E_q[v(X)Y]-\E_q[v(X)h_{\erm}(X)] \\
&=
\E_p[v(X)Y]-\E_p[v(X)h_{\erm}(X)]
=0.
\end{aligned}
\]
Conditioning the structural on $X$ now yields,
\begin{align}
    \label{eq:bias-residual-decomposition}
    \tag{$\spadesuit$}
    \EE{q}{}{U\given X}
    =
    \EE{q}{}{Y\given X}
    -
    h^q_{\target}
    =
    \underbrace{r_q}_{\perp\sV} + \underbrace{b^q\vphantom{_q}}_{\in\sV}
    .
\end{align}
with $r_q, b^q$ both $X$-centered. Now define:
\begin{align*}
    a
    \coloneqq
    \EE{q}{}{
        \var_q\pr*{U\given X}
    }
    ,
    &&
    c
    \coloneqq
    \var_q\pr*{
        \EE{q}{}{
            U\given X
        }
    }
    .
\end{align*}
To show that $h^q_{\target}\in\sB$, we need to show $\sqNorm{b^q}_X\leq \sigma^2\gamma$. This constitutes a chain of three inequalities,
\begin{align*}
    \sqNorm{b}_X
    \;
    \mathrel{
        \hypertarget{ineq:one}{
            \overset{(\ast\ast)}{\leq}
        }
    }
    \;
    c
    \;
    \mathrel{
        \hypertarget{ineq:two}{
            \overset{(\dagger\dagger)}{\leq}
        }
    }
    \;
    \frac{\gamma_{\vx}}{1 - \gamma_{\vx}} a
    \;
    \mathrel{
        \hypertarget{ineq:three}{
            \overset{(\ddagger\ddagger)}{\leq}
        }
    }
    \;
    \frac{\gamma_{\vx}\gamma_y}{1 - \gamma_{\vx}}
    \sigma^2
    \;
    \leq
    \;
    \sigma^2 \gamma
    .
\end{align*}
Here, (\hyperlink{ineq:one}{$\ast\ast$}) holds by Pythagoras on~\cref{eq:bias-residual-decomposition}. And (\hyperlink{ineq:two}{$\dagger\dagger$}) follows from the~$X$-budget
\begin{align*}
    r^2_{U\sim X} \coloneqq 1 - \frac{a}{\sigma^2_U} \leq \gamma_{\vx}
    \quad
    \implies
    \quad
    \sigma^2_U \leq \frac{a}{1 - \gamma_{\vx}}
    \quad
    \implies
    \quad
    c 
    \overset{\text{total-var.}}{=}
    \sigma^2_U - a
    \leq
    \frac{\gamma_{\vx}}{1 - \gamma_{\vx}} a
    .
\end{align*}
To establish (\hyperlink{ineq:three}{$\ddagger\ddagger$}), we need to upper-bound $a$. The $Y$-budget gives,
\begin{align}
\label{eq:r2y}
\tag{$\diamondsuit$}
    r^2_{U\sim Y | X} \coloneqq 1 - \left. \EE{}{}{ \var\pr*{ U \given X, Y } } \right/ a 
    \;\leq\;
    \gamma_y
    ,
    &&
    \implies
    &&
    (1 - \gamma_y)\, a \;\leq\; \EE{}{}{ \var\pr*{ U \given X, Y } }
    .
\end{align}
Now, condition on $X=\vx$ and write $s(\vx)\coloneqq \var\pr*{U\given X = \vx}$, so that $\E[s(X)]=a$. Note that $\EE{}{}{\var\pr*{U\given Y, X} \given X = \vx}$ is the MMSE of predicting $U$ from $Y$, so a \emph{linear} MMSE upper-bounds it:
\begin{align*}
    \EE{}{}{\var\pr*{U\given Y, X} \given X = \vx}
    \leq
    \frac{ s(\vx)\, \sigma^2_{\residual} }{ s(\vx) + \sigma^2_{\residual} }
    .
\end{align*}
The RHS is concave in \(s(\vx)\). For \(a>0\), Jensen and \cref{eq:r2y} give (conclusion immediate for $a = 0$),
\begin{align*}
    (1 - \gamma_y)
    \,
    a
    \leq
    \frac{a\, \sigma^2_{\residual}}{a + \sigma^2_{\residual}}
    ,
    &&
    \implies
    &&
    a
    \leq
    \gamma_y\,(a + \sigma^2_{\residual})
    .
\end{align*}
This, combined with the below inequality gives $a\leq \gamma_y\, \sigma^2$, finally establishing $\sqNorm{b}_X\leq \sigma^2\gamma$ via (\hyperlink{ineq:three}{$\ddagger\ddagger$}).
\begin{align*}
    a + \sigma^2_{\residual}
    \quad
    \overset{\text{total-var.}}{=}
    \quad
    \EE{}{}{
        \var\pr*{
            Y \given X
        }
    }
    \;
    \overset{
        \E[Y \vert X]
            \text{ optimality}
    }{\leq}
    \;
    \min\limits_{h\in\sH}
    \EE{}{}{ \pr*{ Y - h(X) }^2 }
    \quad
    \eqcolon
    \quad
    \sigma^2
    .
\end{align*}

\paragraph{Sufficiency ($\sH_{\pid}\supseteq \sB$).}
Fix $h\in\sB$, then define
\begin{align*}
    b^h \coloneqq h_{\erm} - h
    ,
    \qquad
    e^2\coloneqq \sqNorm{ b^h }_X \leq \sigma^2 \gamma
    .
\end{align*}
Suppose first that $\gamma_{\vx} > 0$ and $e > 0$. Define,
\begin{align*}
    \sigma^2_{W}
    \coloneqq
    \frac{1 - \gamma_{\vx}}{\gamma_{\vx}}
    e^2
    ,
    \qquad
    \sigma^2_{\residual}
    \coloneqq
    \sigma^2
    -
    \sigma^2_W
    ,
\end{align*}
Since \(e^2\leq\sigma^2\gamma\),
\[
\label{eq:sigma-W}
\tag{\(\heartsuit\)}
\implies\sigma_W^2
\leq
\frac{1-\gamma_{\vx}}{\gamma_{\vx}}\sigma^2\gamma
=
\gamma_y\sigma^2
\leq\sigma^2,
\]
so \(\sigma_{\residual}^2\geq0\). Sample independently
\begin{align*}
(X,Z)\sim p_{XZ},
\qquad
W_h\sim\sN(0,\sigma_W^2),
\qquad
\residual_h\sim\sN(0,\sigma_{\residual}^2),
\end{align*}
and define
\[
U_h\coloneqq b^h(X)+W_h,
\qquad
Y=h(X)+U_h+\residual_h
=h_{\erm}(X)+W_h+\residual_h.
\]
This satisfies the structural equation with \(h_{\target}^{q^h}=h\);
\(U_h\) and \(\residual_h\) are centered, and
\(\residual_h\indep(U_h,X,Z)\).

It remains to verify \(q^h\in\sQ_{\pid}\).

\textit{(a) Moment matching.}
For every \(v\in\sV\), \(q_X^h=p_X\) and
\(Y-h_{\erm}(X)=W_h+\residual_h\) is centered, independent of \(X\), and has
variance \(\sigma^2\). Thus
\[
\E_{q^h}[v(X)Y]
=
\E_p[v(X)h_{\erm}(X)]
=
\E_p[v(X)Y],
\]
by~\cref{eq:orthogonal-residual}, while
\[
\E_{q^h}[v(X)^2]=\E_p[v(X)^2],
\qquad
\E_{q^h}[Y^2]
=
\E_p[h_{\erm}(X)^2]+\sigma^2
=
\E_p[Y^2],
\]
the latter by Pythagoras under $p$. First moments match similarly.  Hence \(D_{\sF}(p_{XY},q^h_{XY})=0\).

\textit{(b) The $\gamma_{\vx}$ budget:} Since \(W_h\indep X\),
\[
a=\sigma_W^2,
\qquad
\sigma_U^2=e^2+\sigma_W^2=\left. e^2 \right/ \gamma_{\vx},
\]
and therefore
\begin{align*}
    r^2_{U\sim X}
    =
    1 - \frac{\sigma^2_W}{\left. e^2\right/ \gamma_{\vx}}
    =
    1 - (1 - \gamma_{\vx})
    =
    \gamma_{\vx}
    \leq
    \gamma_{\vx}
    .
\end{align*}
\textit{(c) The $\gamma_y$ budget:} Conditional on \(X\), observing \(Y\) is equivalent to observing
\(W_h+\residual_h\). Therefore
\begin{align*}
    \var_q(U \mid X = \vx, Y) = \var(W \mid W + \residual) = \frac{\sigma^2_W \cdot\sigma_\xi^2}{\sigma^2_W + \sigma_\xi^2} = \frac{\sigma^2_W\,(\sigma^2 - \sigma^2_W)}{\sigma^2},
\end{align*}
whence
\begin{align*}
    r_{U\sim Y|X}^2 = 1 - \frac{\sigma^2_W\,(\sigma^2 - \sigma^2_W)/\sigma^2}{\sigma^2_W} = \frac{\sigma^2_W}{\sigma^2} \leq \gamma_y,
\end{align*}
using $\sigma^2_W \leq \gamma_y\sigma^2$ from~\cref{eq:sigma-W}. Hence $q^h \in \sQ_{\pid}$.

\paragraph{Edge cases.}
If $e = 0$ (equivalently $h = h_{\erm}$ in $\sH_X$; this covers $\gamma_{\vx}=0$ or $\gamma_{y}=0$, where $\gamma=0$ and the ball degenerates to $\{h_{\erm}\}$), take $U_h \coloneqq 0$ and $\residual_h \sim \sN(0,\sigma^2)$. Moment matching \emph{(a)} holds as before, and both budgets \emph{(b)}--\emph{(c)} hold with $r^2 = 0$.

Hence, every \(h\in\sB\) is attained by a Gaussian certificate~(note, $p$ itself is \emph{not} assumed Gaussian).
\end{proof}

\vspace{1\baselineskip}

\begin{corollary}[characterizing the identified interval]
\label{corollary:characterizing-identified-sets}
    Under~{\asmOneC} and~\cref{assumption:partial-r2}, for $\vx\in{\rm supp}(p_X)$,
    \begin{align*}
        \sH_{\pid}(\vx) = \big\{ h_{\erm}(\vx) \big\} \pm \sigma\sqrt{\gamma}\,\norm{g_{\vx}}_{L_2(X)},
    \end{align*}
    where $g_{\vx}\in\sH_X$ is the $L_2(X)$-representer of $h\mapsto h(\vx)$, well defined by the non-degenerate $\langle\cdot,\cdot\rangle_X$.
\end{corollary}
\begin{proof}
For any $h\in\sH_{\pid}$ as given by~\cref{lemma:characterizing-identified-sets}, evaluation at query point $\vx\in{\rm supp}(p_X)$ follows
\begin{align*}
    \tag{Reproducing identity.}
    \abs*{h(\vx)-h_{\erm}(\vx)}
    &= \abs*{\langle h-h_{\erm},\,g_{\vx}\rangle_{L_2(X)}},
    \\
    \tag{Cauchy--Schwarz.}
    &\leq \norm{h-h_{\erm}}_{L_2(X)}\cdot\norm{g_{\vx}}_{L_2(X)}
    ,
    \\
    \tag{From \cref{lemma:characterizing-identified-sets}.}
    &\leq \sigma\sqrt{\gamma}\,\norm{g_{\vx}}_{L_2(X)},
\end{align*}
with equality at $h-h_{\erm}=\pm\sigma\sqrt{\gamma}\,g_{\vx}/\norm{g_{\vx}}_{L_2(X)}\in\sH$, membership by linearity of $\sH$~({\asmOneC}). These two extremal functions are attained since they lie in the ball of~\cref{lemma:characterizing-identified-sets} by construction.

 Finally, since $\sH_{\pid}$ is convex (a ball within linear $\sH$ under~{\asmOneC}) and evaluation $h\mapsto h(\vx)$ is linear, the image $\sH_{\pid}(\vx)$ is an interval, and therefore coincides with the full interval above.
\end{proof}

}

\newpage
\subsection{Proof of \cref{lemma:shift-decomposition}---Shift decomposition}
\label{app:shift-decomposition}
{
\allowdisplaybreaks

\shiftDecomposition*
\begin{proof}
We make three moves: (\rnum{1}) decompose the evaluation form as covariance-plus-mean, (\rnum{2}) read off the two operators it splits into via Riesz, and (\rnum{3}) identify the rank and trace of the mean part.

\paragraph{Evaluation form identity.}
Take $f,g\in\sH_{\stX}$. Scoring them in $L_2(X)$ gives
\begin{align*}
    \langle f, g \rangle_{L_2(X)}
    &=
    \ee\sqPr[\big]{
        f(X)
        \,
        g(X)
    }
    ,
    \\
    &=
    \Cov(
            f(X), g(X)
        )
    +\ee[f(X)]
        \,\ee[g(X)]
    ,
    \\
    &=
    \langle
        f, g
        \rangle_X
        +
        (\meanShift f)
        \,(\meanShift g)
        .
\end{align*}

\paragraph{Operator decomposition via Riesz.}
By~{\asmOneC}, $\pr*{\sH_X, \langle \cdot, \cdot \rangle_{X}}$ is Hilbert. By~{\asmOneB}~(cf.~\cref{app:proofs}),
\[
(\mathcal{H}_{\widetilde{X}}, \langle \cdot, \cdot \rangle_{\widetilde{X}})
\]
is also Hilbert, so every symmetric bilinear form on it is represented by a unique self-adjoint~operator.

Additionally, the forms $\langle \cdot, \cdot \rangle_{L_2(X)}$ and $\langle \cdot, \cdot \rangle_X$ are also both positive-definite on $\mathcal{H}_{\widetilde{X}}$~(\cref{app:proofs}).

Riesz gives unique self-adjoint, positive-definite $\shift,\varShift : \sH_{\stX} \to \sH_{\stX}$
satisfying
\begin{align*}
    \langle f, \shift g \rangle_{\stX}
    &=
    \langle f, g \rangle_{L_2(X)},
    &
    \langle f, \varShift g \rangle_{\stX}
    &=
    \langle f, g \rangle_X.
\end{align*}
Likewise, $\meanShift : \sH_{\stX} \to \R$ being a bounded linear functional has a unique representer
$\meanShift_\sharp \in \sH_{\stX}$~s.t.
\begin{align*}
    \meanShift h
    =
    \langle h, \meanShift_\sharp \rangle_{\stX},
    \qquad
    \forall\, h \in \sH_{\stX}.
\end{align*}
Therefore the adjoint $\meanShiftAdj c = c \meanShift_\sharp$, and
\begin{align*}
    \meanShiftAdj \meanShift
    &=
    \meanShift_\sharp \otimes_{\stX} \meanShift_\sharp,
    &&
    \langle
        f,
        (\meanShiftAdj\meanShift)g
    \rangle_{\stX}
    =
    (\meanShift f)(\meanShift g).
\end{align*}
Combining these identities, for every $f,g \in \sH_{\stX}$,
\begin{align*}
    \langle f, \shift g \rangle_{\stX}
    =
    \langle
        f,
        \pr[\big]{
            \varShift + \meanShiftAdj\meanShift
        } g
    \rangle_{\stX}.
\end{align*}
Non-degeneracy of $\langle \cdot, \cdot \rangle_{\stX}$ on $\sH_{\stX}$ therefore gives
\begin{align*}
    \shift
    =
    \varShift + \meanShiftAdj\meanShift.
\end{align*}
\paragraph{Rank and trace.}
Finally, the cross-product $\meanShiftAdj\meanShift
=
\meanShift_\sharp \otimes_{\stX}\meanShift_\sharp
\mgeq 0$
has image contained in $\Span{\meanShift_\sharp}$, and its only possible nonzero
eigenvalue is $\sqNorm{\meanShift_\sharp}_{\stX}$. Hence,
\begin{align*}
    \Rank{\meanShiftAdj\meanShift}
    \leq 1,
    \qquad
    \tr{\meanShiftAdj\meanShift}
    =
    \sqNorm{\meanShift_\sharp}_{\stX}.
\end{align*}
Moreover,
\begin{align*}
    D_{\sB}(p_X,p_{\stX})
    &=
    \sup_{h\in\sB}
    \abs{\meanShift h}
    =
    \sup_{\norm{h}_{\stX}\leq 1}
    \abs{\langle h,\meanShift_\sharp\rangle_{\stX}}
    =
    \norm{\meanShift_\sharp}_{\stX}.
\end{align*}
by Cauchy–Schwarz, attained at $h = \left.\delta_\sharp\vphantom{\big|}\right/\norm{\delta_\sharp}_{\stX}$. Therefore,
\begin{equation*}
    \tr{\meanShiftAdj\meanShift}
    =
    D_{\sB}^2(p_X,p_{\stX}).
    \qedhere
\end{equation*}
\end{proof}

}

\newpage
\subsection{Proof of \cref{lemma:informativeness-surrogate}---Average error as a worst-case surrogate}
\label{app:informativeness-surrogate}
{
\allowdisplaybreaks

{
\renewcommand{\pid}{}

\begin{lemma}[avg. error as worst-case surrogate]
\label{lemma:informativeness-surrogate}
Let $\sH_{\pid}(\vx)\subset \R$ be a bounded interval at query $\vx$. Under a squared-loss $\ell(a, b)\coloneqq (a - b)^2$, the following sandwich is tight in multiplicative constant~$4$:
\begin{align*}
    \func{ \overline{E}_{\vx} }{ 
        \func{\sH_{\pid}}{\vx}
    }
    \;
    \leq
    \;
    \func{ E^+_\vx }{
        \func{\sH_{\pid}}{\vx}
    }
    \;
    \leq
    \;
    4\cdot\func{ \overline{E}_{\vx} }{
        \func{\sH_{\pid}}{\vx}
    }
    .
\end{align*}
Furthermore, under~{\asmOneC} and \ref{assumption:partial-r2} the following inequality holds for the $k$-dim. ball $\sH_{\pid}$ from~\cref{lemma:characterizing-identified-sets}:
\begin{align*}
    \overline{E}(\sH_{\pid})
    \;
    \leq
    \;
    E^+(\sH_{\pid})
    \;
    \leq
    \;
    \pr*{k + 3}\cdot\overline{E}(\sH_{\pid})
    .
\end{align*}
\end{lemma}
\begin{proof}
The lower bound is trivial (max. dominates mean). For upper bounds: (\rnum{1}) write closed forms of both interval metrics in center offset and half-width, (\rnum{2}) compute a maximizer for their ratio, and (\rnum{3}) derive a common bound for $k$-dim. ellipsoids.

\paragraph{Closed forms.}
Write $\sH_{\pid}(\vx) = [c - w, c + w]$ with interval center $c$ and half-width $w\geq 0$. Let $d\coloneqq \abs{c - h_{\target}(\vx)}$ be the center's offset from the truth. Under sq.-loss, the pointwise excess risk of a candidate $y$ is $E_{\vx}(y)=(y - h_{\target}(\vx))^2$, so worst-case is attained at the endpoint farther from~$h_{\target}(\vx)$,
\begin{align*}
    E^+_{\vx}(\sH_{\pid}(\vx))
    &=
    \max\nolimits_{\abs{u}\leq w} (d + u)^2 = (d + w)^2,
\end{align*}
while a uniform draw decomposes into offset plus uniform second moment,
\begin{align*}
    \tag{$\E[u] = 0$ kills cross term.}
    \overline{E}_{\vx}(\sH_{\pid}(\vx))
    &=
    \E_{u\sim {\rm Unif}[-w, w]}[(d + u)^2] = d^2 + \frac{1}{3} w^2
    .
\end{align*}

\paragraph{Upper bound.}
Take $w > 0$, as otherwise the metrics are equal. Write their ratio in $t\coloneqq d/w \geq 0$,
\begin{align*}
    \frac{
        E^+_{\vx}(\sH_{\pid}(\vx))
    }{
        \overline{E}_{\vx}(\sH_{\pid}(\vx))
    }
    =
    \frac{
        (t+1)^2
    }{
        t^2 + 1/3
    }
    \eqcolon
    f(t)
    ,
    \quad
    \implies
    \quad
    f^\prime(t)
    =
    \frac{
        2 (t + 1) ( 1/3 - t )
    }{
        (t^2 + 1/3)^2
    }
    ,
\end{align*}
so $f$ increases on $t\in [0, 1/3]$ and decreases thereafter, with maximum
\begin{align*}
    f(1/3) = \frac{
        (4/3)^2
    }{
        (1/3)^2 + 1/3
    }
    =
    \frac{16/9}{4/9}
    =
    4
    .
\end{align*}
Hence, $E^+_{\vx}(\sH_{\pid}(\vx)) \leq 4\cdot \overline{E}_{\vx}(\sH_{\pid}(\vx))$, with equality exactly at the offset $d = w/3$.

\paragraph{$k$-dim.\ ellipsoids.}
Write an $L_2(X)$-ellipsoid $\sH=c+\sA \sB_2^k$, with center $c$ and $k$-dim. ball $\sB_2^k$. Let
\[
    d\coloneqq \|c-h_{\target}\|_{L^2(X)},
    \qquad
    v\coloneqq \frac{\operatorname{tr}(\sA^\ast \sA)}{k+2},
    \qquad
    R\coloneqq \sup_{\|u\|_2\leq 1}\|\sA u\|_{L^2(X)}.
\]
with $v$ the average sq.-radius of a centered $L_2(X)$-ellipsoid. Hence, for $U\sim{\rm Unif}(\sB_2^k)$,
\[
    \overline{E}(\sH)
    =
    \E\!\left[\|c+\sA U-h_{\target}\|_{L^2(X)}^2\right]
    =
    d^2+v
    .
\]
Therefore,
\begin{align*}
    E^+(\sH)
    &\leq
    (d+R)^2,
    \\
    &=
    d^2 + R^2 + 2dR,
    \\
    \tag{Cauchy-Schwarz $2ab\leq a^2 + b^2$ on $a = \frac{ dR }{ \sqrt{v} }, b=\sqrt{v}$.}
    &\leq
    d^2 + R^2 + \frac{d^2 R^2}{v} + v
    ,
    \\
    &=
    \pr[\bigg]{1+\frac{R^2}{v}}(d^2+v),
    \\
    \tag{Pythagoras $R^2\leq\operatorname{tr}(\sA^\ast \sA)=(k+2)v$.}
    &\leq
    \pr{1 + \pr*{k + 2}} (d^2+v),
    \\
    &=
    (k+3)\cdot \overline{E}(\sH).
\end{align*}
The bound is sharp: let one semi-axis be $1$, and remaining $k-1$ semi-axes tend to $0$. Place $h_{\target}$ on the longest-axis line at distance $v/R$ from the center. The ratio $E^+(\sH)\!\left/\vphantom{\Big|}\overline{E}(\sH) \right.$ then tends to $k+3$.
\end{proof}
}
}

\newpage
\subsection{Proof of \cref{lemma:haar-group-da}---Reynolds projection with DA}
\label{app:haar-group-da}
{
\allowdisplaybreaks
\begin{lemma}[Reynolds projection with DA]
\label{lemma:haar-group-da}
Under~{\asmOneShortAB}, and \ref{assumption:compact-group}, post-DA ERM $h_{\term}$ satisfies
\begin{align*}
    h_{\term}
    =
    \sPi h_{\erm}
    \in
    \sH_{\inv}
    ,
\end{align*}
where $\sPi$ is the $L_2(X)$-orthogonal proj. onto the ambient inv. subspace $\sI$ with $\sH_{\inv}=\sH\cap\sI$~(cf.~\cref{app:explicit-vs-implicit}). Furthermore, when the target $h_{\target}$ is $\sT$-invariant, the post-DA bias $\tb\coloneqq h_{\term} - h_{\target}$ similarly~follows
\begin{align*}
    \tb
    =
    \sPi b
    .
\end{align*}
\end{lemma}
\begin{proof}
We structure the proof in two parts: (\rnum{1}) the conditional expectation over augmentation acts as a proj. onto $\sI$; (\rnum{2}) both the post-DA ERM and bias naturally follow this conditional~expectation.

\paragraph{Reynolds projection with DA.}
Take the operator $(\sR h)(\cdot) \coloneqq {\E[ h(X) \mid \tX = \cdot]}$ on $h\in L^0_2(X)$, defined $p_{\tX}$-a.e., hence $p_X$-a.e.~({\asmOneB}). $\sR$ coincides with projection $\sPi$  under \cref{assumption:compact-group}, $\sR h = \sPi h$. To prove this, we need to show that it is invariant $\sR h \in \sI$, and its residual $h - \sR h$ is orthogonal to $\sI$.

\emph{(a) Invariance:} Since group $\sT$ is compact, the Haar measure $p_T$ is left-invariant, i.e. for any $\vtau\in\sT$,
\begin{align*}
    \tag{Def. of left-inv. $p_T$, follows from \cref{assumption:compact-group}.}
    && \vtau T
    &\overset{d}{=}
    T,
    \\
    \tag{Follows from $T\indep X$.}
    \Rightarrow
    && (\vtau T, X)
    &\overset{d}{=}
    (T, X),
    \\
    \tag{Apply $(\vtau, \vx)\mapsto (\alpha(\vx, \vtau), h(\vx))$ on both sides.}
    \Rightarrow
    && (\vtau \stX, h(X))
    &\overset{d}{=}
    (\stX, h(X)),
    \\
    \tag{Conditional expectation on both sides.}
    \Rightarrow
    && \ee[h(X) \mid \vtau \tX]
    &=
    \ee[h(X) \mid \tX],
    \\
    \tag{Re-parameterize to $\sR h$. $\sT$ group $\Rightarrow\vtau^{-1}$ exists.}
    \Rightarrow
    && (\sR h)(\vtau^{-1}\tX)
    &=
    (\sR h)(\tX).
\end{align*}
Because $\vtau$ is arbitrary, it follows that $(\sR h)(X) = (\sR h)(\stX)$, where the left-hand evaluation under $p_X$ is again licensed by~{\asmOneB}. Also, $\E[\;\! (\sR h)(\tX) \;\!] = \EE{}{}{h(X)} = 0$, hence $\sR h\in\sI$. Immediately,~if
\begin{align*}
    \sR h\in\sH
    \quad
    \implies
    \quad
    \sR h\in\sH_{\inv}
    .
\end{align*}

\emph{(b) Orthogonality:} The residual $h - \sR h$ is $L_2(X)$-orthogonal to every invariant test function $g\in \sI$:
\begin{align*}
    \tag{Both $\sR h$ and $g$ are $\sT$-inv.}
    \langle \sR h, g \rangle_{L_2(X)}
    \coloneqq
    \EE{}{}{ (\sR h)(X) \cdot g(X)}
    &=
    \EE{}{}{ (\sR h)(\stX) \cdot g(\stX)},
    \\
    \tag{Def. of $\sR$.}
    &=
    \EE{}{}{ \EE{}{}{h(X) \given \stX} \cdot g(\stX)}
    ,
    \\
    \tag{$g(\stX)$ is $\stX$-measurable.}
    &=
    \EE{}{}{ \EE{}{}{
        h(X)\cdot g(\stX)
        \given
        \stX
    } },
    \\
    \tag{Tower property.}
    &=
    \EE{}{}{ h(X)\cdot g(\stX) },
    \\
    \tag{$\sT$-inv. of $g$.}
    &=
    \EE{}{}{ h(X)\cdot g(X)}
    =
    \langle h, g\rangle_{L_2(X)} .
\end{align*}
Subtracting the two sides gives $\langle h - \sR h, g \rangle_{L_2(X)} = 0$ for all $g\in\sI$. Together with part~\emph{(a)}, $\sR h\in\sI$ and $h - \sR h \perp_{L_2(X)} \sI $, therefore for all $h\in L^0_2(X)$,
\begin{align*}
    \sR h = \sPi h .
\end{align*}

\paragraph{Application to ERM and Bias.} Under {\asmOneA}, $h_{\erm}(X) = \EE{}{}{Y \given X}$ and $h_{\term}(\stX) = \EE{}{}{Y \given \stX}$. Since $Y$ is centered, all of $h_{\target}, h_{\erm}, h_{\term}$ (hence $b, \tb$) are centered. Further, $T\indep (X, Y, U, \residual)$ by construction, so $Y - X - \tX$ is a Markov chain. Via the~tower~rule,
\begin{align*}
    h_{\term}(\tX) = \EE{}{}{Y \given \smash[t]{\tX}}
    &= \EE{}{}{ \EE{}{}{Y \given X, \stX} \given \tX },
    \\
    &= \EE{}{}{ \EE{}{}{Y \given X} \given \tX } = \EE{}{}{ h_{\erm}(X) \given \smash[t]{\tX}} = (\sR h_{\erm})(\stX) .
\end{align*}
Note that $h_{\term}$ and $\sR h_{\erm}$ are $\sT$-inv., hence they must also be identically equal in $L_2(X)$ ({\asmOneB}), giving $h_{\term} = \sR h_{\erm} = \sPi h_{\erm}$ by part~\emph{(b)}. And since $h_{\term}\in\sH_{\stX}$ by {\asmOneA}, part~\emph{(a)} places it in the class: $h_{\term}\in\sH_{\inv}$. 
Finally, for a $\sT$-inv $h_{\target}\in\sI$, using $\sPi h_{\target} = h_{\target}$ and linearity of $\sPi$ on $L^0_2(X)$,
\begin{equation*}
    \sPi b
    \quad
    =
    \quad
    \sPi h_{\erm} - \sPi h_{\target}
    \quad
    =
    \quad
    h_{\term} - h_{\target}
    \quad
    =
    \quad
    \tb
    .
    \qedhere
\end{equation*}
\end{proof}
}

\newpage
\subsection{Proof of \cref{lemma:avg-representer-norms}---Average Representer Norms}
\label{app:avg-representer-norms}
{
\allowdisplaybreaks

\begin{lemma}[avg. representer norms]
\label{lemma:avg-representer-norms}
Under~{\asmOneShortBC}, evaluation $h\mapsto h(\vx)$ admits the $L_2(X)$-representer $g_{\vx}\in\sH_X$ and the $L_2(\tX)$-representer $\tg_{\vx}\in\sH_{\stX}$. Averaged over queries $X\sim p_X$~gives
\begin{align*}
    \EE{X}{}{
        \sqNorm{
            g_X
        }_{L_2(X)}
    } = k
    ,
    &&
    \EE{X}{}{
        \sqNorm{
            \tg_X
        }_{L_2(\stX)}
    } = \tr{\shift}
    .
\end{align*}
\end{lemma}
\begin{proof}
Two steps give the result: (\rnum{1}) \textbf{expand} each representer in an orthonormal basis of its~own~centered class; (\rnum{2}) \textbf{average} over $X\sim p_X$, reading each summand as a quadratic form of the~shift~operator.

\paragraph{Expansion.}
Fix an $L_2(\stX)$-orthonormal basis $\br*{\psi_i}_{i=1}^k$ of the class $\sH_{\stX}$ ({\asmOneShortBC}, cf.~\cref{app:proofs}).~Then,
\begin{align*}
    \tag{Parseval in $\langle \cdot, \cdot \rangle_{L_2(\stX)}$.}
    \tg_{\vx}
    &=
    \sum_{i=1}^k
    \langle \tg_{\vx}, \psi_i \rangle_{L_2(\stX)}
    \cdot
    \psi_i
    ,
    \\
    \tag{Reproducing identity $\langle \psi_i, \tg_{\vx}\rangle_{L_2(\stX)} = \func{\psi_i}{\vx}$.}
    &=
    \sum_{i=1}^k
    \func{\psi_i}{\vx}
    \cdot
    \psi_i
    ,
    \\
    \tag{Orthonormality of $\psi_i$.}
    \Longrightarrow
    \quad
    \sqNorm{\tg_{\vx}}_{L_2(\stX)}
    &=
    \sum_{i=1}^k
    \func{\psi_i}{\vx}^2
    .
\end{align*}

\paragraph{Averaging.}
Take expectations over $X\sim p_X$ in each summand,
\begin{align*}
    \tag{Def.~of $\shift$.}
    \EE{X}{}{
        \func{\psi_i}{X}^2
    }
    =
    \sqNorm{\psi_i}_{L_2(X)}
    =
    \langle \psi_i, \shift \psi_i \rangle_{L_2(\stX)}
    ,
\end{align*}
and sum over the basis to obtain a trace,
\begin{align*}
    \tag{Trace in any orthonormal basis.}
    \EE{X}{}{
        \sqNorm{\tg_X}_{L_2(\stX)}
    }
    =
    \sum_{i=1}^k
    \langle \psi_i, \shift \psi_i \rangle_{L_2(\stX)}
    =
    \tr{\shift}
    .
\end{align*}
Repeating the two steps in an $L_2(X)$-orthonormal basis $\br*{\psi_i}_{i=1}^k$ of $\sH_X$ replaces $\shift$ by $\Id$ throughout. The averaged summand is then $\sqNorm{\psi_i}_{L_2(X)}=1$, now resulting in $\EE{X}{}{\sqNorm{g_X}_{L_2(X)}} = \tr{\Id}=~k$.
\end{proof}

}

\end{document}